\documentclass[11pt,a4paper]{article}
\usepackage[left=3cm, right=3cm, top=3.5cm]{geometry}
\usepackage{amssymb,amsmath, amsthm, amsfonts, rotating}
\usepackage{mathrsfs,color,bbold,url}
\usepackage{stmaryrd}
\usepackage{tikz}
\usepackage[all]{xypic}
\usepackage{enumerate}

\newtheorem{theorem}{Theorem}[section]
\newtheorem{lemma}[theorem]{Lemma}
\newtheorem{corollary}[theorem]{Corollary}
\newtheorem{proposition}[theorem]{Proposition}

\theoremstyle{definition}
\newtheorem{definition}[theorem]{Definition}

\newtheorem{example}[theorem]{Example}

\usepackage{adjustbox}
\usepackage[most]{tcolorbox}

\usepackage{tikz-qtree}
\usepackage{tikz-qtree-compat}

\tikzset{every tree node/.style={align=center,anchor=north}}

\makeatother

\begin{document}

\title{The Possibilistic Horn Non-Clausal Knowledge Bases }
\author{Gonzalo E. Imaz\\ [2mm] %
{\small  Artificial Intelligence Research Institute (IIIA) - CSIC, Barcelona, Spain}\\
{\small \tt email: {\{gonzalo\}}@iiia.csic.es}\\[1mm]
}

\date{}
\maketitle

\begin{abstract}

Posibilistic logic is the most extended approach to handle uncertain and partially
inconsistent information. Regarding normal forms, advances in possibilistic reasoning
are mostly focused on  clausal form.
Yet, the encoding of real-world problems  usually results in 
a non-clausal (NC) formula and   
   NC-to-clausal translators produce
severe drawbacks that heavily limit  the practical performance 
of clausal   reasoning.  Thus, by 
 computing formulas   in its original  NC form, we  propose several contributions showing
 that notable advances are also possible in  possibilistic  non-clausal reasoning.
 
   {\em Firstly,} we   define the class of {\em  
 Possibilistic Horn Non-Clausal   Knowledge Bases,} or  
  $\mathcal{\overline{H}}_\Sigma $,
 which subsumes the classes:
   possibilistic Horn   and
    propositional Horn-NC.   $\mathcal{\overline{H}}_\Sigma $ is shown to be
  a kind of NC analogous of the  standard  Horn     class.
 
{\em Secondly}, we   define    {\em  Possibilistic Non-Clausal Unit-Resolution,} 
or   $ \mathcal{UR}_\Sigma $,   and   prove that $ \mathcal{UR}_\Sigma $ correctly   computes the inconsistency degree of $\mathcal{\overline{H}}_\Sigma $
members.  $\mathcal{UR}_\Sigma $ had not been  proposed before
and is formulated in a clausal-like manner, which eases its understanding, 
 formal proofs and future  extension towards  non-clausal resolution.

{\em Thirdly}, we prove that  computing the inconsistency degree
of $\mathcal{\overline{H}}_\Sigma $  members takes  polynomial time.
Although there already exist tractable   classes in possibilistic logic, 
all of them  are  clausal, and thus, 
   $\mathcal{\overline{H}}_\Sigma $ turns out to be the first 
characterized polynomial non-clausal class within possibilistic   reasoning. 
 
We discuss that our approach serves as a starting point to  
  developing uncertain non-clausal  reasoning on the basis 
  of both methodologies: DPLL and    resolution. 
 

 
\vspace{.05cm} 
\noindent {\bf  Keywords:}  {\em Possibilistic Logic; Horn; Non-Clausal; Inconsistency; 
  Tractability;   Resolution; DPLL; Satisfiability Testing; Logic Programming.}
\end{abstract}



\newenvironment{niceproof}{\trivlist\item[\hskip
\labelsep{\it Proof.\/}]\ignorespaces}{\hfill$\blacksquare$\endtrivlist}

\newenvironment{proofsketch}{\trivlist\item[\hskip
\labelsep{\it Proof Sketch.\/}]\ignorespaces}{\hfill$\blacksquare$\endtrivlist}

\newenvironment{niceproofsketch}{\trivlist\item[\hskip
       \labelsep{\bf Proof Sketch.\/}]\ignorespaces}{\hfill$\blacksquare$\endtrivlist}



\section{Introduction}

Possibilistic logic is the most popular approach to represent and  reason
with  uncertain and partially inconsistent knowledge.  Regarding normal forms, the encoding 
of real-world problems  does usually not result 
in a clausal  formula and 
   although      a possibility 
non-clausal formula is theoretically   equivalent to some possibilistic clausal formula
 \cite{DuboisP14, DuboisP94},   approaches needing clausal form 
 transformations are practically infeasible or have  experimentally shown 
 to be highly inefficient as    discussed below. 
 
 \vspace{.05cm}
 Two kinds of clausal form transformation
 are known: (1) one is based on the repetitive application of the distributive laws 
 to the input non-clausal formula until  a logically equivalent clausal formula is obtained; 
 and (2) the other transformation, Tsetin-transformation \cite{Tseitin83}, 
 is based  on recursively 
 substituting    sub-formulas  in the input non-clausal formula  by fresh literals until 
 obtaining an  equi-satisfiable, but not equivalent, clausal formula.
 
 \vspace{.05cm}
 The first   transformation  blows up exponentially 
 the formula size, 
 and since real-world problems   have 
 a large number of variables and connectives,  
 the huge dimension of the resulting clausal formulas     
  prevents even   highly-efficient state-of-the-art
 solvers from attaining solutions in a reasonable  time.  
 
\vspace{.05cm}
The second kind of transformation also involves a number of drawbacks. The Tseitin-transformation usually produces an increase of formula size and  
 number of variables, and also a loss  of information  about the  formula's original structure.
Besides in most cases,  the normal form is not unique. Deciding how to perform the transformation
enormously influences the solving process and it is usually impossible to predict which strategy
is going to be the best, as this depends on the concrete 
solver used and on the kind of problem which should be solved. Further, 
Tseitin-transformation keeps the satisfiability test but losses the logical equivalence,  
 which rules out its usage in many real-world problems. 

\vspace{.05cm} 
We abandon the assumption that the input formula  should be transformed
to clausal form  and directly process it     
in its original structure.
 Since real-world problems  rarely occur in clausal form, 
we   allow an arbitrary nesting of conjunctions and disjunctions and only limit the scope of the negation connective. The non-clausal form considered here is popularly called
negation normal form   (NNF), and  can be obtained deterministically and   causing only a negligible increase of the formula size. 

\vspace{.05cm}
Developing methods for NC   reasoning is an actual  concern in  
the principlal fields of classical logic,  namely  satisfiability solving \cite{Otten11,LiMS19}, 
logic programming \cite{ConfalonieriN11,CabalarFSS19},   theorem proving \cite{FarberKaliszyk19,OliverO20} 
 and quantified boolean formulas \cite{EglySW09,BubeckB13}, 
and in many  other fields (see   \cite{Imaz2021horn} and the references thereof).
 And within non-classical logics, NC formulas with different  
   functionalities  have    been studied in a profusion of  languages:  
   signed many-valued logic   \cite{ MurrayRosenthal94,  BeckertHE98, Stachniak01}, 
 \L ukasiewicz  logic \cite{Lehmke96},  Levesque's three-valued  logic \cite{CadoliS96}, 
   Belnap's four-valued logic  \cite{CadoliS96}, 
 M3  logic \cite{Aguilera97},  fuzzy logic \cite{Habiballa12}, fuzzy description  logic \cite{Habiballa07},    intuitionistic logic \cite{Otten17},
  modal logic  \cite{Otten17},  lattice-valued  logic \cite{00010HZC18} and regular
  many-valued logic \cite{Imaz21b}.
We highlight the proposal in \cite{NievesL12,NievesL15}  as is the only
existing approach, to our knowledge,    to deal with  possibilistic non-clausal formulas, 
 concretely within the  answer set programming field.



\vspace{.05cm}
On the other side,  the   Horn clausal formulas are pivotal elements of our proposed 
possibilistic  reasoning approach
towards combining non-clausal expressiveness with
  high efficiency. 
Horn formulas are recognized as central for deductive databases, declarative programming,
and more generally, for rule-based systems. In fact, Horn formulas have received a great
deal of attention since 1943  \cite{McKinsey43,Horn51} and, at present, 
there is a broad span of areas within
artificial intelligence relying on them, and their scope covers a fairly large spectrum of
realms spread across many logics and a variety of reasoning settings.

\vspace{.05cm}
Regarding   possibilistic Horn formulas, computing their 
inconsistency degree
  is a tracta-ble problem
\cite{Lang00} and even almost-lineal    \cite{AlsinetG00}. 
Related to this standard Horn class but 
     going beyond clausal form, 
we  present a novel possibilistic   class, denoted
   $ \mathcal{\overline{H}}_\Sigma$,   
that is in NC form  and that 
we call  Horn Non-Clausal (Horn-NC). 
We show that $ \mathcal{\overline{H}}_\Sigma$ is a sort 
of non-clausal analogous of the   possibilistic Horn  class.
Besides  the   latter, 
  $ \mathcal{\overline{H}}_\Sigma$ also subsumes the class of propositional Horn-NC formulas 
recently presented   \cite{Imaz2021horn}.

\vspace{.05cm}
From a computational view, we  prove   that computing the inconsistency degree
of   $ \mathcal{\overline{H}}_\Sigma$ members is a  tractable problem.
This   result  signifies  that  polynomiallity in our context
is preserved when upgrading both  from  clausal to non-clausal form and from propositional to possibilistic logic. Polynomiallity is preserved when upgrading from  clausal to non-clausal form because both classes possibilistic Horn  \cite{Lang00} and possibilistic Horn-NC are tractable. Similarly, polynomiallity is preserved when upgrading from propositional to possibilistic logic because both  classes  propositional Horn-NC \cite{Imaz2021horn}  and  possibilistic Horn-NC are tractable.

\vspace{.05cm}
In summary, our contributions are:     introducing the hybrid class of {\em Possibilistic Horn Non-Clausal  
Knowledge Bases,} or $\mathcal{\overline{H}}_\Sigma $, 
and then,  
    proving that computing their inconsistency degree 
    is a polynomial problem.
Our contributions are outlined next.

\vspace{.05cm}
 {\bf Firstly,} the syntactical Horn-NC
restriction is determined  by  lifting the  Horn clausal restriction 
  ``a formula  is Horn if all its  clauses have any number of negative 
literals and  at most one positive literal",
 to the non-clausal  level in the following manner: 
 {\em ``a propositional NC formula is  Horn-NC if all its 
 disjunctions have any number  of negative disjuncts and 
 at most one non-negative   disjunct".} 
%
 By extending such definition to possibilistic logic, 
 we establish straightforwardly that: 
 a  possibilistic  NC knowledge base is Horn-NC
only if all its  propositional  formulas are Horn-NC.
Accordingly, $\mathcal{\overline{H}}_\Sigma $ is defined  
 as the {\em class of Possibilistic  Horn-NC  Knowledge Bases}. 
 Note that $\mathcal{\overline{H}}_\Sigma $ naturally subsumes 
 the standard  possibilistic  Horn  clausal   class.
 
 \vspace{0.05cm} 
 The set relations that the new  class  
 $\mathcal{\overline{H}}_\Sigma $ bears  
 to the   standard possibilistic classes Horn ($\mathcal{H}_\Sigma$), 
  Non-Clausal ($\mathcal{NC}_\Sigma$)  
  and Clausal ($\mathcal{C}_\Sigma$) 
  are depicted  in {\bf Fig. 1}. Specifically,
we show the next   relationships of $\mathcal{\overline{H}}_\Sigma $ with   $\mathcal{H}_\Sigma$ and  $\mathcal{NC}_\Sigma$:  
 (1) $\mathcal{\overline{H}}_\Sigma $ and $ \mathcal{{H}}_\Sigma $ are related in that
 $\mathcal{\overline{H}}_\Sigma $  subsumes syntactically $\mathcal{{H}}_\Sigma $  but both classes are 
semantically equivalent;
and  (2) $\mathcal{\overline{H}}_\Sigma $  and $\mathcal{NC}_\Sigma$ are related  in that 
$\mathcal{\overline{H}}_\Sigma $ contains all NC
 bases whose clausal form  is Horn. Thus, in view of 
  (1) and 
 (2) relations,  $\mathcal{\overline{H}}_\Sigma $ is a sort of NC analogous of $ \mathcal{{H}}_\Sigma $.

\vspace{-.1cm}   
\begin{center}
 \begin{tikzpicture}
  \begin{scope}[blend group=soft light]
    \fill[black!40!white]   (270:.9) circle (2.3);
    \fill[green!100!white] (230:1.2) circle (1.3);
    \fill[red!100!white]    (310:1.2) circle (1.3);
\end{scope}
  \node at (100:.8)       {\bf \ $\ \mathcal{NC}_\Sigma$};
  \node at (215:1.6)      {\bf $\mathcal{C}_\Sigma$};
  \node at (325:1.6)    {\bf $\mathcal{\overline{H}}_\Sigma$};
 \node at (270:.9)         {\bf $\mathcal{{H}}_\Sigma$};
\end{tikzpicture}

\vspace{.1cm}
\footnotesize{{\bf Fig. 1.} 
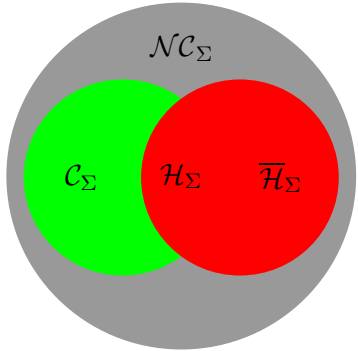
The  Horn, clausal, Horn-NC and NC    classes.}
\end{center}

{\bf Secondly,}  we establish the inferential calculus {\em Posibilistic Non-Clausal  Unit-Resolution}, or  $\mathcal{UR}_\Sigma$,  and then
prove  that $\mathcal{UR}_\Sigma$ correctly computes the inconsistency 
degree of the bases in the class $\mathcal{\overline{H}}_\Sigma $.   
   NC unit-resolution  
for propositional logic has  been recently presented  \cite{Imaz2021horn}
and $\mathcal{UR}_\Sigma$ is its generalization to possibilistic logic. 
$\mathcal{UR}_\Sigma$ is formulated in a clausal-like fashion,  which contrasts
with the  functional-like fashion of the existing (full) non-clausal resolution \cite{Murray82}.
We argue that our  clausal-like formulation   eases
the understanding of  $\mathcal{UR}_\Sigma$, the  building of the required
formal proofs and the future generalization of $\mathcal{UR}_\Sigma$ to determine Non-Clausal Resolution for possibilistic and for  other   uncertainty logics.

\vspace{.05cm}
{\bf Thirdly,} we prove that  computing the consistency degree  
of   $\mathcal{\overline{H}}_\Sigma $ members  
  has polynomial  complexity. There indeed exist polynomial classes in possibilistic logic 
  but all of them are clausal \cite{Lang00}, and so, the tractable non-clausal fragment
  was empty. We think that this is just a first tractable result in possibilistic
    reasoning and  that the approach presented here will   serve  
  to   widen the 
  tractable possibilistic non-clausal fragment.


\vspace{0.05cm} 
Below we give an specific possibilistic  non-clausal  base $\Sigma  $, whose  suffix notation 
will be detailed in Section \ref{sec:NCbasis} and  wherein  $ P, Q, \ldots $ and
$  \neg {P}, \neg {Q}, \ldots $ are positive and negative literals,
 respectively,  and $ \phi_1 $,  $ \phi_2 $ and $ \phi_3 $ are   non-clausal propositional formulas. We will 
$$\varphi= \{\wedge \ \,   P   \ \,  (\vee \ \ \neg Q   \ \ \{\wedge \ \ 
(\vee \ \ \neg P  \ \ \neg Q  \ \ R \,) \ 
\ (\vee \ \ \phi_1 \ \ \{\wedge \ \ \phi_2  \ \  \neg P \,\} \, )
 \ \, Q \,\} \,) \ \  {\phi_3} \ \}  $$
$$\Sigma= \{\,\langle \varphi : {\bf 0.8} \rangle \quad \langle P : {\bf 0.8} \rangle 
\quad \langle \neg Q : {\bf 0.6} \rangle
\quad \langle R : {\bf 0.6} \rangle \quad 
\langle \phi_1 : {\bf .3} \rangle \quad \langle \phi_3 : {\bf 1} \rangle \, \}$$

\noindent   show that    $ \Sigma $ is Horn-NC when  $ \phi_3 $ is Horn-NC and 
at least  one of $ \phi_1 $  or  $ \phi_2 $ is negative. 

\vspace{.1cm}
  Recapitulating, the list of   properties 
of $ \mathcal{\overline{H}}_\Sigma $ is   given below, where   the last two properties have been   shown in \cite{Imaz2021horn} for propositional logic but are inherited by $ \mathcal{\overline{H}}_\Sigma $:
\begin{itemize} 

 \item Computing the inconsistency degree  of $ \mathcal{\overline{H}}_\Sigma $ is  tractable.

 \vspace{-.25cm}
 \item  $ \mathcal{\overline{H}}_\Sigma $ subsumes syntactically the  possibilitic Horn class.
 
 \vspace{-.25cm}
 \item  $ \mathcal{\overline{H}}_\Sigma $ is equivalent  semantically to the possibilitic Horn class.
 
  \vspace{-.25cm}
 \item $ \mathcal{\overline{H}}_\Sigma $ contains all possibilitic NC bases  whose clausal form  is Horn.
 
 
 \vspace{-.25cm}
 \item $ \mathcal{\overline{H}}_\Sigma $ is linearly recognizable \cite{Imaz2021horn}.
 
  \vspace{-.25cm}
 \item  $ \mathcal{\overline{H}}_\Sigma $ is strictly   succincter\footnote{Succinctness
 was defined in \cite{GogicKPS95}.}
 than the possibilistic Horn   class \cite{Imaz2021horn}.

\end{itemize}




The presented approach serves as starting point  
to develop approximate non-clausal reasoning based on (1) DPLL   and (2) resolution: 
(1)    $\mathcal{UR}_\Sigma$ 
paves the way  to define    DPLL  in NC
since its NC Unit-Propagation is  based on  NC Unit-Resolution, i.e. $\mathcal{UR}_\Sigma$; 
and (2) the existing NC resolution \cite{Murray82} presents some 
deficiencies derived from its functional-like formalization, such as not precisely defining 
the potential resolvents. 
Our clausal-like formalization of $\mathcal{UR}_\Sigma$
skips such deficiencies and   signifies a step forward
 towards  defining   NC resolution  for at least those  uncertainty logics
  for which clausal resolution is already   defined, e.g. possibilistic logic
  \cite{DuboisP87, DuboisP90}.

\vspace{.05cm} 
  This  paper is organized as  follows.
  Section \ref{sec:NCbasis}  and  \ref{sec:refresher}  present   
    background  on propositional non-clausal  formulas and on   possibilistic logic, respectively.
   Section \ref{sec:definClassHorn-NCChapeau}    defines  
   the class  $\mathcal{\overline{H}}_\Sigma $.  Section  \ref{sec:Non-Clausal-Unit-resolution}
    introduces  the calculus $\mathcal{UR}_\Sigma$. 
    Section \ref{sec:illustrativeexamples} provides examples illustrating 
    how $\mathcal{UR}_\Sigma$   computes  
      $\mathcal{\overline{H}}_\Sigma $ members.
    Section \ref{sect:ProofsProperties} provides the formal proofs of 
       the correctness of $\mathcal{UR}_\Sigma$ and of the tractability of   
          $\mathcal{\overline{H}}_\Sigma $.
     Section \ref{sec:relatedwork } 
     focuses on
  related    and  future work. Last section    summarizes the main contributions.

\section{Propositional Non-Clausal    Logic} \label{sec:NCbasis}

      This section   presents  some  terminologies used in this paper and  background  on 
 non-clausal (NC)   propositional   logic (see   \cite{DBLP:books/daglib/0029942} for a    complete background).  We present first the needed syntactical concepts and then the semantical  
ones. We begin by introducing  the  language.

\begin{definition}  \label{def:alphCNF} The   NC  language is formed by  
 the     sets:  constants $ \{{\bf \bot, \top}\} $,  propositions 
  $\mathcal{P}=$ \{P, Q, R, \ldots \}, 
 connectives  \{$\neg, \vee,  \wedge\}$ and auxiliary symbols 
 (,  ),  {\em \{ {\em and}   \}}. 
\end{definition}


Next we describe the required elements relative to clausal formulas.

\begin{definition}
  $X  $   (resp.  $ \neg {X}$)
  with $ X \in \mathcal{P} $  is 
   a  positive (resp. negative) literal. 
  $ \mathcal{L}$ is the set of  literals.  Constants and literals are atoms.
   $(\vee \ \ell_1 \ \ell_2 \ \ldots \ \ell_k)$, the $ \ell_i$   being  literals,  is 
a  clause.   A    clause with  at most
one   positive   literal is   Horn. 
    $\{\wedge \  C_1 \ C_2 \ \ldots \ C_n\}$, the  $ C_i $ being   clauses, is a 
  clausal formula. $ \mathcal{C} $ and $ \mathcal{H} $ 
  are the set of   clausal and Horn formulas, respectively.   
\end{definition}

\noindent {\bf Note.}  We firstly justify  our chosen  notation of non-clausal formulas
before defining them. Thus, for  the sake of readability of 
   non-clausal formulas, we will  employ:

\begin{enumerate}
 \item The prefix notation   as it     
       requires only one  $\vee / \wedge$-connective 
      per formula,  while infix notation requires
   $k-1$, $k$  being  the arity of the involved $\vee / \wedge$-connective.
  \item  Two   formula 
    delimiters (see  Definition \ref{def:NCformulas}),     $(\vee  \,\ldots \,)$  for disjunctions and  
 $\{\wedge  \,\ldots \,\}$ for conjunctions, 
  to better 
   distinguish  them   inside
  non-clausal  formulas.
  \end{enumerate}

\noindent  So our next definition is that of non-clausal formulas\footnote{Also called 
"negation normal form formulas" in the literature.},  whose differential feature 
 is that the connective
$ \neg $ can occur only in front of propositions, i.e.  at atomic level.

\begin{definition}  \label{def:NCformulas} The   
non-clausal formulas over a set of propositional variables $\mathcal{P}  $ is the smallest set $\mathcal{NC}$ such that the following conditions hold:
\begin{itemize}
\item   $\{\bot, \top\} \cup \mathcal{L} \,\subset \,\mathcal{NC}$.

\item  If  \ $\forall i \in \{1, \ldots k \}$,  $\varphi_i \in \mathcal{NC}$  \,then    
\,$\{\wedge  \ \varphi_1 \ldots \varphi_{i}  \ldots    \varphi_k\} \in  \mathcal{NC}$.

\item   If \ $\forall i \in \{1, \ldots k \}$, $\varphi_i \in \mathcal{NC}$ \,then  \,$\,(\vee  \ \varphi_1  \ldots 
\varphi_{i} \ldots   \varphi_k) \in  \mathcal{NC}$.
\end{itemize}

\vspace{.1cm}
 -- $\{\wedge    \,\varphi_1 \ldots  \varphi_{i}  \ldots  \varphi_k \}$ and any $\varphi_i$ 
 are called    conjunction and conjunct,  respectively.

\vspace{.1cm}
  -- $(\vee   \,\varphi_1  \ldots  \varphi_{i} \ldots  \varphi_k )$ and any $\varphi_i$ are called  
disjunction  and disjunct, respectively. 

\vspace{.1cm}
-- $[ \odot  \,\varphi_1  \ldots  \varphi_{i} \ldots  \varphi_k]$    stands for both 
 $(\vee   \,\varphi_1  \ldots  \varphi_{i} \ldots   \varphi_k ) \mbox{\ and \ }
\{\wedge   \,\varphi_1  \ldots  \varphi_{i}  \ldots  \varphi_k \}.$  
\end{definition}

\begin{example} \label{exsec1:varisexamples}  $ \varphi_1 $ 
to $ \varphi_3 $ below are  NC formulas, while   $ \varphi_4 $ is  not.
We will show that $ \varphi_2 $ 
is   Horn-NC while   $ \varphi_1 $ is not Horn-NC, and as $ \varphi_3 $ includes $ \varphi_1 $, then $ \varphi_3 $ is not Horn-NC either. On the other side, the example   
in the  Introduction  is Horn-NC under certain conditions.

\begin{itemize}

\item     $  \ \varphi_1=\{\wedge  \ \ 
(\vee  \ \ \neg {P}  \ \ Q  \ \ \bot   \,)
 \ \ (\vee \ \ Q   \  \   \{\wedge \ \neg {R}  \ \ S  \ \ \top \, \} \, )  \, \}$

\item     $ \ \, \varphi_2= (\vee \ \ \{\wedge \ \ \neg {P}  \ \ \top  \} \  \
  \{\wedge  \ \  (\vee  \ \  \neg {P}  \ \ R   \,) \ \ 
  \{\wedge \ \ Q  \ \ (\vee \ \ P  \ \ \neg {S}  \,) \,\} \,\} \ 
  \ \{\wedge \ \ \bot \ \  Q\}\,)$

\item    $  \ \varphi_3=(\vee \ \ \varphi_1 \ \ \{\wedge  \ \ Q    \ \ 
(\vee \ \ \varphi_1 \ \ \neg {Q}  \ \ \varphi_2) \,\} \ \ 
\{\wedge \ \ \varphi_2 \ \ \top  \ \    \varphi_1  \,\} \,)$

\item  \   $ \varphi_4= \neg (\vee \ \varphi_1 \ \varphi_2) $ \qed

\end{itemize}
\end{example} 

\begin{definition} \label{def:sub-for} Sub-formulas are  recursively defined as 
follows.  The unique sub-formula of an atom  is the atom itself, and 
the sub-formulas of a formula $\varphi=\langle \odot \  \varphi_1 \ldots \varphi_i \ldots \varphi_k \rangle$ 
  are $\varphi$ itself plus  the sub-formulas of  the     $\varphi_i$'s. 
\end{definition}

\begin{example} The sub-formulas of a clausal formula are the 
formula itself plus its clauses, literals
 and constants.
\end{example}

\begin{definition} \label{def:graphformula}  
NC formulas  are   modeled by trees if:  (i) the nodes are:
each atom  is a   {\em   leaf} and
each  occurrence of   
 a $ \wedge /\vee $-connective  is   an {\em  internal node}; and (ii)  {\em the arcs} are:  each
 sub-formula $[ \odot \  \varphi_1 \ldots  \varphi_{i}  \ldots \,\varphi_k ]$    
 is  a  $k$-ary hyper-arc
 linking  the node of  $\odot$      
  with, for every $i$, the node of $\varphi_i$  
  if $\varphi_i$  is an atom and
 with the node of its connective  otherwise.
\end{definition}

\begin{example} \label{Ex:tree} The   tree    of  
$\{\wedge \ \, \neg R \ \, (\vee \ \, \{\wedge \ \, \neg P \ \, Q \} \ \, \bot \ \,
(\vee \ \, \neg P \ \, \neg R)\,)\,\}   $ is given in Fig. 2.

\begin{center}
\begin{tikzpicture}[sibling distance=10em,
  every node/.style = {shape=rectangle, rounded corners,
    draw, align=center,
    top color=pink, bottom color=pink!100}]]
  \node {$\wedge$}
    child { node {$\neg {R}$} }
      child { node {$\vee$} 
        child { node {$\wedge$} 
           child { node {$\neg {P}$} }
           child { node {$Q$}    } }
        child { node {$\bot$} } 
        child  { node {$\vee$} 
           child { node {$\neg {P}$}  }  
           child { node {$\neg {R}$}  }}};
\end{tikzpicture}

\vspace{.1cm}
{\normalsize{\bf Fig. 2.}   Tree of Example \ref{Ex:tree}.}
\end{center}
\vspace{-.5cm}
\qed
\end{example}


\noindent 
\noindent {\bf  Remark.}  Directed acyclic graphs (DAGs)  generalize trees 
and allow for important savings in    space and  time.
Our approach also applies when  NC formulas  are represented 
and implemented by DAGs.  Nevertheless, for simplicity, we will use  formulas representable by   trees in the illustrative examples throughout this article.
  




\vspace{.15cm}
\noindent $-$ In the remaining of this section, we present  semantical notions.

 \vspace{.1cm}
\noindent $-$   In the next definition (first line), the empty disjunction $ (\vee) $ is 
   considered equivalent  to a  $ \bot $-constant
    and the   empty conjunction $ \{\wedge\} $  to
  a $ \top $-constant.

 \begin{definition} \label{def:interpretation} An interpretation  
  $\omega$ maps the formulas $ \mathcal{NC} $  into the truth-value set $\{0,1\}$  and
  is extended from propositional variables $ \mathcal{P}$ to 
  formulas $ \mathcal{NC} $  via the  rules below,   
where $X \in \mathcal{P}$ and 
$ \varphi_i \in  \mathcal{NC}$, $1 \leq i \leq k  $.  We will denote $ \Omega $
the   universe of  interpretations.
\begin{itemize}

\item   $\omega(\bot)=\omega(\,(\vee)\,)=0$
 \, and \, $\omega(\top)=\omega(\,\{\wedge\}\,)=1$.

\item   $\omega(X) + \omega(\neg X)= 1 $. 

\item      $\omega( \, (\vee\  \varphi_1 \ldots \varphi_{i}  \ldots  \varphi_k) \, ) \,=
  \mathrm{max}\{\omega(\varphi_i): 1 \leq i \leq k \}$.

\item       $\omega(\, \{\wedge\  \varphi_1 \ldots \varphi_{i}  \ldots   \varphi_k\}\, )
=
\mathrm{min}\{\omega(\varphi_i): 1 \leq i \leq k \}$.

\end{itemize} 
\end{definition}

\begin{definition} \label{def:model} $ \varphi $ and $ \varphi' $ being   formulas, some well-known semantical notions  follow. An interpretation    $\omega$ is a   model of   $\varphi$  if  $\omega(\varphi)=1.$ 
If $\varphi$  has a model then it is   consistent and otherwise inconsistent.
  $\varphi$ and $\varphi'$ are  (logically)  equivalent, denoted
 $\varphi \equiv \varphi'$,   if   
   $\forall \omega$, 
  $\omega(\varphi)=\omega(\varphi')$.
 $\varphi'$ is  logical consequence of  $\varphi$, denoted
 $\varphi \models \varphi'$,   if   
   $\forall \omega$, 
  $\omega(\varphi) \leq \omega(\varphi')$.

\end{definition}

 Next, some well-known rules allowing to simplify formulas  are supplied.

\begin{definition} \label{def:simpl}    
 Constant-free,  equivalent      formulas   are  straightforwardly obtained 
  by recursively applying to  sub-formulas the   simplifying  rules below:  

\vspace{.25cm}
$\bullet$ Replace \  \,$(\vee \  \top \  \varphi \,)$   \ with \  $  \top $. 

\vspace{.1cm}
$\bullet$  Replace \ $\{\wedge \   \bot \  \varphi \,\}$  \  with \ $  \bot $.

\vspace{.1cm}
$\bullet$ Replace \ $\{\wedge \   \top \  \varphi \,\}$   \   with  \ $ \varphi  $.

\vspace{.1cm}
$\bullet$  Replace \ \,$(\vee \   \bot \  \varphi \,)$  \ \,with  \ $  \varphi$.
 
\end{definition}

\begin{example} \label{ex:constantsfree} The constant-free,   equivalent NC formula  of 
$ \varphi_2 $  in Example \ref{exsec1:varisexamples} is:
   $$ \varphi=(\vee \ \ \neg P  \ \  \{\wedge \ \ (\vee \ \ \neg P   \ \ R  )  \ \
\{\wedge \ \ Q   \ \   (\vee \ \  P   \ \ \neg S  ) \, \} \, \}  )$$    

\end{example}

\noindent {\bf Remark.}  For simplicity and since free-constant,    equivalent  formulas are 
 easily obtained,  hereafter we will consider only free-constant formulas.
 
 \section{Necessity-Valued Possibilistic Logic} \label{sec:refresher}
 
  Let us have a brief 
refresher on  necessity-valued possibilistic logic (the reader may consult  \cite{DuboisP94,DuboisP04a, DuboisP14} 
for   more details).

\subsection{Semantics}

 At the   semantic level, possibilistic logic is defined in terms of a 
 {\em possibilistic distribution} $ \pi $ on the universe $ \Omega $ of interpretations,
 i.e. an $ \Omega \rightarrow [0,1] $ mapping which intuitively encodes for each 
  $ \omega \in \Omega$ to what extent it  is plausible that $ \omega $ is the actual world.
  $ \pi(\omega)=0 $ means that $\omega $ is impossible,
  $ \pi(\omega)=1 $ means that nothing prevents $ \omega $  from being true,
  whereas $ 0 < \pi(\omega) < 1 $ means that $ \omega $ is only somewhat possible 
  to be the real world. Possibility degrees are  interpreted qualitatively: 
 when $ \pi(\omega) >  \pi(\omega')$,
 $ \omega $ is considered more plausible than $ \omega' $.
  A possibilistic distribution $ \pi $ is  \emph{normalized}  if 
 $ \exists \omega \in \Omega, \pi(\omega)=1$, i.e. at least one interpretation
 is entirely plausible.

 \vspace{.1cm}
 A possibility distribution $ \pi $  induces two uncertainty functions 
 from the formulas $ \mathcal{NC} $  to $ [0,1] $, called 
   possibility and necessity functions and noted $ \Pi $ and $ N $, respectively,
  which allow us to rank  formulas.   $ \Pi $ is 
 defined by Dubois \emph{et al.} (1994) \cite{DuboisP94} as:
 $$ \Pi(\varphi)=\mathrm{max}\{\pi(\omega) \ \vert \ \omega \in \Omega, \omega \models \varphi\}, $$ 
 
\noindent and evaluates the extent to which  $ \varphi $  is consistent with the 
 beliefs expressed by $ \pi $. The dual \emph{necessity measure} $ N $ is defined
 by: 
 $$ N (\varphi)= 1- 
 \Pi(\neg \varphi)=\mathrm{inf}\{1-\pi(\omega) \ \vert 
 \ \omega \in \Omega, \omega \nvDash \varphi \},$$

\noindent and evaluates the extent to which   $ \varphi $ is entailed by the available
 beliefs \cite{DuboisP94}. So the lower the possibility of an interpretation
  that makes $ \varphi $ False, the higher the necessity degree of $ \varphi $.  
  $ N(\varphi)=1 $  means $ \varphi $ is a totally certain piece of knowledge, 
  whereas  $ N(\varphi)=0 $
  expresses the complete lack of knowledge of priority about $ \varphi $.
  Note that always  $ N(\top)=1 $ for any 
 possibility distribution, while $ \Pi(\top)=1 $ (and, related, $ N(\bot) $=0) only
 holds  when the possibility distribution is normalized, i.e. only normalized distributions
 can express consistent beliefs \cite{DuboisP94}.
  
 \vspace{.1cm} 
  A major property of $N$ is Min-Decomposability:
 $\forall \varphi, \psi, N(\varphi \wedge \psi)= \mathrm{min}(N(\varphi),N(\psi))  $. 
 However, for disjunctions only  
 $ N(\varphi \vee \psi) \geq \mathrm{max}(N(\varphi),N(\psi))$  holds.
 Further, one has    
 $ N(\varphi) \leq N (\psi) $   if   $ \varphi \models \psi  $, and hence,
 $ N(\varphi)=N (\psi) $   if   $ \varphi \equiv \psi  $.


\subsection{Syntactics}

A possibilistic  formula  is a   pair
 $ \langle \varphi :\alpha \rangle \in \mathcal{NC} \times (0,1]$,  where $ \varphi $ is a 
 propositional  NC formula,  $ \alpha \in (0,1]$  expresses
 the certainty that $ \varphi $ is the case, and it is  interpreted
 as the semantic constraint $ N(p) \geq \alpha $. So formulas of the form 
 $  \langle \varphi : 0 \rangle $ are 
 excluded. 
 A {\em possibilistic  base} $ \Sigma $ is 
 a  collection of possibilistic formulas 
 $\Sigma=\{ \langle\varphi_i : \alpha_i \rangle\, \vert \, i=1, \ldots, k \,\}$ 
 and   corresponds to a set of constraints on possibility distributions.
 The classical knowledge base associated with $ \Sigma $ is denoted as $ \Sigma^* $,
 i.e. $ \Sigma^*=\{\varphi \vert \langle  \varphi :  \alpha \rangle \in \Sigma\} $.
 $ \Sigma $ is consistent if and only if   $ \Sigma^* $ is
 consistent. It is noticeable that, due to  Min-Decomposability, 
 {\em a possibilistic logic base can be easily put in  clausal form.}\footnote{Nevertheless, as  said previously,  this translation can blow up  exponentially the size of formulas and so can  dramatically
 reduce the overall efficiency of the clausal reasoner.} 
  
  \vspace{.1cm}  
  Typically, there can be many possibility 
  distributions that satisfy the
  set of constraints $ N(\varphi) \geq \alpha $ but we are usually only interested in   the \emph{least specific possibility  distribution}, i.e. the possibility 
  distribution that makes minimal commitments, 
  namely, the {\em greatest possibility distribution}
  w.r.t. the following ordering: 
    $ \pi $  is a least specific possibility distribution
 compatible with $ \Sigma $ if for any 
 $ \pi' $, $ \pi' \neq \pi $,  compatible with $ \Sigma $,  one has
 $ \forall \omega \in \Omega, \pi(\omega) \geq \pi'(\omega)   $. Such a least specific  possibility distribution
  always exists and is unique \cite{DuboisP94}.
  
  \vspace{.1cm}  
  Thus, for  a given  $ \langle  \varphi :  \alpha \rangle $,  possibilistic distributions 
 should  consider that  an   
  $ \omega $ that makes $  \varphi$ True is possible at the maximal level, say 1, while an
   $ \omega $  that makes $ \varphi $ False is possible at most  at level
  $ 1 -\alpha $.    Thus 
 the semantic counterpart of a  base $ \Sigma $, or 
 the least specific distribution $ \pi_\Sigma $
 is defined by, $\forall \omega, \omega \in \Omega $: 
%
 %
%
 $$\pi_\Sigma(\omega)= \left\{
\begin{array}{l l}
1 & \   \mbox{if} \ \ \forall\langle\varphi_i,\alpha_i\rangle \in \Sigma, \omega \models \varphi_i\\
\mathrm{min}\{1- \alpha_i \,\vert \, \omega  \nvDash \varphi_i, 
\, \langle\varphi_i, \alpha_i\rangle \in \Sigma\} & \   \mbox{otherwise}\\
\end{array} \right. $$

 \begin{proposition} \label{propos:leastspecific} Let  $ \Sigma $ be a possibilistic  base. 
 For any possibility distribution $ \pi $  on $ \Omega $, $ \pi $ satisfies 
 $ \Sigma $  if and only if $ \pi \leq \pi_\Sigma $.
 \end{proposition}
 
 Proposition \ref{propos:leastspecific}  says that  $ \pi_\Sigma $ is the least specific possibility
 distribution satisfying $ \Sigma $ and it has been shown in reference \cite{DuboisP94}.

\subsection{Syntactic Deduction} 

 This subsection introduces some few notions about  
   deduction in possibilistic logic and 
  starts by the  well-known possibilistic inference rules to be handled in this article:

 \begin{definition}  \label{def:inderences}  We define below three  rules, where $ \ell \in \mathcal{L};  
 \varphi,\psi \in \mathcal{NC}$ and $\alpha, \beta \in (0,1] $.
  The first  is possibilistic resolution \cite{DuboisP87,DuboisP90}; 
the second  rule is    Min-Decomposability; and the third rule,  Max-Necessity, follows 
from the   semantic constraint meaning of  $ \langle  \varphi :  \alpha \rangle$.
\begin{itemize}

\item  $\mbox{\underline{Resol}}: \quad \   \langle \, (\vee \ \, \ell \ \, \varphi) : \alpha \rangle, 
\langle \, (\vee \ \, \neg \ell \ \, \psi) : \beta \, \rangle
\ \vdash \ \langle \ (\vee \ \, \varphi  \ \psi) :  
 \mbox{\bf min}\{\alpha, \beta\} \ \rangle$.

\item  $ \mbox{\underline{MinD}}: \quad  \,\langle  \varphi  : \alpha   \rangle,  
 \ \langle  \psi : \beta  \rangle  
\ \vdash  \ \langle \ \{\wedge  \, \varphi  \ \psi \} :
\mbox{\bf min}\{\alpha,  \beta\} \ \rangle$.


\item $ \mbox{\underline{MaxN}}: \quad \langle  \varphi :  \alpha \rangle, 
\ \langle \varphi : \beta \rangle 
\ \vdash \ \langle \ \varphi : \mbox{\em \bf max}\{\alpha, \beta\} \ \rangle$.
\end{itemize}
  

\end{definition}

 
Before formulating the soundness and completeness theorem in possibilistic logic, we need to
introduce the next concept of $ \alpha $-cut;
 we call the
 $ \alpha $-cut (resp. strict $ \alpha $-cut) of $ \Sigma$, denoted $ \Sigma_{\geq \alpha}$
 (resp. $ \Sigma_{> \alpha}$), the set of classical formulas in   $ \Sigma$ having
 a necessity degree at least equal to $ \alpha $  (resp. strictly greater than $ \alpha $),
 namely $ \Sigma_{\geq \alpha} 
 =\{ \varphi  \,\vert  \,\langle\varphi :\beta\rangle \in \Sigma, \beta \geq \alpha\}$  (resp.
 $ \Sigma_{> \alpha} 
 =\{  \varphi  \,\vert  \,\langle\varphi :\beta \rangle \in \Sigma, \beta > \alpha\}$).

 \begin{theorem} \label{theo:soundn+complete}  The following soundness and completeness theorem holds:
 $$ \Sigma \models_\pi \langle\varphi :\alpha \rangle 
 \ \Leftrightarrow \ \Sigma \vdash_\mathrm{Res}  \langle  \varphi :  \alpha \rangle 
 \quad \Longleftrightarrow
 \quad  \Sigma_{\geq \alpha} \models \varphi   
\ \Leftrightarrow \ \Sigma_{\geq \alpha} \vdash \varphi  $$
 
\noindent where $\models_\pi$ means any $ \omega $ compatible with $ \Sigma$ 
 is also compatible with $\langle\varphi :\alpha \rangle  $, or formally, 
 $ \forall \omega, \pi_\Sigma(\omega) \leq \pi_{\{\langle p : \alpha \rangle\}} (\omega) $. 
 $ \vdash_\mathrm{Res}  $ 
 relies on the repeated use  of   possibilistic resolution.
 
\end{theorem} 

 The last half of the above expression  reduces to the soundness and completeness
 theorem of propositional logic applied to each level cut of $ \Sigma $, which is an ordinary
 propositional   base. 

 \subsection{Partial Inconsistency}

 The   inconsistency degree of a base $ \Sigma $ 
 in terms of its $ \alpha $-cut  can be equivalently
defined as  the largest weight $ \alpha $ such that the $ \alpha $-cut of 
$ \Sigma $ is inconsistent:
$$\mathrm{Inc}(\Sigma)=\mathrm{max}\{\alpha  \, \vert \, \Sigma_{\geq \alpha }\, \mathrm{is \ inconsistent}\}.$$

 $ \mathrm{Inc}(\Sigma)=0 $ entails $ \Sigma^* $ is consistent. 
 In  \cite{DuboisP94},  the inconsistency degree of $ \Sigma $  is 
defined by the least possibility distribution $ \pi_\Sigma $,  concretely  
$\mathrm{Inc}(\Sigma)=1-\mathrm{sup}_{\omega \in \Omega} \pi_\Sigma (\omega).$

\vspace{.1cm}
To check   whether $  \varphi$ follows from
  $ \Sigma $, one should add $ \langle\neg \varphi : 1 \rangle $ to $ \Sigma $ and 
 then check  whether
$ \Sigma \cup \{\langle\neg \varphi : 1 \rangle\} \vdash 
\langle \bot : \alpha \rangle$. Equivalently    the maximum $ \alpha $ s.t. 
$ \Sigma \models \langle  \varphi :  \alpha \rangle$ is given by the inconsistency degree of
$ \Sigma \cup \{\langle\neg \varphi : 1 \rangle \}$, i.e. 
$ \Sigma \models \langle  \varphi :  \alpha \rangle$ iff 
$ \alpha=\mathrm{Inc}(\Sigma \cup \{\langle \neg \varphi, 1 \rangle\} ).$

\begin{proposition} \label{pro:iff-iff-iff} The next statements are proven in \cite{DuboisP94}:
 $$ \Sigma \models \langle  \varphi :  \alpha \rangle 
\ \  \mbox{\em iff} \ \ 
\Sigma \cup \{ \langle\neg \varphi : 1 \rangle \}  \vdash \langle\bot : \alpha \rangle 
 \ \ \mbox{\em iff} \ \
\alpha =  \mathrm{Inc}(\Sigma \cup \{\langle\neg \varphi : 1\rangle \})
 \ \ \mbox{\em iff} \ \
 \Sigma_{\geq \alpha} \vdash \varphi. $$

\end{proposition}

This result shows that any deduction problem in possibilistic logic can be viewed as computing
an inconsistency degree.

  \subsection{Clausal and Non-Clausal  Bases}
 
According to previous definitions,   to each  class of propositional formulas
corresponds a class of possibilistic bases.
Below, we   define   the possibilistic   classes  handled here
  and after recall  the    complexity
  of computing the inconsistency degree of their members.  
%

    \begin{definition}  
A possibilistic base $\Sigma=\{ \langle\varphi_i : \alpha_i \rangle\, \vert \, i=1, \ldots, k \,\}$ is  called  Horn, clausal or NC
if  all its   formulas $ \varphi_i, 1 \leq i \leq 1 $, are Horn, clausal or NC, 
 respectively. $\mathcal{H}_\Sigma$, $\mathcal{C}_\Sigma$ and $\mathcal{NC}_\Sigma$  
denote, respectively, the classes of possibilistic  Horn, clausal and NC  bases. 
\end{definition}

In this article we define in Definition   \ref{def:HNCKnowledgeBase}  
a novel possibilistic class, i.e. the possibilistic  Horn-NC class.  
   Next we just define  acronyms associated to   the problems
  of  computing the  inconsistency degree of   
  the four mentioned possibilistic  classes. 

\begin{definition}  \label{def:complexitiesofclassesofbases} Horn-INC, 
CL-INC, Horn-NC-INC  and NC-INC denote 
respectively the problems of computing the  inconsistency degree of 
 Horn, clausal, Horn-NC and NC
   bases.
\end{definition}

\noindent  {\bf Complexities.} Regarding the complexities of the previous problems, we have:
\begin{enumerate}
\item [$ \bullet $]  Clausal Pbs.: CL-INC is Co-NP-complete \cite{Lang00} and  
  Horn-INC is polynomial \cite{Lang00}.

\vspace{-.1cm}
\item [$ \bullet $]  NC-INC   is   Co-NP-complete. This claim stems from:
 (i) Theorem \ref{theo:soundn+complete} applies  to both clausal and NC
bases; and (ii)  checking whether
an interpretation is a model of an NC propositional formula is polynomial 
as for  clausal formulas. 

\vspace{-.1cm}
\item [$ \bullet $] Horn-NC-INC has polynomial complexity as proven in Section 
\ref{sect:ProofsProperties}.

\end{enumerate}

 \section{The Possibilistic Horn-NC  Class: $\mathcal{\overline{H}}_\Sigma$} \label{sec:definClassHorn-NCChapeau}

 This section  defines   the class {\em $\mathcal{\overline{H}}_\Sigma$} 
 of Possibilistic   Horn Non-Clausal (Horn-NC) 
  bases and states its properties and relationships with other possibilistic classes.
  The   proofs   were  given in \cite{Imaz2021horn} but are provided in an Appendix for
the sake of the paper being self contained.

  \vspace{.15cm} 
 {\em $\mathcal{\overline{H}}_\Sigma$}  subsumes  the next  two classes:

\begin{itemize}

\vspace{-.15cm}
\item   possibilistic Horn, or $\mathcal{{H}}_\Sigma$; and

\vspace{-.2cm}
\item   propositional  Horn-NC, or $\mathcal{\overline{H}}$ (recently presented  \cite{Imaz2021horn}).

\end{itemize}

\vspace{-.15cm}
\noindent    We first define the latter, i.e. the   class $\mathcal{\overline{H}}$ of 
propositional Horn-NC formulas, which is the propositional component of 
 the new possibilistic  class $\mathcal{\overline{H}}_\Sigma$ to be introduced.

 \subsection{Simple Definition of  $\mathcal{\overline{H}}$}
 
Below,  we       define $\mathcal{\overline{H}}$ in a simple way,  and  in the next
subsection, will give its detailed   definition   by taking a closer look
to this simple definition.
\noindent   We start by defining   
the negative formulas, which 
        generalize   the negative     literals in  the clausal framework.

 \begin{definition}  \label{def:negative} A  non-clausal formula is negative if it 
  has uniquely negative literals. 
 We will denote  $\mathcal{N }^-$   the set of
  negative formulas. 
\end{definition}


\begin{example}   Trivially negative literals are basic negative formulas. Another 
example of negative NC formula is $(\vee \ \ \{\wedge   \  \, \neg P   \  \, \neg R   \, \} 
\ \ \{\wedge \ \, \neg S   \  \,(\vee \ \, \neg P    \ \, \neg Q   \,)\,\}\,) \in \mathcal{N}^-$. \qed
\end{example}

 Next we   upgrade   the Horn  pattern 
 ``a   Horn clause has (any number of negative literals and) 
 at most  one   positive   literal"  to the NC 
 context  in the next   straightforward way:

\begin{definition} \label{theorem:visual} An   NC formula is  Horn-NC
  if all its disjunctions have any number of negative disjuncts and at most one 
 non-negative disjunct. 
We denote   $\mathcal{\overline{H}}$ the class of   Horn-NC formulas.
\end{definition}


  Clearly  the class $\mathcal{\overline{H}}$  subsumes  the Horn class $  \mathcal{H}$.
   From Definition \ref{theorem:visual} it follows trivially  that all 
sub-formulas of any Horn-NC  are Horn-NC too.
 Yet, the converse does not hold:   there are non-Horn-NC formulas
whose all sub-formulas are Horn-NC.

\begin{example} \label{exam:simple}  
One can see that $\varphi_1$ below has only one non-negative disjunct and so 
 $\varphi_1$ is   Horn-NC,  while $\varphi_2$  is not Horn-NC as it  
 has two non-negative disjuncts.
 \begin{itemize}
\item   $\varphi_1 = (\vee  \   \  \{\wedge \  \   \neg Q  \ \    \neg S  \} \  \  
\{\wedge \  \   R   \   \   P   \,\} \,).$ 
 
\item   $\varphi_2 = (\vee  \ \ \{\wedge  \ \  \neg Q  \ \ S   \} \ \ 
\{\wedge \ \ R  \ \ \neg P  \,\} \,)$. \qed

\end{itemize} 
\end{example}

 \begin{example} \label{Ex:morecomplex} We  now consider     
  both $\varphi$  in Example \ref{ex:constantsfree} (copied below) 
   and $\varphi'$ below too, which results from $\varphi$ by 
just switching  its  literal $\neg P  $  for $P   $: 
\begin{itemize}

\item $ \varphi=(\vee \ \ \neg P  \ \  \{\wedge \ \ (\vee \ \ \neg P   \ \ R  )  \ \
\{\wedge \ \ Q   \ \   (\vee \ \  P   \ \ \neg S  ) \, \} \, \}  )$

\item $\varphi'= (\vee \ \ \ P   \ \  \,\{\wedge \ \ (\vee \ \ \neg P 
   \ \ R   )  \ \
\{\wedge \ \ Q   \ \   (\vee \ \  P   \ \ \neg S  ) \, \} \, \}  )$ 

\end{itemize}
\noindent    
All     disjunctions of 
$\varphi$, i.e. $(\vee  \ \, \neg P   \ \, R   )$, 
$(\vee \  P    \  \neg S  )$ and   $\varphi$ itself (By Definition \ref{def:sub-for},  
$\varphi$ is  a sub-formula of   $\varphi$), have exactly one non-negative disjunct; 
so $\varphi$  is Horn-NC.
Yet, $\varphi' =(\vee \  P   \ \phi)$,
 $\phi$  being  non-negative, has two
non-negative disjuncts; thus  $\varphi'$ is not Horn-NC.  \qed
\end{example}

 \subsection{Detailed Definition of  $\mathcal{\overline{H}}$}

\noindent Before giving a  fine-grained definition of $\mathcal{\overline{H}}$,  we individually  
and inductively specify: 

\vspace{.2cm}
$\bullet$          Horn-NC conjunctions,    in Lemma \ref{def:HNCconjunc}, and

\vspace{.08cm}
$\bullet$     Horn-NC disjunctions,   in Lemma \ref{def:disjunHNC},

 \vspace{.2cm}
\noindent  and subsequently,   we {\em compactly}  specify   $\mathcal{\overline{H}}$   
 by  merging  both specifications.

\vspace{.15cm}
Just as conjunctions of Horn  clausal  formulas     are    Horn too, 
 likewise   conjunctions of Horn-NC formulas are Horn-NC too, which is straightforwardly formalized next.
  
  \begin{lemma} \label{def:HNCconjunc} 
  Conjunctions of   
      Horn-NC formulas  are   Horn-NC as well, formally:
  $$\{\wedge \  \varphi_1   \,\ldots\,  \varphi_{i} \ldots \,\varphi_k\} \in \mathcal{\overline{H}}
  \mbox{\em \ \ iff \ \ for }   1 \leq i   \leq  k, \ \varphi_i  \in \mathcal{\overline{H}}.$$
  \end{lemma}
  
\begin{niceproof} It is obvious that if all  sub-formulas $\varphi_{i}$ 
individually verify Definition \ref{theorem:visual}
so does a conjunction thereof, and vice versa.
\end{niceproof}

\begin{example}  \label{ex:Semant-Conjunt} If  $H$ is   Horn, 
 $\phi_1$ is $ \varphi_1 $ from Example \ref{exam:simple} and 
 $\phi_2$ is  $ \varphi $ from  Example \ref{Ex:morecomplex}, 
 i.e. $\phi_1$ and $\phi_2$ are  Horn-NC, then for instance  
$\varphi_1=\{\wedge \ H \ \phi_1 \    \phi_2  \}$ is Horn-NC. \qed
\end{example}

In order to give now a detailed  definition of  $ \mathcal{\overline{H}} $,  we    
    verify that  Definition \ref{theorem:visual}    can be 
  equivalently   reformulated in the next inductive manner:
 {\em ``an NC is Horn-NC if all its   disjunctive sub-formulas have any number 
 of negative disjuncts and  one disjunct is Horn-NC".}
 This leads  to the next
    formalization and  statement.

\begin{lemma} \label{def:disjunHNC}  
A NC disjunction $\varphi=(\vee \ \varphi_1 \ldots \varphi_i \ldots  \varphi_k)$ 
with  $k \geq    1$ disjuncts 
pertains to $ \mathcal{\overline{H}} $  iff  
$ \varphi $ has  $k-1$  negative   disjuncts and one Horn-NC disjunct,    
 formally 
$$\varphi=(\vee \ \varphi_1 \ldots \varphi_i \ldots \varphi_k) \in \mathcal{\overline{H}} \mbox{\em \ \ iff} 
  \quad \exists i \ \mbox{\em s.t.} \ \varphi_i \in \mathcal{\overline{H}} 
\ \ \mbox{\em and} \  \ \forall j \neq i,  \varphi_j \in \mathcal{N}^-.$$ 
\end{lemma}

  \begin{niceproof} See Appendix.
  \end{niceproof}

\noindent The next claims follow trivially from Lemma \ref{def:disjunHNC}:
 
 \vspace{.2cm}           
  $\bullet$  Horn clauses are  non-recursive Horn-NC disjunctions.
  
  \vspace{.1cm}        
  $\bullet$ NC disjunctions with all negative disjuncts are Horn-NC.
 
 \vspace{.1cm}
  $\bullet$  NC disjunctions with $ k \geq    2$      non-negative disjuncts are not Horn-NC.  
 
\vspace{.2cm}
\noindent Next, we first reexamine, bearing Lemma \ref{def:disjunHNC} in mind,
 the  formulas from Example  \ref{exam:simple}, included in  Example \ref{exam:NF},  
 and  then those from Example \ref{Ex:morecomplex}, included in Example \ref{ex:disjunHNC}.

\begin{example} \label{exam:NF} Below we analyze   $\varphi_1$  and $\varphi_2$  from  Example \ref{exam:simple}.

\begin{itemize}

\item    $\varphi_1 = (\vee  \   \  \{\wedge \  \ \neg Q   \ \  \neg S  \} \  \ 
\{\wedge \  \ R   \   \  P   \,\} \,).$  

-- Clearly $\{\wedge \  \ \neg Q  \ \  \neg S  \} \in \mathcal{N}^-$.   

-- By Lemma \ref{def:HNCconjunc},
$\{\wedge \  \ R   \   \ P   \,\} \in \mathcal{\overline{H}}$.

-- By Lemma \ref{def:disjunHNC},  $\varphi_1 \in \mathcal{\overline{H}}$.

\item   $\varphi_2 = (\vee   \   \   \{\wedge  \ \  \neg Q   \   \ S   \} \  \ 
\{\wedge \  \ R   \  \ \neg P   \,\} \,)$.

-- Obviously $\{\wedge  \ \  \neg Q   \   \ S   \} \notin \mathcal{N}^-$ 
and \ $\{\wedge \  \ R   \  \ \neg P   \,\} \notin \mathcal{N}^-$.

-- According to Lemma \ref{def:disjunHNC}, 
  $\varphi_2 \notin \mathcal{\overline{H}}$. \qed

\end{itemize}
    
\end{example}

\begin{example} \label{ex:disjunHNC} Consider again     $\varphi$ 
 and $\varphi'$  from Example \ref{Ex:morecomplex}  and 
     recall that $\varphi'$ results from $\varphi$ by  
  just  switching  its literal  $\neg P $ for $P $.   
%
%
 Below we    check  one-by-one
 whether or not the sub-formulas of both  $\varphi$ and  $\varphi'$ are in $\mathcal{\overline{H}}$. 
\begin{itemize}

\item By  Lemma \ref{def:disjunHNC}, $(\vee  \ \ \neg P   \ \ R   ) \in \mathcal{\overline{H}}$.

\item By  Lemma \ref{def:disjunHNC},  $(\vee \ \ P   \  \  \neg S )  \in \mathcal{\overline{H}}$.

\item  By  Lemma  \ref{def:HNCconjunc}, 
 $\{\wedge \ \ Q   \ \   (\vee \  \ P  \ \  \neg S  ) \, \}\in \mathcal{\overline{H}}$. 
 
\item  By  Lemma  \ref{def:HNCconjunc}, 
 $\phi=\{\wedge  \  \ (\vee  \ \ \neg P   \ \ R   )  \ \ 
\{\wedge \ \ Q   \ \   (\vee \ \ P   \ \  \neg S) \, \} \, \} \in \mathcal{\overline{H}}$.
 

\item Using previous formula $\phi$, we have $\varphi= (\vee \ \neg P   \ \,\phi \,)$. 

 \hspace{.5cm} -- Since $\neg P   \in \mathcal{N}^-$
and   $\phi \in \mathcal{\overline{H}}$, by Lemma \ref{def:disjunHNC}, $\varphi \in \mathcal{\overline{H}}$.
    
\item  The second formula in  Example \ref{Ex:morecomplex}  is 
$\varphi'=(\vee \   P   \ \phi \,)$. 

\hspace{.5cm} -- Since $P  ,\,\phi \notin \mathcal{N}^-$,   by Lemma \ref{def:disjunHNC},  $\varphi' \notin \mathcal{\overline{H}}$. \qed

\end{itemize} 
\end{example}

\noindent   By using   Lemmas  \ref{def:HNCconjunc}  and    \ref{def:disjunHNC}, the class  $\mathcal{\overline{H}}$ is syntactically, compactly and inductively defined as follows.

\begin{definition} \label{def:syntacticalNC}   
We  define the set    $\mathcal{\widehat{H}}$  
over the set of propositional variables $ \mathcal{P} $ as the smallest set such that the
conditions below hold,  where
 $k \geq    1$ and $\mathcal{L}$ is the set of   literals.
\begin{itemize}

\item [(1)] \ $\mathcal{L} \subset \mathcal{\widehat{H}}.$    \hspace{10.05cm}   

\item [(2)]  \  If  \ $\forall i, \,\varphi_i  \in \mathcal{\widehat{H}} \ \  \mbox{then} \  \
   \{\wedge \  \varphi_1   \,\ldots\, \varphi_{i} \ldots  \varphi_k\} 
   \in \mathcal{\widehat{H}}.$ \hspace{3.27cm}  

\item  [(3)]  \  If \  $\varphi_i \in \mathcal{\widehat{H}}$ \  and \  
 $\forall j \neq i$,  $\varphi_j \in \mathcal{N}^- \  \ \mbox{then} \ \ (\vee \ \varphi_1 \ldots \varphi_i \ldots   \,\varphi_k) 
 \in \mathcal{\widehat{H}}.$ \hspace{.24cm}  
 
\end{itemize}

\end{definition}

\begin{theorem} \label{th:HNCequality}  We have that \ $\mathcal{\widehat{H}}=\mathcal{\overline{H}}$.
\end{theorem}

\begin{niceproof} See Appendix.
\end{niceproof}

  Theorem \ref{th:HNCequality} below   states   that   $\mathcal{\widehat{H}}$ 
and  $\mathcal{\overline{H}}$ indeed coincide, namely  Definition \ref{def:syntacticalNC}  
is the   recursive and compact definition of the class 
$\mathcal{\overline{H}}$ of Horn-NC formulas. Besides, inspired  by    Definition  \ref{def:syntacticalNC},  in \cite{Imaz2021horn}
  a  linear algorithm  is designed 
  that recognizes whether a given  
 NC    $ \varphi $ is Horn-NC and,   such property is  inherited   
 by the possibilistic Horn-NC  formulas. 

\begin{example} \label{ex:ExamComplet} Viewed from Definition \ref{def:syntacticalNC}, we analyze   $\varphi$ and $\varphi'$ from Example \ref{ex:disjunHNC}: 
\begin{itemize}
\item By   ({3}),   $(\vee  \ \ \neg P   \ \ R   ) \in \mathcal{\overline{H}}$. 

\item By   ({3}), $(\vee \ \ P   \ \  \neg S  )  \in \mathcal{\overline{H}}$. 

\item  By    ({2}),  $\{\wedge \ \ Q   \ \ (\vee \  P  \ \  \neg S  ) \, \}
 \in \mathcal{\overline{H}}$. 

\item  By   ({2}), 
 $\phi=\{\wedge  \  \ (\vee  \ \ \neg P   \ \ R   )  \ \ 
\{\wedge \ \ Q   \ \   (\vee \ \ P  \ \  \neg S  ) \, \} \, \} 
\in \mathcal{\overline{H}}$.
 
\item By  ({3}), $\varphi= (\vee \ \ \neg P  \ \ \phi \,) \in \mathcal{\overline{H}}$

\item  By   ({3}),
$\varphi'=(\vee \ \ P  \ \ \phi \,) \notin \mathcal{\overline{H}}$. \qed

\end{itemize}
\end{example}

\begin{example} \label{ex:nested} If we assume that 
 $\varphi_1$, $\varphi_2$ and  $\varphi_3$ are negative  and 
 $\varphi_4$ and $\varphi_5$ are  Horn-NC, then  according to Definition \ref{def:syntacticalNC},   
 four examples of  Horn-NC    formulas  follow. 
\begin{itemize}

 \item  By    ({3}), \ $\varphi_6=(\vee \ \  \varphi_1 \ \  \varphi_4 ) \in \mathcal{\overline{H}}$.

 \item  By    ({2}), \ $\varphi_7=\{\wedge \ \  \varphi_1 \ \  \varphi_5 \ \ \varphi_6 \} 
 \in \mathcal{\overline{H}}$.

 \item  By    ({3}), \ $\varphi_8=(\vee \ \ \varphi_1 \  \ \varphi_2 \ \  
   \varphi_7) \in \mathcal{\overline{H}}$.
   
\item By  ({2}), \ $\varphi_9=\{\wedge \ \  \varphi_6 \ \  \varphi_7 \ \ \varphi_8 \} 
 \in \mathcal{\overline{H}}$. \qed

\end{itemize}

\end{example}

\noindent Next, we analyze a more complete example, concretely $\varphi  $ from the Introduction.

\begin{example} \label{ex:introduction} Let us take $ \varphi $ below, wherein 
$ \phi_1, \phi_2$ and  $\phi_3$ are NC formulas:  

\vspace{.25cm}
$\varphi=\{\wedge \ \,   P   \ \,  (\vee \ \ \neg Q   \ \ \{\wedge \ \ 
(\vee \ \ \neg P  \ \ \neg Q  \ \ R \,) \ 
\ (\vee \ \ \phi_1 \ \ \{\wedge \ \ \phi_2  \ \  \neg P \,\} \, )
 \ \, Q \,\} \,) \ \  {\phi_3} \ \}.$

\vspace{.25cm}
\noindent The  disjunctions of $ \varphi $ and the proper $ \varphi $ can be rewritten as follows: 
\begin{itemize}
\item $\psi_1= (\vee \  \ \neg P  \  \ \neg Q  \ \  R \,) $.

\item   $\psi_2= (\vee \ \ \phi_1 \  \ \{\wedge \ \ \phi_2  \ \ \neg P \} \, ) $.

\item  $\psi_3=(\vee \ \ \neg Q   \ \ \{\wedge \ \ \psi_1 \ \ 
\psi_2 \ \, Q   \,\} \,) .$

\item  $ \varphi = \{\wedge \ \  P  \ \ \psi_3 \ \ \phi_3 \,\}.$

\end{itemize}

\noindent  We analyze one-by-one such disjunctions and finally the proper $ \varphi $:
\begin{itemize}

\item  $ \psi_1$: \ Trivially, $ \psi_1 $ is Horn, so $ \psi_1 \in  \mathcal{\overline{H}}$. 

\item  $ \psi_2 $: \
$ \psi_2 \in \mathcal{\overline{H}}$  \  if \ $ \phi_1, \phi_2 \in \mathcal{\overline{H}} $ 
\ and  if at least one of   $ \phi_1$ or $ \phi_2$ is negative.

\item  $ \psi_3 $: \ $ \psi_3 \in \mathcal{\overline{H}}$  
   \ if \    $\psi_2 \in \mathcal{\overline{H}}$ (as $ \psi_1 \in \mathcal{\overline{H}}$).
 
\item \ $ \varphi $: \   
 $ \varphi \in \mathcal{\overline{H}}$ \  if \ $ \psi_2, \phi_3 \in \mathcal{\overline{H}}$
 (as $ \psi_3 \in \mathcal{\overline{H}}$  \  
if  \   $\psi_2 \in \mathcal{\overline{H}}$).
\end{itemize}

\noindent Summarizing the conditions  on   $ \varphi $ and on $ \psi_2 $, we have that:
 
 \vspace{.2cm} 
\noindent \   $\bullet$ $ \varphi $ is Horn-NC  \underline{if}   $\phi_3$,  $\phi_1$ and $ \phi_2$  
 are  Horn-NC     \underline{and  if at least one of}   $ \phi_1 $   or $ \phi_2 $ is negative.

\vspace{.2cm} 
\noindent If we consider that $ \varphi $ implicitly verifies  Definition \ref{theorem:visual} (all sub-formulas of 
a Horn-NC  are Horn-NC),  then we 
  conclude that $ \varphi $ is Horn-NC   if at least 
one of $ \phi_1 $   or $ \phi_2 $ is negative. \qed

\end{example}

\subsection{Properties of the Class  $ \mathcal{\overline{H}} $  }

An important feature  of Horn-formulas  is the following:

\begin{theorem} \label{the:HNCtoHorn} Applying $ \vee/\wedge $-distributivity to a Horn-NC $ \varphi $ results in a Horn formula.
\end{theorem}

\begin{niceproof}
See Appendix.
\end{niceproof}

We  already saw that
syntactically  $ \mathcal{\overline{H}} $ subsumes  $\mathcal{{H}} $,  but besides, 
   $ \mathcal{\overline{H}}  $ is semantically  related
to  $ \mathcal{{H}}  $ as Theorem \ref{theo:relation-Horn-Horn} claims.

\begin{theorem} \label{theo:relation-Horn-Horn} 
    $ \mathcal{\overline{H}} $ and $\mathcal{{H}} $ are semantically equivalent: each formula in a class is logically equivalent to some formula in the other class.
\end{theorem}

\begin{niceproof} By Theorem \ref{the:HNCtoHorn}, for every 
$ \varphi \in \mathcal{\overline{H}}$ there exists $ H \in \mathcal{H} $ such that
$ \varphi \equiv H $. The converse follows from the fact that $ \mathcal{H} \subset \mathcal{\overline{H}}$.
\end{niceproof}

The next theorem make it explicit how the  classes Horn-NC and NC
are related.

\begin{theorem}  \label{theo:relation-NC-H}   
$ \mathcal{\overline{H}} $  contains the next  NC fragment:   if 
applying  $ \wedge /\vee $ distributivity 
 to an NC formula $ \varphi $ results in
a Horn formula, then 
 $ \varphi $ is in  $ \mathcal{\overline{H}} $. 
\end{theorem}

\begin{niceproof}
See Appendix.
\end{niceproof}

The syntactical and semantical properties exhibited by   $ \mathcal{\overline{H}} $ 
affirmed  by the last three theorems suggest that   
$ \mathcal{\overline{H}} $ is a kind of NC
analogous of the standard Horn class $ \mathcal{{H}}  $.

\subsection{The Definition of  $\mathcal{\overline{H}}_\Sigma$}
\vspace{.1cm}
Finally, from $ \mathcal{\overline{H}} $, we straightforwardly define the class $\mathcal{\overline{H}}_\Sigma $ of possibilistic Horn-NC bases.


\begin{definition} \label{def:HNCKnowledgeBase}
A possibilistic Horn-NC  formula is a pair $ \langle\varphi : \alpha \rangle$, where  
 $ \varphi \in \mathcal{\overline{H}} $ and  $ \alpha \in (0 \ 1] $.
A  possibilistic Horn-NC  base  $ \Sigma $ 
is a set 
of  possibilistic Horn-NC formulas.   $ \mathcal{\overline{H}}_\Sigma $ 
 denotes the class of  possibilistic Horn-NC bases.
\end{definition}

\begin{example} We take the next Horn-NCs:   $ \varphi $ from Example  \ref{Ex:morecomplex}
and  $ \varphi_9 $ from Example \ref{ex:nested}. By $ \varphi' $ we denote   
 $ \varphi $  from Example \ref{ex:introduction} considering
that  the specified conditions 
 warranting that $ \varphi $  is Horn-NC are met.   An example of a possibilistic Horn-NC   base is:
$$\{\,\langle P \, : \, {\bf .8}\rangle, \  
\langle \varphi \, : \, {\bf .8} \rangle,  \  \langle \varphi_9 \, , \, {\bf .5} \rangle, \ 
\langle \varphi' \, , \, {\bf .9} \rangle, \ \langle \neg Q \, , \, {\bf .1} \rangle \,  \}. $$ 
\end{example}

\begin{corollary}    $ \mathcal{\overline{H}}_\Sigma $ and $\mathcal{{H}}_\Sigma$ are semantically equivalent: each formula in a class is  equivalent to some formula in the other class. 
\end{corollary}

\begin{niceproof} It follows  from the definitions of  $ \mathcal{\overline{H}}_\Sigma $ 
and $\mathcal{{H}}_\Sigma$ and  Theorem  \ref{theo:relation-Horn-Horn}.
\end{niceproof}

\begin{corollary} $ \mathcal{\overline{H}}_\Sigma $ is the next $ \mathcal{NC}_\Sigma $ fragment: if 
  $ \langle \varphi : \alpha \rangle \in \mathcal{NC}_\Sigma$ and 
  applying  $ \wedge /\vee $ distributivity 
 to    $ \varphi $ results in   a Horn  formula,
 then   $ \langle \varphi : \alpha \rangle 
 \in \mathcal{\overline{H}}_\Sigma$.
\end{corollary}

\begin{proof}  It follows from the definitions of $ \mathcal{\overline{H}}_\Sigma $ and 
$ \mathcal{NC}_\Sigma $ and Theorem \ref{theo:relation-NC-H}.
\end{proof}

\noindent {\bf Remark.} Since $ \mathcal{\overline{H}} $ is the NC analogous
of $ \mathcal{{H}} $ so is $ \mathcal{\overline{H}}_\Sigma $ of $ \mathcal{{H}}_\Sigma $.

\section{Possibilistic NC Unit-Resolution {\em   $\mathcal{UR}_\Sigma$}} \label{sec:Non-Clausal-Unit-resolution}

Possibilistic clausal resolution was defined in the  
1980s \cite{DuboisP87,DuboisP90} but
   possibilistic  non-clausal resolution has not been proposed yet.  
   This section is   a step forward towards its definition as we  define 
 Possibilistic Non-Clausal Unit-Resolution, denoted 
  $\mathcal{UR}_\Sigma$, which  
  is  an  extension of the  calculus presented  in \cite{Imaz2021horn}  for propositional logic.
%
    The  main inference rule of {\em   $\mathcal{UR}_\Sigma$}
  is  called  {\em UR$_{\Sigma}$}, and   
    while the other rules in $\mathcal{UR}_\Sigma$ are simple,  
     {\em UR$_\Sigma$} is somewhat involved and so is presented  progressively as follows: 

\vspace{.15cm}
$\bullet$  for  quasi-clausal    Horn-NC    bases 
in Subsection \ref{subsec:Quasi-ClausalHNCs};  and

\vspace{.1cm} 
$\bullet$ for   nested  Horn-NC    bases in Subsection \ref{subsec:GeneralHNCs}.

\vspace{.15cm}
Afterwards, \underline{Subsection  \ref{subsec:TheCalculus}} describes  
 $\mathcal{UR}_\Sigma$, which besides {\em UR$_{\Sigma}$},
comprises: (a) the propositional  rule
  {\em UR$_P$},  which is  {\em UR$_{\Sigma}$}  adapted to propositional logic,
(b) the propositional simplification rules, and (c) the possibilistic rules {\em MinD}
and {\em MaxN}.
\underline{Subsection \ref{subsect:find}} gives the algorithm to obtain 
$ \mbox{Inc}(\Sigma) $ which combines {\em $\mathcal{UR}_{\Sigma}$} with   $ \alpha $-cuts
of the input $ \Sigma $. To  end this section,  \underline{Subsection \ref{subsub:2inferencerules}} gives   two further inferences rules, 
 not needed for warranting the completeness of $\mathcal{UR}_\Sigma$.   
 We recall  that 
  $ \bot $ and   $ (\vee) $  are equivalent (see Definition 
  \ref{def:interpretation}).

\subsection{ Quasi-Clausal NC Unit-Resolution} \label{subsec:Quasi-ClausalHNCs}

\vspace{.05cm}
We   start our presentation with propositional formulas and then
 switch to possibilistic bases. Assume  propositional formulas 
with the  quasi-clausal pattern  below  in which 
 $ {\textcolor{red} {\ell}} $ and 
 $ {\textcolor{blue} {\neg {\ell}}} $ are any literal
 and its negated one,  and the $ \varphi $'s and the $ \phi $'s are formulas:  
$$\{{\textcolor{red} {\wedge}}  \ \varphi_1  \, \ldots  \, \varphi_{l-1} 
 \ {\textcolor{red} {\ell}} \  \varphi_{l+1}  \, \ldots  \, \varphi_{i-1} \
({\textcolor{blue} {\vee}} \ \, \phi_1 \, \ldots \, \phi_{j-1} \ 
{\textcolor{blue} {\neg {\ell}}} \ \phi_{j+1} \, \ldots \, \phi_k) \ \varphi_{i+1} \, \ldots \, \varphi_n \}$$

\noindent We say that    these formulas  are  quasi-clausal because   
if the $\varphi$'s and   $\phi$'s were  clauses and  literals, respectively, then 
such formulas would  be  clausal.  It is not hard to see that  
a quasi-clausal formula is equivalent to a formula of the kind:  
$$\{{\textcolor{red} {\wedge}}  \ \varphi_1  \, \ldots  \, \varphi_{l-1} 
 \ {\textcolor{red} {\ell}} \  \varphi_{l+1}  \, \ldots  \, \varphi_{i-1} \ 
({\textcolor{blue} {\vee}} \  \phi_1 \, \ldots \, \phi_j  
\, \phi_{j+1} \, \ldots \, \phi_k) \ \varphi_{i+1} \, \ldots \, \varphi_n \}$$

\noindent and thus, one can derive   the next   simple  
 inference rule for propositional formulas:  
\begin{equation} \label{eq:simpleinf}
\frac{{\textcolor{red} {\ell}}     \ {\textcolor{red} {\wedge}} \
({\textcolor{blue} {\vee}} \ \ \phi_1 \, \ldots \, \phi_j \ 
{\textcolor{blue} {\neg {\ell}}} \ \phi_{j+1} \, \ldots \, \phi_k) }
{({\textcolor{blue} {\vee}} \ \ \phi_1 \, \ldots \, \phi_j  
 \ \phi_{j+1} \, \ldots \, \phi_k) }
\end{equation}

 Notice that for clausal formulas, 
  {\em Rule (\ref{eq:simpleinf}) coincides with  
clausal unit-resolution.} 

\vspace{.15cm}
Now let us switch to   possibilistic   bases. The    setting in which
NC unit-resolution is applicable is  when 
      $ \Sigma $   has two Horn-NC formulas such that one is a unit clause
$\langle {\textcolor{red} {\ell}}   : \alpha\rangle$  and the other   has the   
 pattern:
 $\langle \, \{\wedge \  \varphi_{1} \, \ldots \, \varphi_i \ ({\textcolor{blue} {\vee}} \ \, \phi_1 \, \ldots \, \phi_{j-1} \ 
{\textcolor{blue} {\neg {\ell}}} \ \phi_{j+1} \, \ldots \, \phi_k) \ \varphi_{i+1} \, \ldots \, \varphi_n \} : \beta \, \rangle$.
Namely, as $ \Sigma $ is an implicit conjunction of its formulas, then $ \Sigma $ contains a conjunction:
\begin{equation}\label{eq:quasi-clausal-possi-unit}
\langle {\textcolor{red} {\ell}}   : \alpha\rangle 
\wedge \langle \, \{\wedge \  \varphi_{1} \, \ldots \, \varphi_{i-1} \ ({\textcolor{blue} {\vee}} \ \, \phi_1 \, \ldots \, \phi_{j-1} \ 
{\textcolor{blue} {\neg {\ell}}} \ \phi_{j+1} \, \ldots \, \phi_k) \ \varphi_{i+1} \, \ldots \, \varphi_n \} : \beta \, \rangle
\end{equation}

In this setting and   by using   Min-Decomposability, i.e.   $ N(\varphi \wedge \psi) = \mathrm{min}(N(\varphi), N(\psi))  $ (Definition \ref{def:inderences}),  
one can easily derive the next   possibilistic inference:
 
\begin{equation} \label{eq:firstposs}
\frac{ \langle {\textcolor{red} {\ell}} \, : \, \alpha\rangle  \  {\textcolor{red} {\wedge}}  \
\langle\,({\textcolor{blue} {\vee}} \ \ \phi_1 \, \ldots \, \phi_j \ 
{\textcolor{blue} {\neg {\ell}}} \ \phi_{j+1} \, \ldots \, \phi_k)\, : \,\beta\rangle }
{\langle\,({\textcolor{blue} {\vee}} \ \ \phi_1 \, \ldots \, \phi_j  
 \ \phi_{j+1} \, \ldots \, \phi_k) \, : \, \mathrm{\bf min}\{\alpha,  \beta\}  \,\rangle }
\end{equation} 

The soundness of   (\ref{eq:firstposs})  follows immediately 
from the property   Min-Decomposability.
If  $\mathcal{D}({\textcolor{blue} {\neg {\ell}}})$
stands for     
$(\vee \  \phi_1 \, \ldots \, \phi_j \, \phi_{j+1} \, \ldots  \, \phi_n)$,
then the previous  rule    can be concisely rewritten as: 
 
\begin{equation} \label{eq:without-conjun}
\frac{ \langle{\textcolor{red} {\ell}} \, : \, \alpha\rangle \ \, {\textcolor{red}\wedge} \ \, 
\langle \,({\textcolor{blue} {\vee}} \ \, {\textcolor{blue} {\neg {\ell}}} \ \, \mathcal{D}({\textcolor{blue} {\neg {\ell}}}) \,)\, : \beta \rangle }
{  \langle\mathcal{D}({\textcolor{blue} {\neg {\ell}}}) \, :
\, \mathrm{\bf min}\{\alpha,\beta\} \,\rangle}{{\mbox{\,}}}
\end{equation}

 Notice that   the previous rule amounts to substituting the formula 
 referred to by the right conjunct in the numerator with the formula in the denominator, 
 and in practice, 
to just  eliminate $ {\textcolor{blue} {\neg {\ell}}} $ and update the necessity weight. 
 Let us illustrate these notions.

\begin{example} \label{ex:withoutconj} Let $ \Sigma $  be  a  
base  including $ \varphi_1 $ and $ \varphi_2 $ below,  where $ \phi $
is a   formula:
\begin{itemize}
\item  $  \varphi_1=\langle     {\textcolor{red} P}  : {\bf .8} \rangle$ \   \quad 
 
\item  $\varphi_2=  \langle \, \{{\textcolor{red} {\wedge}}   \ \phi  \     
(\vee   \ \  \neg {R} \ \  {\textcolor{blue} {\neg {P}}} \ \ S \,) \ \ 
(\vee \ \ S \ \   \{\wedge \ \ \neg {Q}  \ \ \neg {P} \,\} \, ) \ \ R \,\} \, : \, {\bf .6} \rangle.$

\end{itemize}
Taking  $ {\textcolor{red} P} $ in $ \varphi_1 $ and the 
 left-most  ${\textcolor{blue} {\neg {P}}}$ in $\varphi_2$,
we have    $\mathcal{D}({\textcolor{blue} {\neg {P}}})=(\vee \  \neg {R}  \ S )$, 
 and by applying   
$$\Sigma \leftarrow \Sigma \, \cup \, \langle\, 
\{{\textcolor{black} {\wedge}} \ \, \phi  \ \,    
(\vee   \ \  \neg {R}  \ \ S \,) \ \ 
(\vee \ \ S \ \   \{\wedge \ \ \neg {Q}  \ \ \neg {P} \,\} \, ) \ \ R \,\} 
\, : \, {\bf .6} \, \rangle.$$ 
Rule (\ref{eq:without-conjun}) to $ \varphi_2 $,  
the above  formula is deduced and added to the base $ \Sigma $. \qed 
\end{example}

  We now extend our analysis from  formulas with pattern 
$\langle \, ({\textcolor{blue} {\vee}} \ \, {\textcolor{blue} {\neg {\ell}}} \ \, 
\mathcal{D}({\textcolor{blue} {\neg {\ell}}}) \,) : \beta \rangle$ to formulas    
with pattern  $\langle \, ({\textcolor{blue} {\vee}} \  \ \mathcal{C}({\textcolor{blue} {\neg {\ell}}})\ \ 
 \mathcal{D}({\textcolor{blue} {\neg {\ell}}})\,  ) : \beta \, \rangle$ 
wherein  $\mathcal{C}({\textcolor{blue} {\neg {\ell}}})$
is  the maximal sub-formula  that becomes false when 
$ {\textcolor{blue} {\neg {\ell}}} $ is false, namely,    $\mathcal{C}({\textcolor{blue} {\neg {\ell}}})$ is:   $(i)$
the maximal  sub-formula, and   $(ii)$ equivalent to  a conjunction of the kind 
 ${\textcolor{blue} {\neg {\ell}}} \,\wedge \, \psi $. In other words,    $\mathcal{C}({\textcolor{blue} {\neg {\ell}}})$ is the maximal sub-formula "conjunctively linked" to  ${\textcolor{blue} {\neg {\ell}}}$.

\vspace{.15cm} 
\noindent For instance
 If the input base   $ \Sigma $ contains  
  $\langle {\textcolor{red} {{\ell}}} : \alpha \rangle$ and another formula of the kind:
 $$ \langle \ (\vee \ \varphi_1 \quad  \{\wedge \ \phi_ 1 \ \{\wedge \ {\textcolor{blue} {\neg {\ell}}} \  (\vee \ \ \phi_2 \ \neg {P} \,)\} \  
 \phi_3\} \quad \varphi_2 ): \beta \ \rangle$$
 
 \vspace{.1cm}
\noindent then  $\mathcal{C}({\textcolor{blue} {\neg \mathcal{\ell}}})=\{\wedge \ \phi_ 1 \ \{\wedge \ {\textcolor{blue} {\neg {\ell}}} \ 
 (\vee \ \ \phi_2 \ \neg {P} \,)\} \ \phi_3\,\}$ because: 

 
 \vspace{.15cm} 
 $(ii)$ \quad \
$\mathcal{C}({\textcolor{blue} {\neg \mathcal{\ell}}})$  is equivalent to 
${\textcolor{blue} {\neg {\ell}}} \wedge \psi= 
{\textcolor{blue} {\neg {\ell}}} \wedge \{\wedge \ \phi_ 1 \   
 (\vee \ \phi_2 \ \neg {P} \,) \  \phi_3\}$; and 
 
 \ $(i)$ \quad \,  no  sub-formula $\mathcal{C}'({\textcolor{blue} {\neg \mathcal{\ell}}}) $ bigger than   $\mathcal{C}({\textcolor{blue} {\neg \mathcal{\ell}}})$ verifies $\mathcal{C}'({\textcolor{blue} {\neg \mathcal{\ell}}})\equiv {\textcolor{blue} {\neg {\ell}}} \wedge \psi'$. 
 
 \vspace{.15cm}
\noindent  Clearly, if $ {\textcolor{blue} {\neg {\ell}}} $
 becomes false so does  $\mathcal{C}({\textcolor{blue} {\neg \mathcal{\ell}}})=\{\wedge \ \phi_ 1 \ \{\wedge \ {\textcolor{blue} {\neg {\ell}}} \ 
 (\vee \ \ \phi_2 \ \neg {P} \,)\} \ \phi_3\,\}$.

\vspace{.1cm}
\noindent {\bf Remark. }  $\mathcal{C}({\textcolor{blue} {\neg {\ell}}})$ contains  ${\textcolor{blue} {\neg {\ell}}}$
but  $\mathcal{D}({\textcolor{blue} {\neg {\ell}}})$ excludes  it.

\begin{example}\label{ex:with-conjuction}  The   formula given below is an extension of 
 $\varphi_2  $  
from Example \ref{ex:withoutconj}, in which,
by clarity, its previous sub-formula 
$ (\vee \  S \    \{\wedge \  \neg {Q}  \  \neg {P} \,\} \, ) $ 
is denoted $ \phi_1 $ and the previous literal $ {\textcolor{blue} {\neg {P}}} $
 is now extended to the formula
  $ \{\wedge \  {\textcolor{blue} {\neg {P}}} \  (\vee \  S  \  \neg {R} ) \, \} $
   including $ {\textcolor{blue} {\neg {P}}} $:
$$ \varphi = \langle\, \{{\textcolor{red} {\wedge}} \  \phi  \     
(\vee   \   \neg {R} \  \{\wedge \  {\textcolor{blue} {\neg {P}}} \  (\vee \  S  \  \neg {R} ) \, \} \ \ S \,) \ \ 
\phi_1 \ \ R \,\} \, : \, {\bf .6} \,\rangle$$
Taking the left-most
 ${\textcolor{blue} {\neg {P}}}$ ($ \phi_1 $ has also another literal ${\neg {P}}$),
  $\varphi$ has a sub-formula with pattern  
   $ ({\textcolor{blue} {\vee}} \   
\mathcal{C}({\textcolor{blue} {\neg {P}}}) \ \mathcal{D}({\textcolor{blue} {\neg {P}}})\, )$, in which  
$\mathcal{C}({\textcolor{blue} {\neg {P}}})=\{\wedge \  {\textcolor{blue} {\neg {P}}} \ \ (\vee \ S \ \neg {R} ) \, \}$ and $\mathcal{D}({\textcolor{blue} {\neg {P}}})= 
(\vee \ \, \neg {R} \ S).$ \qed
\end{example} 
 
 Regarding the inference rule, we have that when 
   $ \Sigma $ has both a unitary clause 
 $ \langle{\textcolor{red} \ell} \, : \, \alpha \rangle $ and another formula  
  $\langle \varphi : \beta \rangle $ such that $ \varphi $ has 
  the pattern
$  ({\textcolor{blue} {\vee}} \ \, \mathcal{C}({\textcolor{blue} {\neg {\ell}}}) \ \, \mathcal{D}({\textcolor{blue}{\neg {\ell}}}) \,) $, then
the \underline{possibilistic} NC unit-resolution rule  is easily obtained by extending   Rule (\ref{eq:without-conjun})  as follows: 

\begin{equation} \label{eq:with-conjunct}
\frac{ \langle {\textcolor{red} {\ell}} \, : \, \alpha \rangle \ \, {\textcolor{red}\wedge} \ \, 
\langle\,({\textcolor{blue} {\vee}} \ \, \mathcal{C}({\textcolor{blue} {\neg {\ell}}}) \ \, \mathcal{D} ({\textcolor{blue} {\neg {\ell}}}) \,)\, : \beta \,\rangle }
{ \langle \, \mathcal{D}({\textcolor{blue} {\neg {\ell}}} \, ) \, :
\, \mathrm{\bf min}(\alpha,\beta ) \,\rangle}{{\mbox{\,}}}
\end{equation}

The  soundness of   (\ref{eq:with-conjunct}) follows from 
${\textcolor{red} {\ell}} \wedge  \mathcal{C}({\textcolor{blue} {\neg {\ell}}}) \equiv \bot$ and
its proof  is given in Section \ref{sect:ProofsProperties}.
   {\bf Fig. 3} depicts  Rule  (\ref{eq:with-conjunct}) where the left and right trees  represent, respectively,
 the numerator and denominator   of    (\ref{eq:with-conjunct}).
\begin{center}
\begin{adjustbox}{valign=t}

\begin{tikzpicture} [every node/.append style = {text=black}, sibling distance=1.5cm ]
\node {{\textcolor{red}{$\wedge$}}}
    child[line width=.2mm] {node  {$\langle{\textcolor{red} \ell} : \alpha \rangle \ \ $}}
    child[line width=.2mm] {node {$\langle {\textcolor{blue}{\vee}} : \beta \rangle$}
        child[line width=.2mm] {node {$\mathcal{C}({\textcolor{blue}{\neg {\ell}}})$}}
        child[line width=.2mm] {node {$\mathcal{D}({\textcolor{blue}{\neg {\ell}}})$}}
        } 
            ;
\end{tikzpicture}

\end{adjustbox} 
\begin{adjustbox}{valign=t} \hspace{1.5cm} 

\begin{tikzpicture} [every node/.append style = {text=black}, sibling distance=1.8cm ]
\node {{\textcolor{red}{$\wedge$}}}
    child[line width=.2mm] {node  {$\langle{\textcolor{red} \ell} : \alpha \rangle \ \ $}}
    child[line width=.2mm] {node {$\ \ \langle{\textcolor{blue}{\vee}} : 
     \mathrm{\bf min}(\alpha,\beta)\rangle$}
        child[line width=.2mm] {node {$\mathcal{D}({\textcolor{blue}{\neg {\ell}}})$}}
        } 
            ;
\end{tikzpicture}
\end{adjustbox}

\vspace{.2cm}
{\small {\bf Fig. 3.} Depicting  Rule (\ref{eq:with-conjunct}}).

\end{center}

\begin{example}   Rule (\ref{eq:with-conjunct})  
with $  \varphi_1=\langle     {\textcolor{red} P}  : {\bf .3} \rangle$ and with
$ \varphi  $ from Example \ref{ex:with-conjuction} derives:
$$\Sigma \leftarrow \Sigma \, \cup \, \langle\,\{{\textcolor{black} {\wedge}} \ \,  \phi \  \,    
(\vee   \ \,  \neg {R}  \ S \,) \ \ 
\phi_1 \ \ R \,\} 
\, : \, {\bf .3} \,\rangle.$$
\vspace{-.6cm} 
\end{example}


\subsection{Nested  NC Unit-Resolution} \label{subsec:GeneralHNCs}


Coming back to the    almost-clausal  formulas 
expressed in   (\ref{eq:quasi-clausal-possi-unit}) and extending its literal
$ {\textcolor{blue} {\neg {\ell}}} $ to $ \mathcal{C}({\textcolor{blue} {\neg {\ell}}}) $, we 
now rewrite them   compactly as indicated below,  where $ \Pi $ and $ \Pi' $  denote
   a concatenation of formulas,  namely $ \Pi=\varphi_1 \ldots  \varphi_{i-1}$ 
 and 
 $ \Pi'=\varphi_{i+1} \ldots \varphi_n $:
 $$\langle{\textcolor{red} {\ell}} \, : \, \alpha\rangle \wedge   \langle \ \{{\textcolor{red} {\wedge}}  \ \, \Pi
  \ \,
({\textcolor{blue} {\vee}} \   \mathcal{C}({\textcolor{blue} {\neg {\ell}}}) \  
 \mathcal{D}({\textcolor{blue} {\neg {\ell}}}) \,) \ \, \Pi' \} : \beta  \ \rangle$$
 
 We  now analyze the      \underline{nested} Horn-NC bases $ \Sigma $ to which 
 NC unit-resolution  can  be indeed  applied.
 That is,   $ \Sigma $ must have a 
   unit-clause 
 $\langle {\textcolor{red} {\ell}} \, : \, \alpha\rangle  $ 
 and a possibilistic nested   Horn-NC  formula, denoted $\langle \Pi : \beta \rangle$,
 with  a syntactical pattern of the next 
 kind\footnote{The notation $ [ \odot \ \varphi_1 \ldots \varphi_k ] $ was introduced 
 in Definition  \ref{def:NCformulas}, bottom.}:
 $$ \langle  \ [ \odot_1 \ \, \Pi_1 \ \ldots  \ \ 
 [ \odot_k \ \ \Pi_{k} \ \,  
 ({\textcolor{blue} {\vee}} \   \mathcal{C}({\textcolor{blue} {\neg {\ell}}}) \  
 \mathcal{D}({\textcolor{blue} {\neg {\ell}}}) \,) \ \,  \Pi'_k \, ] \ \ 
    \ldots  \  \,  \Pi'_1 \,]   : \beta \ \rangle $$
    
    \noindent where all  the $\Pi_j$'s and $\Pi_j'$'s are concatenations of  formulas, e.g. for the
nesting level $ j, 1 \leq j \leq k$, we have
$\Pi_j=\varphi_{j_1} \ldots \varphi_{j_{i-1}}$  and 
$\Pi'_j=\varphi_{j_{i+1}} \ldots \varphi_{j_{n_j}}$. Since the presence of formulas
in the base $ \Sigma $ means that they are conjunctively linked, then one has:

 \begin{equation} \label{eq:formulaexpresiongeneral}  
 \langle {\textcolor{red} {\ell}} \, : \, \alpha\rangle \wedge \ 
 \langle  \ [ \odot_1 \ \, \Pi_1 \ \ldots  \ \ 
 [ \odot_k \ \ \Pi_{k} \ \,  
 ({\textcolor{blue} {\vee}} \   \mathcal{C}({\textcolor{blue} {\neg {\ell}}}) \  
 \mathcal{D}({\textcolor{blue} {\neg {\ell}}}) \,) \ \,  \Pi'_k \, ] \ \ 
    \ldots  \  \,  \Pi'_1 \,]   : \beta \ \rangle 
 \end{equation}

 By following the same principle that led us to
    Rule   (\ref{eq:with-conjunct}) and taking into account that
 $ N(\varphi_1 \wedge \varphi_2)=\mbox{min}\{N(\varphi_1),N(\varphi_1)\} $,   
one   obtains the nested NC unit-resolution rule:
 
 \vspace{.05cm}
 \begin{equation} \label{eq:generalPN-Horn}
 \frac{ \langle {\textcolor{red} {\ell}} \, : \, \alpha\rangle \wedge \  \langle \, [ \odot_1 \ \Pi_1    \ldots  [ \odot_k  
 {\textcolor{red}  \ \ \Pi_{k} \ \  
 ({\textcolor{blue} {\vee}} \ \  \mathcal{C}({\textcolor{blue} {\neg {\ell}}}) \ \ \mathcal{D}({\textcolor{blue} {\neg {\ell}}}) \,)  
 \ \ \Pi'_k\,] } \ldots  \, \Pi'_1] : \beta \ \rangle}
 { \langle \  [ \odot_1 \ \Pi_1   \ldots  [ \odot_k  
 {\textcolor{blue}  \ \ \Pi_{k} \ \  
  \mathcal{D}({\textcolor{blue} {\neg {\ell}}})  \ \ \Pi'_k\,] }
  \ldots  \, \Pi'_1 ] : \mbox{\bf min}\{\alpha,\beta\} \ \rangle } 
 {\ \mbox{}}
 \end{equation}
  
  \vspace{.05cm}
Recapitulating, Rule (\ref{eq:generalPN-Horn}) indicates that if the Horn-NC    
$ \Sigma $ has two formulas such that one is a unit clause $ \langle {\textcolor{red} {\ell}} \, : \, \alpha\rangle $ and the other  $\langle \Pi: \beta \rangle $ has the pattern of the right conjunct in the  numerator, 
then $ \Pi $ can be {\bf replaced} with the formula in the
 denominator. In practice, applying (\ref{eq:generalPN-Horn})  amounts to just 
   removing
  $\mathcal{C}({\textcolor{blue} {\neg {\ell}}})$  from $ \Pi $ and updating the necessity weight.

\vspace{.15cm}
   We now  denote  $ \Pi $  the right conjunct in the numerator of (\ref{eq:generalPN-Horn})
   and by 
$\Pi \succ ({\textcolor{blue} {\vee}} \  \mathcal{C}({\textcolor{blue} {\neg {\ell}}}) \  \mathcal{D}({\textcolor{blue} {\neg {\ell}}})  \,)$ denote  that $({\textcolor{blue} {\vee}} \ \mathcal{C}({\textcolor{blue} {\neg {\ell}}}) \  \mathcal{D}({\textcolor{blue} {\neg {\ell}}}) \,)$ 
 is a sub-formula of $\Pi$.  
 Rule  (\ref{eq:generalPN-Horn}) above
can be  compacted,  giving rise to a more concise formulation of  
{\em UR$_{\Sigma}$}:

$$\tcboxmath{\frac{\langle \textcolor{red} \ell : \alpha \rangle \ {\textcolor{red} {\wedge}} 
\ \langle \ \Pi \succ ({\textcolor{blue} {\vee}} \ 
 \ \mathcal{C}({\textcolor{blue} {\neg \ell}}) \ \ 
 \mathcal{D}({\textcolor{blue} {\neg \ell}})  \,) : \beta \ \rangle  }
{
\langle \ \Pi \succ  \mathcal{D}({\textcolor{blue} {\neg \ell}})  
: \mathrm{\bf min}\{\alpha,\beta\}  \ \rangle }{\mbox{\,{\em UR$_{\Sigma}$}}}}$$

\vspace{.2cm}
 The soundness of the rule {\em UR$_{\Sigma}$} follows   from    
 $ {\textcolor{red} \ell} \wedge  \mathcal{C}({\textcolor{blue} {\neg {\ell}}}) \equiv \bot$ and is proven in Section \ref{sect:ProofsProperties}. Two simple examples illustrating  how  
{\em UR$_{\Sigma}$} works  are Examples \ref{ex:firstexample} and \ref{ex:secondexample}.
Two  more complete
examples are Examples \ref{ex:BasePossib} and \ref{ex:basepossib-cont} but they employ   other inferences and mechanisms relative to $ \mathcal{UR}_	\Sigma $
given in the remaining  of this section.

\vspace{.1cm}
{\bf Remark.} It is not difficult to check that, for clausal formulas, {\em    UR$_{\Sigma}$}  
{\em coincides  with the} {\em standard possibilistic clausal 
unit-resolution \cite{DuboisP94, DuboisP14}.}
This clausal-like formulation of NC unit-resolution contrasts with the functional-like  one
 of classical NC resolution  handled until now 
 in the literature and presented in  \cite{Murray82} 
 (see also \cite{BachmairG01}).  We believe that our version, as previously said, 
 is more suitable to understand,  implement  and formally  analyze.

\subsection{The Calculus  $\mathcal{UR}_\Sigma$} \label{subsec:TheCalculus}

 Besides {\em UR$_\Sigma$,} the calculus
$\mathcal{UR}_\Sigma$     also 
includes the rules:
 (a) propositional NC unit-resolution, or {\em UR$_P$}, which is   
 {\em UR$_\Sigma$} but applied inside propositional formulas, 
  (b)  the rules   to simplify propositional formulas, and
 (c) the possibilistic rules Min-Decomposability (MinD) and Max-Necessity (MaxN) 
  mentioned in Definition \ref{def:inderences}.

\subsubsection{Propositional NC Unit-Resolution}

A major difference between computing the inconsistency degree of 
   clausal and   non-clausal 
bases is that the unity members of the former, i.e.  clauses, are always consistent, 
while a non-clausal formula can  itself   be inconsistent. 
That is, the input
  $ \Sigma $ can contain a formula $ \langle \Pi : \alpha \rangle $ where
$ \Pi $ is inconsistent, and if so,  $ \langle \Pi : \alpha \rangle $ is equivalent to $ \langle \bot : \alpha \rangle $, 
which brings to:

\begin{proposition} \label{prop:input-propo-form} If 
$ \langle \Pi : \alpha \rangle \in \Sigma$ 
and $ \Pi $ is inconsistent then 
$ \mathrm{Inc}(\Sigma) \geq \alpha $. 
\end{proposition}

\begin{niceproof} By definition 
$ \mathrm{Inc}(\Sigma) = \mathrm{max}\{\beta \,\vert \,
\Sigma_{\geq \beta} \mbox{ is inconsistent} \} $. Trivially if $ \Pi $ is inconsistent 
then $ \bot \in \Sigma_{\geq \alpha} $,
and thus, $\Sigma_{\geq \alpha} $ is inconsistent. So  $ \mathrm{Inc}(\Sigma) \geq \alpha $.
\end{niceproof}

Hence, first of all, the propositional formula  $ \Pi$ of each $ \langle \Pi : \alpha \rangle \in \Sigma$ must be checked for consistency. If $ \Pi $ is   inconsistent, 
then, by definition,  $ \mathrm{Inc}(\Sigma) $ is the maximum of $ \alpha $ and the inconsistency 
degree of the 
strict $ \alpha $-cut of $ \Sigma $. Thus, one can remove from $ \Sigma $ all formulas $ \langle \Pi : \beta \rangle $ such that $ \beta \leq \alpha $
and search whether $ \mathrm{Inc}(\Sigma_{> \alpha}) > 0$.

 \vspace{.1cm}
 The inference rule  {\em  UR$_P$}   testing the consistency of a propositional  $ \Pi $, where
 $ \langle \Pi : \alpha \rangle \in \Sigma$, is easily derived from {\em  UR$_\Sigma$} 
by considering that the conjunction of a unit clause  
${\textcolor{red} {{\ell}}} $ and of 
a   formula  $ \Pi $ containing $ \mathcal{C}({\textcolor{blue} {\neg {\ell}}}) $ happens inside
 $ \Pi $. Thus {\em  UR$_P$} is as follows:
  
  \begin{equation} \label{eq:Propo-UR}
 \tcboxmath{\frac{\langle \ {\textcolor{red} \ell} \ {\textcolor{red} {\wedge}} \ 
\Pi \succ ({\textcolor{blue} {\vee}} \ 
 \ \mathcal{C}({\textcolor{blue} {\neg {\ell}}}) 
 \ \ \mathcal{D}({\textcolor{blue} {\neg {\ell}}})  \,) : \alpha \ \rangle }
{\langle \ \Pi \succ  \mathcal{D}({\textcolor{blue} {\neg {\ell}}}) 
: \alpha \rangle   }{\mbox{\em UR$_P$}}}
\end{equation}

A complete example through which we show  how {\em  UR$_P$} proceeds  testing the consistency
of propositional NC formulas is Example \ref{ex:FormulaPropo} in the next section,  and
Example \ref{ex:combination-example} illustrates the effects of applying Proposition \ref{prop:input-propo-form}.

\subsubsection{Simplification Formulas Rules}

Each application of the previous  {\em  UR$_\Sigma$}  and {\em  UR$_P$}  demands
the subsequent application of trivial logical simplifications of propositional formulas.
For instance, $ (\vee \ \varphi \ (\vee \ P \ (\vee \ \neg R \ \phi) )) $
and $ (\vee \ P \ \{\wedge  \ (\vee) \ \varphi \} ) $ can be obviously substituted
by  $ (\vee \ \varphi \ P \ \neg R \ \phi) $ and $ P $, respectively.
   Next we formalize  such kind of  simplification   rules.
 Being $ \Pi $  the propositional formula of a possibilistic formula 
 $ \langle \Pi : \alpha \rangle $ in a   base $  \Sigma $, the first two rules below 
  simplify  formulas by (upwards) propagating 
   $(\vee)$   from sub-formulas to formulas:
  
$$\frac{ \langle \ \Pi \succ   \, ({\textcolor{black} {\vee}} \  \ \phi_1 \ldots   \phi_{i-1} 
\, (\vee) \, \phi_{i+1} \ldots \phi_k  \,)  : \alpha \rangle }
{ \langle \ \Pi \succ  ({\textcolor{black} {\vee}} \  \ \phi_1 \ldots 
 \phi_{i-1}  \, \phi_{i+1} \ldots \phi_k  \,) : \alpha \ \rangle}{{\bf \bot}\vee}$$

$$\frac{\langle \  \Pi \succ \{{\textcolor{black} {\wedge}} \  \ \varphi_1 \ldots   \varphi_{i-1} 
\, (\vee) \, \varphi_{i+1} \ldots \varphi_k  \,\} : \alpha \rangle }
{ \langle \  \Pi \succ \, (\vee)  : \alpha \ \rangle  }{{\bf \bot}\wedge}$$

\noindent The next two rules  remove redundant connectives. The first one removes a connective $ \odot $ if it is applied to a single formula,
i.e. $ [ \odot  \  \phi_1 ] $,
and the second one removes a connective  if it is inside
another equal connective, i.e. applies to sub-formulas with the pattern 
$  [ {\textcolor{black} {\odot_1}} \  \ \varphi_1 \ldots   \varphi_{i-1} 
\, [ \odot_2 \  \phi_1 \ldots \phi_n \,] \, \varphi_{i+1} \ldots \varphi_k  \,], \odot_1=\odot_2  $. So the formal  rules are:

$$\frac{ \langle \ \Pi \succ [ {\textcolor{black} {\odot_1}} \  \ \varphi_1 \ldots   \varphi_{i-1} 
\, [\, \odot_2 \  \phi_1 \,] \, \varphi_{i+1} \ldots \varphi_k  \,] : \alpha \rangle }
{\langle \  \Pi \succ [ {\textcolor{black} {\odot_1}} \  \ \varphi_1 \ldots   \varphi_{i-1} 
\, \phi_1  \, \varphi_{i+1} \ldots \varphi_k  \,] : \alpha \rangle }{\, \odot \phi }$$
 
$$\frac{\langle \  \Pi \succ [ {\textcolor{black} {\odot_1}} \  \ \varphi_1 \ldots   \varphi_{i-1} 
\, [ \odot_2 \  \phi_1 \ldots \phi_n \,] \, \varphi_{i+1} \ldots \varphi_k  \,] : \alpha \ \rangle, \odot_1=\odot_2  }
{\langle \ \Pi \succ [ {\textcolor{black} {\odot_1}} \  \ \varphi_1 \ldots   \varphi_{i-1} 
\,    \phi_1 \ldots \phi_n  \, \varphi_{i+1} \ldots \varphi_k  \,] : \alpha \ \rangle }
{\,\odot \odot}$$

\subsubsection{Possibilistic Rules}

We recall the possibilistic rules in Definition \ref{def:inderences}.
We pay attention to the case in which a conjunction 
$ \langle  \, \{\wedge \ \ \varphi_1 \ldots \varphi_i \ldots \varphi_k\} 
 : \alpha \rangle $ is deduced. It is clear that, in this case, we can deduce 
that  the necessity weight of   each individual conjunct $ \varphi $  is $ \alpha $:

 $$ \frac{\langle \, \{\wedge \ \ \varphi_1 \ldots \varphi_i \ldots \varphi_k\}
 : \alpha \rangle }
 {\{\langle \varphi_1 : \alpha \rangle, \ldots, 
\langle \varphi_i : \alpha\rangle,
 \ldots, \langle \varphi_k : \alpha \rangle \} } 
 {\mbox{\em \  \underline{MinD}}}$$

\vspace{.15cm}
The last needed rule to be included in $ \mathcal{UR}_\Sigma $ is {\em MaxN}:
 $$ \langle \varphi : \alpha \rangle, 
\langle \varphi : \beta \rangle \vdash  
\langle \varphi : \mbox{max}\{\alpha, \beta \} \rangle \quad 
\mbox{\em \underline{MaxN}}$$

 \subsubsection{The   Calculus $\mathcal{UR}_\Sigma$  }
 \label{subsec:Calculus}

The calculus    $\mathcal{UR}_\Sigma$  is  composed of all the above inference rules:

\begin{definition} We  define    $\mathcal{UR}_\Sigma$ as the  calculus 
 formed by     {\em UR$_{\Sigma}$}, {\em UR$_P$}, 
  the rules  {\em MinD}   and {\em MaxN}, and the    simplification rules,  namely  
   {\em  $$\mathcal{UR}_\Sigma=\{\mbox{UR}_{\Sigma},
   \mbox{UR}_P,\mbox{MinD}, \mbox{MaxN}, 
   {\bf \bot} \vee, {\bf \bot} \wedge,  \odot \,\phi, \odot \odot \, \}.$$ }
\end{definition}

   Examples \ref{ex:firstexample} and \ref{ex:secondexample}
   are simple examples of how $\mathcal{UR}_\Sigma$ proceeds.
   Example \ref{ex:BasePossib},  
   and its continuation Example \ref{ex:basepossib-cont}, are 
   quite complete examples.  Example \ref{ex:BasePossib} illustrates how $ \mathcal{UR}_\Sigma $  searches for just one empty clause $ \langle \bot : \alpha \rangle $ and Example \ref{ex:basepossib-cont} determines $ \mbox{Inc}(\Sigma) $.
 
 \vspace{.15cm}
  {\bf Remark.} Having established  
  possibilistic NC unit-resolution, the procedure 
NC unit-propagation    for possibilistic NC formulas can 
  be designed, and  on top of it,   the possibilistic NC   DPLL scheme
  can be defined.

\subsection{Finding $\mbox{Inc}(\Sigma)$} \label{subsect:find}

 The   calculus  
$ \mathcal{UR}_\Sigma $  determines {\bf just one} sub-set of contradictory formulas
along with  its inconsistency degree. Yet, a given $\Sigma  $
can typically contain  many contradictory subsets, each of them induces
the deduction of one empty clause $ \langle \bot : \alpha \rangle $.  By definition of
$ \mbox{Inc}(\Sigma)=\mbox{max}\{\alpha: \Sigma_{\geq \alpha} 
\mbox{is inconsistent} \}$ and   by Proposition \ref{pro:iff-iff-iff}, we have that:
\begin{equation} \label{eq:INCSigma}
 \mbox{Inc}(\Sigma)=\mbox{max}\{\alpha: \Sigma_{\geq \alpha} 
\mbox{is inconsistent} \}=\mbox{max}\{\alpha \,\vert \,
\Sigma \vdash \langle \bot : \alpha \rangle\}.
\end{equation}

Our simple strategy to find $ \mbox{Inc}(\Sigma)$ is as follows.
Firstly, using $\mathcal{UR}_\Sigma$, we determine one inconsistent
subset  $ \Sigma_1 \subseteq \Sigma $   
and its $ \mbox{Inc}(\Sigma_1)=\alpha $, which according to (\ref{eq:INCSigma}),
amounts to deducing   $\langle \bot: \alpha \rangle$.
 In the future, we are only interested in knowing whether $ \mbox{Inc}(\Sigma)>\alpha $ 
 and so, for that, we require only the strict $ \alpha $-cut of $ \Sigma $, i.e. $ \Sigma_{>\alpha}  $. Again using $\mathcal{UR}_\Sigma$, we attempt to deduce 
 $\langle \bot: \beta \rangle$ and, if it is obtained,   continue with $ \Sigma_{>\beta}  $.
These operations are recursively performed until getting a consistent base, i.e.  
  the empty clause is no longer  deduced.  Then the  $ \alpha $-cut 
of the last inconsistent base   
is  the   sought $ \mbox{Inc}(\Sigma)$. This process is algorithmically described below.
{\bf Find} should be  called with {\bf Inc} = 0.


\vspace{.4cm}
\noindent \ \  {\bf Find($ \Sigma$, Inc) }

\vspace{-.1cm}
{\em 
\begin{enumerate}

\item [(1)] 
Apply \, $ \mathcal{UR}_\Sigma $ to $\Sigma  $ and 
if  $\langle \bot : \alpha \rangle$ is derived 
  then go to (2) else Return  {\bf Inc}.

\item [(2)]  We search  whether there exists
$ \beta  > \alpha $ such that $ \Sigma_{> \beta} $ is inconsistent.
Thus, we update 
$ \Sigma \leftarrow \{\langle \varphi : \beta \rangle \ \vert \
\langle \varphi : \beta \rangle \in \Sigma, \ \beta > \alpha\} $   and
     $ {\bf Inc} \leftarrow \alpha  $;  
     and  call \,{\bf Find($ \Sigma $, Inc)}.

\end{enumerate}
}
\noindent The value  {\bf Inc} returned by {\bf Find} is   $\mbox{Inc}(\Sigma) $,
which is proven in Section \ref{sect:ProofsProperties}. If {\bf Inc} = 0 then the input $ \Sigma $ is consistent. 
   In Example \ref{ex:basepossib-cont},
we illustrate the  algorithmic strategy of {\bf Find}.

\subsection{Further Inferences Rules} \label{subsub:2inferencerules}
This last subsection  presents  two further inferences rules  no required
to ensure completeness but, since they allow shorter proofs, their appropriate management can yield 
significant speed-ups. These two rules are:
Propositional NC Local Unit-Resolution and Possibilistic NC Hyper Unit-Resolution.


\subsubsection{Propositional NC Local-Unit-Resolution} 

{\em $\mbox{UR}_{P} $} could also apply to propositional sub-formulas 
    and  can be  used in the general framework of  non-Horn-NC bases.
 The {\em $\mbox{UR}_{P} $} local application   
  means that  applying {\em $\mbox{UR}_{P} $} to
    sub-formulas $ \varphi $ 
   of any formula $ \Pi $, where  
    $\langle \Pi : \alpha \rangle \in \Sigma $, 
    such that $ \varphi $ has the 
 {\em $\mbox{UR}_{P} $}  numerator pattern,
     should be authorized. Namely,
 applying {\em $\mbox{UR}_P $}  to   sub-formulas   with pattern   
 $\varphi={\textcolor{red} {\ell}} \ {\textcolor{red} {\wedge}} \ 
\Pi \cdot ({\textcolor{blue} {\vee}} \  
 \mathcal{C}({\textcolor{blue} {\neg \ell}}) \ 
  \mathcal{D}({\textcolor{blue} {\neg \ell}})  \,)$   should be permitted
 and so, $ \varphi $ could be substituted  with 
${\textcolor{red} {\ell}} \ {\textcolor{red} {\wedge}} \ 
\Pi \cdot   \mathcal{D}({\textcolor{blue} {\neg \ell}}) $.  
Hence,     the formal specification of the  Propositional NC
 Local-Unit-Resolution 
rule, {\em $\mbox{LUR}$}, for any non-Horn-NC 
     $ \varphi $ is:

$$\tcboxmath{\frac{\langle \ \Pi 
\succ (\,{\textcolor{red} {\ell}} \ \ {\textcolor{red} {\wedge}} \ \ 
\varphi \succ ({\textcolor{blue} {\vee}} \ \ 
 \mathcal{C}({\textcolor{blue} {\neg \ell}}) \ \
  \mathcal{D}({\textcolor{blue} {\neg \ell}})  \,) \ \ ) : \alpha \ \rangle}
{\langle \ \Pi \succ (\,{\textcolor{red} {\ell}} \ {\textcolor{red} {\wedge}} \ 
\varphi \succ    \mathcal{D}({\textcolor{blue} {\neg \ell}})  \,) 
: \alpha \ \rangle }{{\mbox{\em \ \underline{LUR}}}}}$$

This inference rule should be read: if    $\langle \Pi : \alpha \rangle  \in \Sigma $ 
and $ \Pi $ has a conjunctive {\bf sub-formula}   with a literal 
$  \textcolor{red} \ell   $  conjunctively linked to a sub-formula
$ \varphi $ having   pattern $ ({\textcolor{blue} {\vee}} \ 
 \ \mathcal{C}({\textcolor{blue} {\neg \ell}}) \ \ 
 \mathcal{D}({\textcolor{blue} {\neg \ell}})  \,) $, then its component 
 $\mathcal{C}({\textcolor{blue} {\neg \ell}})$ can be eliminated.
 
 \vspace{.1cm}
 An example illustrating the functioning of the previous
 rule is Example \ref{ex:LUR}.


\vspace{.1cm}
{\bf Remark.} The introduction of this new  rule {\em $\mbox{LUR}$} applicable to certain
sub-formulas habilitates new sequences of inferences, and so, shorter proofs are now available. 

\begin{proposition} Let  $ \langle\Pi, \alpha \rangle \in \Sigma $. If applying
{\em ${\mbox{\em LUR}}$} to $ \Pi $ results in
$ \Pi' $, then $ \Pi $  and  $ \Pi' $ are logically equivalent.
\end{proposition}

\begin{niceproof} The   soundness of ${\mbox{\em LUR}}$
follows from that of {\em UR$_P$} proved in Lemma \ref{lem:propo-complet}.
\end{niceproof}

\noindent {\bf Remark.} The  obtaining  of the simplification rules 
for their local application to sub-formulas  is similarly obtained.


\subsubsection{Possibilistic NC Hyper-Unit-Resolution}
  
The given definition of possibilistic  NC  unit-resolution, or  
{\em UR$_\Sigma$}, can be  extended in order to 
obtain Possibilistic  NC Hyper-Unit-Resolution
 (${\mbox{\em HUR}}$). Then assume that the possibilistic base 
 has a unit-clause $ \langle \textcolor{red} \ell : \alpha \rangle  $ and two 
 sub-formulas $ ({\textcolor{blue} {\vee}} \ 
 \ \mathcal{C}({\textcolor{blue} {\neg \ell^1}}) \ \ 
 \mathcal{D}({\textcolor{blue} {\neg \ell^1}})  \,) $ and
 $ ({\textcolor{blue} {\vee}} \ 
 \ \mathcal{C}({\textcolor{blue} {\neg \ell^2}}) \ \ 
 \mathcal{D}({\textcolor{blue} {\neg \ell^2}})  \,) $,
 where $\neg \ell^i $ denotes a specific occurrence
 of $ \neg \ell $. The simultaneous application of NC unit-resolution with two sub-formulas
 is formally expressed as follows:
$$\frac{ \langle \textcolor{red} \ell : \alpha \rangle 
\ {\textcolor{red} {\wedge}} \ 
\langle\Pi^1 \succ \, ({\textcolor{blue} {\vee}} \ 
 \ \mathcal{C}({\textcolor{blue} {\neg \ell^1}}) \ \ 
 \mathcal{D}({\textcolor{blue} {\neg \ell^1}})  \,) : \beta^1 \rangle  \ 
 {\textcolor{red} {\wedge}} \   
 \langle\Pi^2 \succ \, ({\textcolor{blue} {\vee}} \ 
 \ \mathcal{C}({\textcolor{blue} {\neg \ell^2}}) \ \ 
 \mathcal{D}({\textcolor{blue} {\neg \ell^2}})  \,) : \beta^2 \rangle}
{ \langle \Pi^1 \succ  \mathcal{D}({\textcolor{blue} {\neg \ell^1}})  
: \mathrm{\bf min}\{\alpha,\beta^1\}  \, \rangle \, \wedge \,
\langle \Pi^2 \succ  \mathcal{D}({\textcolor{blue} {\neg \ell^2}})  
: \mathrm{\bf min}\{\alpha, \beta^2\}  \, \rangle }{\mbox{\em }}$$

\noindent If   the sub-formula
 $\langle \Pi^i \succ ({\textcolor{black} {\vee}} \ 
  \mathcal{C}({\textcolor{black} {\neg \ell^i}}) 
 \  \mathcal{D}({\textcolor{black} {\neg \ell^i}}) ) : \beta^i \rangle$ 
 is denoted $\langle\Pi, \mathcal{CD}({\textcolor{black} {\neg \ell}}), \beta \rangle^i$,
 then 
 ${\mbox{\em HUR}}$  for $ k $ sub-formulas is formally expressed  below,
 where, for  $ i, 1 \leq i \leq k $, $\beta'^i  =min\{\alpha,\beta^i\}$. 
 
$$\tcboxmath{\frac{ 
\langle {\textcolor{red} {\ell }} : \alpha \rangle \ {\textcolor{red} {\wedge}} \ 
\langle \Pi, \mathcal{CD}({\textcolor{blue} {\neg \ell}}), 
\beta \rangle^1
{\textcolor{red} {\wedge}} \ldots {\textcolor{red} {\wedge}} \ 
\langle \Pi, \mathcal{CD}({\textcolor{blue} {\neg \ell}}), 
\beta  \rangle^i \
 {\textcolor{red} {\wedge}} \,\ldots  {\textcolor{red} {\wedge}}
\ \langle \Pi , \mathcal{CD}({\textcolor{blue} {\neg \ell}}), \beta \rangle^k  }
{ \langle \Pi, \mathcal{D}({\textcolor{blue} {\ell}}), \beta' \rangle^1 \ 
{\textcolor{red} {\wedge}} \ \ldots \ {\textcolor{red} {\wedge}} \
\langle \Pi, \mathcal{D}({\textcolor{blue} {\ell}}), 
\beta' \rangle^i \ {\textcolor{red} {\wedge}} \
\ldots \ {\textcolor{red} {\wedge}} \
\langle \Pi, \mathcal{D}({\textcolor{blue} {\neg \ell}}), \beta' \rangle^k  } 
{\mbox{\em \ \underline{HUR}}}}$$

Since
 $ {\textcolor{black} {\neg \ell^i }}, 1 \leq i\leq k $, are 
   literal occurrences that are pairwise different, so are the   sub-formulas 
 $ \mathcal{CD}^i $ (and so $ \mathcal{D}^i $)
in the numerator and denominator of ${\mbox{\em HUR}}$. However, the formulas $ \Pi^i $
are not necessarily different; see {\bf Example \ref{ex:HUR}} for this concrete question and in general
for checking the working of the rule ${\mbox{\em HUR}}$. 

\vspace{.1cm}
{\bf Remark.} An NC hyper unit-resolution rule, more general than ${\mbox{\em HUR}}$,
 can be devised to include simultaneously several unit-clauses so that for each unit-clause 
$ \langle \ell : \alpha \rangle $ 
  several sub-formulas $\langle \Pi^i \succ ({\textcolor{black} {\vee}} \ 
  \mathcal{C}({\textcolor{black} {\neg \ell^i}}) 
 \  \mathcal{D}({\textcolor{black} {\neg \ell^i}}) ) : \beta^i \rangle$  can be considered. 
 In other words, one can consider simultaneously $ k \geq 2 $ 
  unit clauses and so simultaneously apply $ k \geq 2 $  {\em HUR} rules.

\section{Illustrative Examples} \label{sec:illustrativeexamples}

This section gives examples illustrating  the notions  presented in the previous sections.
Concretely, we provide   the next examples:

\vspace{.15cm}
 -- Example \ref{ex:firstexample}: a simple  inconsistent possibilistic  Horn-NC base. 
 
 \vspace{.05cm}
 -- Example  \ref{ex:secondexample}: a simple consistent possibilistic Horn-NC base.
  
  \vspace{.05cm}
 -- Example \ref{ex:FormulaPropo}:  a   complete propositional Horn-NC formula.
 
 \vspace{.05cm}
 -- Example  \ref{ex:combination-example}:  an input base with an inconsistent propositional formula.

 \vspace{.05cm}
 -- Example \ref{ex:BasePossib}:  a   complete possibilistic Horn-NC base.
 
 \vspace{.05cm}
 -- Example \ref{ex:basepossib-cont}: an example showing the strategy of   {\bf Find}.
 
 \vspace{.05cm}
 --   Example \ref{ex:LUR}:  an example of NC Local Unit-Resolution.
 
 \vspace{.05cm}
 -- Example \ref{ex:HUR}:  an example of NC Hyper Unit-Resolution. 

\vspace{.15cm}
\noindent Among such examples, we highlight Examples \ref{ex:BasePossib} 
and \ref{ex:basepossib-cont}, which  contain
 a rather complete Horn-NC base, whose inconsistency
degree is obtained  in two phases.  The first one is provided in Example \ref{ex:BasePossib}, 
and the second   in Example  \ref{ex:basepossib-cont}.
All inference rules of $\mathcal{UR}_{\Sigma}$ are needed
as well as   their combination   with the strict $ \alpha $-cuts of the input $ \Sigma $.

\begin{example}   \label{ex:firstexample}
 Let us assume that $ \Sigma  $ is the next possibilistic Horn-NC base:
$$\Sigma= \{ \, \langle P  : {\bf .8} \rangle,   \ \,
\langle \, \Pi_1=(\vee \ \ \{\wedge \ \ \neg P \ \ \neg Q\} \ \   Q ) : {\bf .6} \rangle, \ \,  
 \langle \, (\vee \ \neg P \ \neg Q ) : {\bf .7} \rangle \}$$

\begin{itemize}
\item {\em UR$_\Sigma$} with  $\langle P  : {\bf .8} \rangle$ and $ \Pi_1 $ gives rise to the next matchings:

-- $   \ \Pi=\Pi_1 $ 

  --  $   \ ({\textcolor{blue} {\vee}} \ 
 \ \mathcal{C}({\textcolor{blue} {\neg P}}) \  \mathcal{D}({\textcolor{blue} {\neg P}})  \,)
 =\Pi_1$  
 
-- $ \ \mathcal{C}({\textcolor{blue} {\neg P}})= \{\wedge \ \ \neg P \ \ \neg Q\}$
   
-- $  \ \mathcal{D}({\textcolor{blue} {\neg P}})=  Q $

 \item  Hence,  {\em UR$_{\Sigma}$}  adds:
 \quad $\Sigma \leftarrow \Sigma \ \cup \ \langle \, (\vee  \ \,   Q ) : {\bf .6} \rangle$ 

\item Applying simplifications to the last formula: \quad $\Sigma \leftarrow \Sigma \ \cup \ \langle       Q   : {\bf .6} \rangle$ 

\item {\em UR$_\Sigma$} with $\langle       Q   : {\bf .6} \rangle$ and with the last formula in the initial $ \Sigma $ gives:

-- $   \ \Pi=(\vee \ \neg P \ \neg Q ) $ 

  --  $   \ ({\textcolor{blue} {\vee}} \ 
 \ \mathcal{C}({\textcolor{blue} {\neg Q}}) \  \mathcal{D}({\textcolor{blue} {\neg Q}})  \,)
 =\Pi$  
 
-- $ \ \mathcal{C}({\textcolor{blue} {\neg Q}})= {\textcolor{blue} {\neg Q}}$
   
-- $  \ \mathcal{D}({\textcolor{blue} {\neg Q}})=  \neg P $

\item  Hence,  {\em UR$_{\Sigma}$}  adds: \ $ \Sigma \leftarrow \Sigma \ \cup \ \langle \, (\vee \ \neg P  ) : {\bf .6} \rangle$

\item  Resolving $\langle P  : {\bf .8} \rangle$ in  the input $ \Sigma $ 
with the last formula:   $\Sigma \leftarrow \Sigma \ \cup   \langle     \bot : {\bf .6} \rangle$
   
 \item  Therefore  {\em $\mathcal{UR}_{\Sigma}$}  obtains $ \mbox{Inc}(\Sigma)=  {\bf .6}$ \qed
\end{itemize}

\end{example}

\begin{example}   \label{ex:secondexample}
 Let us assume that $ \Sigma  $ is the next possibilistic Horn-NC base:
$$\Sigma= \{ \, \langle Q  : {\bf .8} \rangle,   \ \,
\langle \, \Pi_1=(\vee \ \ \neg Q \ \ \{\wedge \ \ R \ \ 
(\vee \ \ {\textcolor{blue} {\neg Q}} \ \ \{\wedge \ \ S \ \ \neg P \} 
\,)  \, \} \,) : {\bf .6} \rangle, \ \,  
\langle \, (\vee \ \neg P \ \neg Q ) : {\bf .7} \rangle  \}$$

\begin{itemize}
\item {\em UR$_{\Sigma}$} with $\langle Q  : {\bf .8} \rangle$ and with the 
rightest $ {\textcolor{blue} {\neg Q}} $ in $ \Pi_1 $ gives the next matchings:

-- $   \ \Pi=\Pi_1 $ 

  --  $   \ ({\textcolor{blue} {\vee}} \ 
 \ \mathcal{C}({\textcolor{blue} {\neg Q}}) \  \mathcal{D}({\textcolor{blue} {\neg Q}})  \,)
 =(\vee \ \ {\textcolor{blue} {\neg Q}} \ \ \{\wedge \ \ S \ \ \neg P \} 
\,)$  
 
-- $ \ \mathcal{C}({\textcolor{blue} {\neg Q}})=  {\textcolor{blue} {\neg Q}}$
   
-- $  \ \mathcal{D}({\textcolor{blue} {\neg Q}})=  \{\wedge \ \ S \ \ \neg P \}  $

 \item  Hence,  {\em UR$_{\Sigma}$}  adds:
 \quad $\Sigma \leftarrow \Sigma \ \cup \ \langle \, (\vee \ \ \neg Q \ \ \{\wedge \ \ R \ \ 
(\vee \ \   \{\wedge \ \ S \ \ \neg P \} 
\,)  \, \} \,) : {\bf .6} \rangle$ 

\item After simplifications: \  $\Sigma \leftarrow \Sigma \ \cup \ \langle \, 
(\vee \ \ {\textcolor{blue} {\neg Q}} \ \ \{\wedge \ \ R \ \     S \ \ \neg P    \, \} \,) : {\bf .6} \rangle$

\item Using again $\langle Q  : {\bf .8} \rangle$ and the last formula:
 
-- $   \ \Pi=(\vee \ \ {\textcolor{blue} {\neg Q}} \ \ \{\wedge \ \ R \ \     S \ \ \neg P    \, \} \,) $ 

  --  $   \ ({\textcolor{blue} {\vee}} \ 
 \ \mathcal{C}({\textcolor{blue} {\neg Q}}) \  \mathcal{D}({\textcolor{blue} {\neg Q}})  \,)
 =\Pi$  
 
-- $ \ \mathcal{C}({\textcolor{blue} {\neg Q}})=  {\textcolor{blue} {\neg Q}}$
   
-- $  \ \mathcal{D}({\textcolor{blue} {\neg Q}})=  \{\wedge \ \ R \ \     S \ \ \neg P    \, \} $

\item Hence,  {\em UR$_{\Sigma}$}  adds: \  $\Sigma \leftarrow \Sigma \ \cup \ \langle \, (\vee \ \   \{\wedge \ \ R \ \     S \ \ \neg P    \, \} \,) : {\bf .6} \rangle$

\item After simplifications: \  $\Sigma \leftarrow \Sigma \ \cup \ 
\langle \, \{\wedge \ \ R \ \ S \ \ \neg P \, \} \ : {\bf .6} \rangle$

\item Applying the rule {\em MinD}: \ 
$\Sigma \leftarrow \Sigma \ \cup 
\{\langle R : {\bf .6} \rangle,  \langle S : {\bf .6} \rangle, 
\langle \neg P : {\bf .6} \rangle \}$.

\item Using the first and last formulas in the initial $ \Sigma $: \ $\Sigma \leftarrow \Sigma \ \cup \ \langle \, (\vee \ \neg P ) : {\bf .7} \rangle$

\item Applying {\em MaxN} with $\langle \neg P : {\bf .6} \rangle$ and  $ \langle \, (\vee \ \neg P ) : {\bf .7} \rangle$ the former is eliminated.

\item No more resolvents apply.   

\item The propositional component  of $ \Sigma $ is consistent, so   $ \mbox{Inc}(\Sigma)=0 $. \qed
\end{itemize}
\end{example}

Next, we give a rather elaborated \underline{propositional} formula and show
how the propositional NC unit-resolution, or {\em  UR$_P$}, together 
with the simplification rules, detect
its inconsistency.

\begin{example} \label{ex:FormulaPropo} Let us assume that the input $ \Sigma $ 
has a possibilistic Horn-NC $ \langle \varphi: \alpha \rangle$ given below, where
$ \phi_1 $ and $ \phi_2 $ are assumed to be Horn-NC formulas.
$$ \langle \ \varphi=\{\wedge \ \ (\vee \ \ R \ \phi_1) \ \ (\vee \ \ \neg {P} \ \ \{\wedge \  \   
(\vee   \ \   \neg {P} \  \ \neg {R} \,) 
\quad (\vee \ \ \phi_2 \ \   \{\wedge \ \ \neg {Q}  \ \ {\textcolor{blue} {\neg {P}}} \,\} \, ) \ \ R \, \} \, ) \ {\textcolor{red} P} \,\} : \alpha \ \rangle$$ 

\noindent   The    tree associated with $ \varphi $ is depicted in {\bf Fig. 4}, left. 

\begin{center}
\begin{adjustbox}{valign=t}
\begin{tikzpicture}[]
\Tree
[.{$\wedge :  \alpha$}   [.$\vee$ [.R\\ ][.$\phi_1$\\ ] ] 
             [.$\vee$  [.{\color{black}$\neg {P}$}\\ ] 
                       [.$\wedge$  [.$\vee$ [.$\neg {P}$\\ ] [.$\neg {R}$\\ ]  ] 
                                   [.$\vee$ [.$\phi_2$\\ ] 
                                              [.$\wedge$ [.$\neg {Q}$\\ ] [.{\color{blue}$\neg {P}$}\\ ] ]
                                              ]
                                   [.R\\ ] 
                       ]
              ]
             [.{\color{red} P}\\ ]
        ]
\end{tikzpicture}
\end{adjustbox} 
\begin{adjustbox}{valign=t} \hspace{.5cm}
\begin{tikzpicture}[]
\Tree
[.{$\wedge :  \alpha$}   [.$\vee$ [.R\\ ][.$\phi_1$\\ ] ] 
             [.$\vee$  [.{\color{blue}$\neg {P}$}\\ ] 
                       [.$\wedge$  [.$\vee$ [.$\neg {P}$\\ ] [.$\neg {R}$\\ ]  ] 
                                   [.$\vee$ [.$\phi_2$\\ ] 
                                              ]
                                   [.R\\ ] 
                       ]
              ]
             [.{\color{red} P}\\ ]
        ]
\end{tikzpicture}
\end{adjustbox}

\vspace{-.3cm}
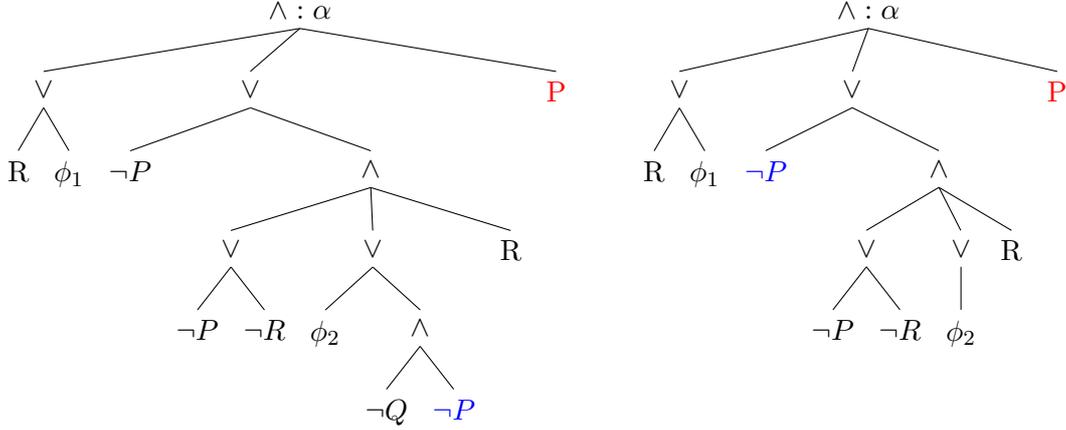
{\small {\bf Fig.  4.} {\em Formulas $ \varphi $ (left) and  $ \varphi' $ (right)}}
\end{center}

\noindent Thus, before computing
the inconsistency degree of $ \Sigma $, one needs to check whether its 
propositional formulas are inconsistent.  We show below how  {\em UR$_P$} 
checks the inconsistency of $ \varphi $.
\noindent {\em UR$_P$} with  ${\textcolor{red} P}$ and the right-most 
${\textcolor{blue} {\neg {P}}}$ yields the next matchings in the {\em UR$_P$} numerator:

\begin{itemize}

\item $\Pi 
=(\vee \ \ \neg {P} \ \ \ \{\wedge \  \   
(\vee   \ \   \neg {P} \  \ \neg {R} \,)  
\ \ \,(\vee \ \ \phi_2 \ \   \{\wedge \ \ \neg {Q}  \ \ {\color{blue}\neg {P}} \,\} \, ) \,  \ \ R \, \} \, )$

\item $ (\vee  \ \mathcal{C}({\color{blue}\neg {P}}) \ \mathcal{D}({\color{blue}\neg {P}}))
=(\vee \ \ \phi_2 \ \   \{\wedge \ \ \neg {Q}  \ \ {\color{blue}\neg {P}} \,\} \, ).$

\item $ \mathcal{C}({\color{blue}\neg {P}})=\{\wedge \ \ \neg {Q}  \ \ {\color{blue}\neg {P}} \,\} $

\item $ \mathcal{D}({\color{blue}\neg {P}})= \phi_2  $
\end{itemize}
\noindent Applying   {\em UR$_P$} to $ \varphi $ yields:    
$$\varphi'=\langle  \ \{\wedge \ \ (\vee \ \ R \ \phi_1) \ \ (\vee \ \ \neg {P} \ \ \{\wedge \  \   
(\vee   \ \   \neg {P} \  \ \neg {R} \,) 
\quad (\vee \ \ \phi_2  \, ) \ \ R \, \} \, )  \,\} : \alpha \ \rangle$$

\noindent  The resulting tree is   the right  one in  {\bf Fig. 4}. 
Assume that we proceed now   with a second NC unit-resolution step
by picking the same  ${\textcolor{red} P}$ and the left-most  ${\color{blue}\neg {P}}$ 
(colored blue  in {\bf Fig. 4}, right). Then, 
the right conjunct of the numerator of {\em UR$_P$} is as follows:  

\begin{itemize}

\item  $\Pi  =
(\vee \ {\color{blue}\neg {P}} \ (\vee \ \{\wedge \  \   
(\vee   \    \neg {P} \  \ \neg {R} \,) 
\ \  (\vee \  \phi_2  \, ) \ \ R \, \} \, )) $

\item $ (\vee  \ \mathcal{C}({\color{blue}\neg {P}}) \ \mathcal{D}({\color{blue}\neg {P}}))
= \Pi$

\item $\mathcal{C}({\color{blue}\neg {P}})={\color{blue}\neg {P}}$  

\item  $\mathcal{D}({\color{blue}\neg {P}})=(\vee \ \{\wedge \  \   
(\vee   \    \neg {P} \  \ \neg {R} \,) 
\ \  (\vee \  \phi_2  \, ) \ \ R \, \} \, )$
\end{itemize}
\noindent By applying  {\em UR$_P$} to $ \varphi' $,   the obtained formula 
 is depicted  in    {\bf Fig. 5}, left: 
\begin{center}
\begin{adjustbox}{valign=t}
\begin{tikzpicture}[]
\Tree
[.{$\wedge :  \alpha$}   [.$\vee$ [.R\\ ][.$\phi_1$\\ ] ] 
             [.$\vee$   
                       [.$\wedge$  [.$\vee$ [.$\neg {P}$\\ ] [.$\neg {R}$\\ ]  ] 
                                   [.$\vee$ [.$\phi_2$\\ ] 
                                              ]
                                   [.R\\ ] 
                       ]
              ]
             [.{\color{red} P}\\ ]
        ]
\end{tikzpicture}
\end{adjustbox}
\begin{adjustbox}{valign=t} \hspace{2.cm}
\begin{tikzpicture}[]
\Tree
[.{$\wedge :  \alpha$}   [.$\vee$ [.R\\ ][.$\phi_1$\\ ] ]   
                        [.$\vee$ [.{\textcolor{blue}{$\neg {P}$}}\\ ] 
                        [.{\textcolor{blue}{$\neg {R}$}}\\ ]  ] 
                                    [.$\phi_2$\\ ] 
                                   [.{\textcolor{red} R}\\ ] 
             [.{\textcolor{red} P}\\ ]
        ]
\end{tikzpicture}
\end{adjustbox}

\vspace*{-.4cm}
{\small {\bf Fig. 5.} {\em Example \ref{ex:FormulaPropo} continued.}} 
\end{center}

\noindent After three  simplification steps, one gets the formula 
associated with the  right tree in  {\bf Fig. 5.}
Finally,  
 two  applications of  {\em UR$_P$}  to the two  pairs {\textcolor{red} {$R$}} and {\textcolor{blue}{$\neg {R}$}}, 
and  {\textcolor{red} {$P$}} and {\textcolor{blue}{$\neg {P}$}}, lead 
  the calculus   to  derive $\langle (\vee) : \alpha\rangle$. \qed
\end{example}

In the next example, we illustrate the effects of finding  that  one of the propositional
formulas of the input base is inconsistent.


\begin{example} \label{ex:combination-example} Let  $ \varphi $ be the formula
from Example \ref{ex:FormulaPropo} and  $\Sigma_1 $ be $ \Sigma $ 
from Example \ref{ex:firstexample} 
and   let us analyze the base 
$ \Sigma = \Sigma_1   \cup   \{\langle \varphi : {\bf .6} \rangle \}$.
Then, firstly  the propositional rules of  {\em $\mathcal{UR}_\Sigma$} are applied to
each propositional Horn-NC in $ \Sigma $, and in particular, 
to $ \langle \varphi : {\bf .6} \rangle $, 
which, according to Example \ref{ex:FormulaPropo},  yields 
$  \langle (\vee) : {\bf .6} \rangle $. Then before calling {\bf Find},
$ \Sigma_1 $ is reduced to 
$ \Sigma_1=\{\langle P : {\bf .8} \rangle,  
\   \langle \, (\vee \ \neg P \ \neg Q ) : {\bf .7} \rangle\} $ and then 
{\bf Find} is called with such $ \Sigma_1 $. Since $ \Sigma_1$ is consistent,  
one can conclude that $ \mathrm{Inc}(\Sigma) = {\bf .6} $.

\end{example}

We next give a complete formula and illustrate how   $\mathcal{UR}_\Sigma$ determines just one inconsistent
 subset $ \Sigma_1 $ of  a  Horn-NC  base $ \Sigma $ and its degree 
 $ \mathrm{Inc}(\Sigma_1) $. By now, we are not
 concerned with finding the maximum inconsistency degree, but just in finding
 one inconsistent subset. Later, in Example \ref{ex:basepossib-cont}, we will illustrate the process performed by {\bf Find} to obtain
 the   inconsistency degree  $ \mathrm{Inc}(\Sigma) $. 

\begin{example}  \label{ex:BasePossib} 
 Let us assume that $ \Sigma  $ is the next possibilistic Horn-NC base:
$$\Sigma= \{ \, \langle P  : {\bf .8} \rangle,   \ \,
\langle \, \Pi_1 : {\bf .6} \rangle, \ \,  
\langle \, \Pi_2 : {\bf .5} \rangle, \ \, \langle \, \{\wedge \ \neg P \ \neg Q \} : {\bf .7} \rangle \}$$

\noindent wherein the propositional formulas $ \Pi_1 $ and $ \Pi_2 $, both individually consistent, are as follows:

\begin{itemize}
\item $  \Pi_1 = (\vee \ \ \{\wedge \ \ \neg P \ \ \neg Q\} \ \ \{\wedge \ \, Q \ \ P \}) $

\item $ \Pi_2 = (\vee \ \ \neg Q \ \ \{\wedge \ \ R \ \ 
(\vee \ \ \neg Q \ \ \{\wedge \ \ S \ \ \neg P \} 
\,)  \, \} \,) $
\end{itemize} 

\noindent The input base $ \Sigma $ is inconsistent and below, we step-by-step provide  the inferences carried 
out by the calculus {\em   $\mathcal{UR}_\Sigma$} to derive one
empty formula $ \langle \bot : \alpha \rangle $.

\begin{itemize}

\item  We apply {\em UR$_{\Sigma}$} with $ \langle P  : {\bf .8} \rangle $ 
and $ \langle \, \Pi_1 : {\bf .6} \rangle $ and the next matchings:

-- $\ \Pi=\Pi_1 $

 -- $ \ ({\textcolor{blue} {\vee}} \ 
 \ \mathcal{C}({\textcolor{blue} {\neg P}}) \ \ \mathcal{D}({\textcolor{blue} {\neg P}})  \,)=\Pi_1 $
 
--   $\ \mathcal{C}({\textcolor{blue} {\neg P}})=\{\wedge \ \neg P \ \neg Q\} $ 
  
-- $\ \mathcal{D}({\textcolor{blue} {\neg P}})=\{\wedge \ Q \  P\} $

 \item  Hence,  {\em UR$_{\Sigma}$}  adds:
 \quad $\Sigma \leftarrow \Sigma \ \cup \ \langle \, (\vee  \ \, \{\wedge \ Q \ P \}) : {\bf .6} \rangle$ 

\item Simplifying the last formula: \quad $\Sigma \leftarrow \Sigma \ \cup \ \langle    \, \{\wedge \ Q \ P \} : {\bf .6} \rangle$ 

\item Applying   {\em MinD}  to the last formula:
\  $\Sigma \leftarrow \Sigma \ \cup  \ \langle    \,    Q \  : {\bf .6} \rangle \cup \langle P  : {\bf .6} \rangle$  

\item Since   $\langle P  : {\bf .8} \rangle, \langle P  : {\bf .6} \rangle \in \Sigma$,
by {\em MaxN}:
   \  $ \Sigma \leftarrow \Sigma /\langle P  : {\bf .6}\rangle   $

\item Applying  {\em UR$_{\Sigma}$} with $ \langle    \,    Q \  : {\bf .6} \rangle $ 
  and the rightest $ \neg Q $   of $ \langle \, \Pi_2 : {\bf .5} \rangle $:

 -- $  \ \Pi=\Pi_2 $   

 -- $ \ ({\textcolor{blue} {\vee}} \ 
  \mathcal{C}({\textcolor{blue} {\neg Q}}) \  \mathcal{D}({\textcolor{blue} {\neg Q}})  \,)
 =(\vee \  {\textcolor{blue} {\neg Q}} \  \{\wedge \  S \  \neg P \} \,) $
    
 -- $ \ \mathcal{C}({\textcolor{blue} {\neg Q}})= {\textcolor{blue} {\neg Q}}$   
 
 -- $ \ \mathcal{D}({\textcolor{blue} {\neg Q}})=\{\wedge \ S \  \neg P\} $

 \item   Thus {\em UR$_{\Sigma}$}  adds: \quad $ \Sigma \leftarrow \Sigma \ \cup \
\langle \, (\vee \ \, \neg Q \ \, 
\{\wedge \ \, R \ \, (\vee \ \,  \{\wedge \ \, S \ \, \neg P \} 
\,)  \, \} \,) : {\bf .5} \rangle$

  -- We denote $ \langle \Pi_3 : {\bf .5} \rangle $ the last added formula.
   
\item Applying  {\em UR$_{\Sigma}$} with again $ \langle \, Q : {\bf .6} \rangle $  and the last formula   $ \langle \Pi_3 : {\bf .5} \rangle $:

-- $   \ \Pi=\Pi_3 $ 

  --  $   \ ({\textcolor{blue} {\vee}} \ 
 \ \mathcal{C}({\textcolor{blue} {\neg Q}}) \  \mathcal{D}({\textcolor{blue} {\neg Q}})  \,)
 =\Pi_3$  
 
-- $ \ \mathcal{C}({\textcolor{blue} {\neg Q}})= {\textcolor{blue} {\neg Q}}$
   
-- $  \ \mathcal{D}({\textcolor{blue} {\neg Q}})=\{\wedge \ \ R \ \ (\vee \ \  \{\wedge \ \ S \ \ \neg P \} 
\,)  \, \} $

\item   Hence  {\em UR$_{\Sigma}$} adds: \quad  $\Sigma \leftarrow \Sigma \ \cup \ \langle \, (\vee  \ \, 
\{\wedge \ \, R \ \ (\vee \ \,  \{\wedge \ \, S \ \, \neg P \} 
\,)  \, \} \,) : {\bf .5} \rangle$

\item Simplifying  the last formula: \quad
 $ \Sigma \leftarrow \Sigma \ \cup \ \langle \,  \{\wedge \ \ R  \ \ S \ \ \neg P \}    : {\bf .5} \rangle $

\item Using the rule {\em InvMinD:}  

 \ $ \Sigma \leftarrow \Sigma \ \cup \ \{ \langle \,  R    : {\bf .5} \rangle,  \langle \,  S     : {\bf .5} \rangle, \langle \,  \neg P : {\bf .5} \rangle \} $

\item From $\langle \,   \neg P     : {\bf .5} \rangle$ and the initial 
   $\langle P  : {\bf .8} \rangle$: \quad $\Sigma \leftarrow \Sigma \cup \langle \,  (\vee)    : {\bf .5} \rangle$.

\item So the (first) inconsistency degree found is {\bf .5}. \qed
\end{itemize}

\end{example}

\noindent Next example  continues   the previous one towards,  this time, computing   the proper
$ \mbox{Inc}(\Sigma) $. This example illustrates the strategy of {\bf Find} to do so. 

\begin{example} \label{ex:basepossib-cont}  Let us continue with Example \ref{ex:BasePossib}.
Since $\langle \, (\vee) : {\bf .5} \rangle$ 
was found, for checking whether $ \mbox{Inc}(\Sigma) >   {\bf .5}$,   all possibilistic formulas
whose necessity weight is not bigger than   ${\bf .5}$ are useless, that is, one can obtain the strict {\bf .5}-cut of $ \Sigma $.
Thus,   the new base is 
$ \Sigma_{>{\bf .5}}= \Sigma_1 \cup \Sigma_2 $,  where $ \Sigma_1 $ and
$ \Sigma_2 $  are  the strict {\bf .5}-cut of the \underline{initial formulas} 
 and of the
 \underline{deduced formulas},   respectively, and which are given below:

$$\Sigma_1= \{ \, \langle P  : {\bf .8} \rangle, \ \,  
\langle \, \Pi_1 : {\bf .6} \rangle, \ \,  
 \langle \, \{\wedge \ \neg P \ \neg Q \} : {\bf .7} \rangle \, \}$$ 
$$\Sigma_2=\{\langle \, (\vee  \ \, \{\wedge \ Q \ P \}) : {\bf .6} \rangle,
\ \ \langle    \, \{\wedge \ Q \ P \} : {\bf .6} \rangle, \ 
\ \langle \, Q \  : {\bf .6} \rangle,
  \}
$$

One can check that
 the only non-subsumed formula in $ \Sigma_2 $ is
 $\langle \, Q \  : {\bf .6} \rangle$.
 So, $\Sigma_2  $ is reduced to  $\Sigma_2 = \{ \langle \, Q \  : {\bf .6} \rangle\}$.
  Now,  {\bf Find} newly launches  the process to compute  the inconsistency   of
 $ \Sigma= \Sigma_1 \cup  \{ \langle \, Q \  : {\bf .6} \rangle\}$ with {\em ${\bf Inc}={\bf .5}$}  
   and follows  the next steps:
 
\begin{itemize}

\item   Using     $ \langle \, Q \  : {\bf .6} \rangle$  and  
 right-most formula in $ \Sigma_1 $ yields: \ \  $\langle \,  (\vee)    : {\bf .6} \rangle$.

\item The new $ \Sigma $ is  
$ \Sigma= \{  \langle P  : {\bf .8} \rangle,  \langle \, \{\wedge \ \neg P \ \neg Q \} : {\bf .7} \rangle \}$ 
and the new {\bf  Inc} is {\bf .6}.

 \item   $\mathcal{UR}_\Sigma $ is relaunched and finds
 $\langle \,  (\vee)    : {\bf .7} \rangle$.

\item The new $ \Sigma$ is $\{ \, \langle P  : {\bf .8} \rangle \, \}$ 
and the new   {\bf  Inc} is {\bf .7}.

\item  $\mathcal{UR}_\Sigma $   finds  $ \Sigma $ is consistent.

\item Hence {\bf Find} returns {\bf Inc = .7.}
 
\end{itemize}
\end{example}

Next example illustrates the application of  NC Local Unit-resolution, or {\em LUR.}

\begin{example}  \label{ex:LUR} Consider again the formula $ \varphi $ from Example \ref{ex:FormulaPropo}:
 $$\langle \ \varphi=\{\wedge \ \ (\vee \ \ R \ \phi_1) \ \ (\vee \ \ \neg {P} \ \ \{\wedge \  \   
(\vee   \ \   \neg {P} \  \ \neg {R} \,) 
\quad (\vee \ \ \phi_2 \ \   \{\wedge \ \ \neg {Q}  \ \ {\textcolor{black} {\neg {P}}} \,\} \, ) \ \ R \, \} \, ) \ {\textcolor{black} P} \,\} : \alpha \ \rangle.$$

\noindent One can check that   its sub-formula
$$ \phi=(\vee \ \ \neg P   \ \ \{\wedge \ \ (\vee \ \ \neg P  \ \ {\textcolor{blue} {\neg R}}) 
\ \ (\vee \ \ \phi_2 \ \ \{\wedge \ \ \neg {Q}  
\ \ {\textcolor{black} {\neg P  }} \}  ) \ \ {\textcolor{red} R} \} \,)$$ 
has the  pattern of the 
${\mbox{\em LUR}}$  numerator regarding  $ {\textcolor{blue} {\neg R}} $
and $ {\textcolor{red} R} $. Thus
 ${\mbox{\em LUR}}$ can be applied   and so $ \phi $ be replaced,
after simplifications,  with
$ (\vee \ \ \neg P   \ \ \{\wedge \ \   \neg P   \ \ (\vee \ \ \phi_2 \ \ \{\wedge \ \ \neg Q  \ \ {\textcolor{black} {\neg P  }} \}  ) \ \ R \} )$ in the   formula $ \varphi $. 
In this specific example, only one literal is removed, but in a general case,
big sub-formulas may be eliminated. \qed
\end{example}

The next example is devoted to the rule of NC hyper unit-resolution, or {\em HUR.}
 
\begin{example} \label{ex:HUR}   Let us reconsider also the formula in previous Example \ref{ex:FormulaPropo} (recall that the $i$ superscript in $ {\textcolor{blue} {\neg {P}^i}} $ denotes an specific literal occurrence
of $ {\textcolor{blue} {\neg {P}}} $):
$$ \langle \  \{\wedge \ \ (\vee \ \ R \ \phi_1) \ \ 
(\vee \ \ {\textcolor{blue} {\neg {P}^1}} \ \ \{\wedge \  \   
(\vee   \ \   {\textcolor{blue} {\neg {P}^2}} \  \ \neg {R} \,) 
\quad (\vee \ \ \phi_2 \ \   \{\wedge \ \ \neg {Q}  \ \ {\textcolor{blue} {\neg {P}^3}} \,\} \, ) \ \ R \, \} \, ) \ {\textcolor{red} P} \,\} : \alpha \ \rangle$$ 

\noindent One can apply NC Hyper Unit-Resolution with $ {\textcolor{red} P}  $ and the three literals ${\textcolor{blue} {\neg {P}^i}}$. The formula $ \Pi $ in the numerator
of {\em HUR} is   the same for the three literals, so it is noted $ \Pi^{1,2,3}$, but the
formulas $ (\vee \ \ \mathcal{C}(\neg {P}^i) \ \ \mathcal{D}(\neg {P}^i))$
 are different and are given below:

\vspace{.2cm}
-- $ \Pi^{1,2,3}= (\vee \ \ \neg {P} \ \ \{\wedge \  \   
(\vee   \ \   \neg {P} \  \ \neg {R} \,) 
\quad (\vee \ \ \phi_2 \ \   \{\wedge \ \ \neg {Q}  \ \ {\textcolor{black} {\neg {P}}} \,\} \, ) \ \ R \, \} \, )$

\vspace{.15cm}
-- $ (\vee \ \ \mathcal{C}(\neg {P}^1) \ \ \mathcal{D}(\neg {P}^1))=\Pi^{1,2,3}$

\vspace{.15cm}
-- $ (\vee \ \ \mathcal{C}(\neg {P}^2) \ \ \mathcal{D}(\neg {P}^2))= (\vee   \ \   {\textcolor{blue} {\neg {P}^2}} \  \ \neg {R} \,) $

\vspace{.15cm}
-- $ (\vee \ \ \mathcal{C}(\neg {P}^3) \ \ \mathcal{D}(\neg {P}^3))= (\vee \ \ \phi_2 \ \   \{\wedge \ \ \neg {Q}  \ \ {\textcolor{blue} {\neg {P}^3}} \,\} \, )$

\vspace{.25cm}
By applying NC Hyper Unit-Resolution, one gets:
$$ \langle \  \{\wedge \ \ (\vee \ \ R \ \phi_1) \ \ 
(\vee \ \   \{\wedge \  \   
(\vee   \ \     \neg {R} \,) 
\quad (\vee \ \ \phi_2     \, ) \ \ R \, \} \, ) \ {\textcolor{red} P} \,\} : \alpha \ \rangle$$

After simplifying:
$$ \langle \  \{\wedge \ \ (\vee \ \ R \ \phi_1) \ \ 
       \neg {R} \ \  \phi_2    \ \ R   \,\ {\textcolor{red} P} \,\} : \alpha \ \rangle$$

 \vspace{.15cm} 
Clearly,   a simple NC unit-resolution deduces $ \langle (\vee) : \alpha \rangle $.
Altogether, in this particular example, the rule {\em HUR} accelerates considerably
the proof of  inconsistency.
\end{example}

\section{Correctness of $\mathcal{UR}_{\Sigma}$ and Polynomiality of $ \mathcal{\overline{H}}_\Sigma$} \label{sect:ProofsProperties}

This section provides the proofs of both, the correctness of $\mathcal{UR}_{\Sigma}$
to determine the  inconsistency degree of $ \mathcal{\overline{H}}_\Sigma$
 and of the tractability
of   $ \mathcal{\overline{H}}_\Sigma$ (i.e. of Horn-NC-INC, Definition \ref{def:complexitiesofclassesofbases}).
Some proofs, as those of soundness of the simplification rules, are simple and intuitive
and so are omitted.
The following formal proofs are provided:

\vspace{.15cm}
-- soundness  of   quasi-clausal NC unit-resolution (Rule (\ref{eq:with-conjunct}));

\vspace{.05cm}
--  soundness of   nested  NC unit-resolution  {\em UR$_{\Sigma}$}; 

\vspace{.05cm}
-- correctness of the propositional rules  of    $\mathcal{UR}_{\Sigma}$; 
 
 \vspace{.05cm} 
  -- correctness  of the complete  {\em $\mathcal{UR}_{\Sigma}$};

\vspace{.05cm}  
  -- correctness of the algorithm {\bf Find}; and
 
 \vspace{.05cm} 
   --  polynomial complexity    of $ \mathcal{\overline{H}}_\Sigma$. 


\begin{proposition}  Rule  (\ref{eq:with-conjunct}) is sound:
$$\langle{\textcolor{black} {\ell}} \, : \, \alpha \rangle  \, {\textcolor{black}\wedge} 
 \, \langle \,({\textcolor{black} {\vee}} \ \, \mathcal{C}({\textcolor{black} {\neg {\ell}}}) \ \, \mathcal{D}({\textcolor{black} {\neg {\ell}}}) )\, : \beta \rangle \  \models \  \langle \, \mathcal{D}({\textcolor{black} {\neg {\ell}}})  : \mathrm{\bf min}\{\alpha,\beta\} \,\rangle .$$
\end{proposition}

\begin{niceproof}
Denoting $ \mathcal{F}=\langle{\textcolor{black} {\ell}} \, : \, \alpha 
\rangle  \, {\textcolor{black}\wedge} 
 \, \langle \,({\textcolor{black} {\vee}} \ \, \mathcal{C}({\textcolor{black} {\neg {\ell}}}) \ \, \mathcal{D}({\textcolor{black} {\neg {\ell}}}) )\, : \beta \rangle$, we have:

\vspace{.2cm} 
- By {\em MinD}, \ $\mathcal{F} \models \ \langle \, ({\vee} \ \ {\ell} \wedge \mathcal{C}(  {\neg {\ell}}) \ \ 
{\ell} \wedge \mathcal{D}(  {\neg {\ell}})) : \mathrm{\bf min}\{\alpha,\beta\} \,\rangle $.

\vspace{.2cm}
- Since  $ {\ell}\wedge\mathcal{C}(  {\neg {\ell}} \,) \equiv \bot $ then: 
$  \ \mathcal{F} \models \  \langle\,  {\ell} \wedge \mathcal{D}(  {\neg {\ell}} \,)  
 : \mathrm{\bf min}\{\alpha,\beta\} \,\rangle $.
 
\vspace{.2cm}
- By {\em MinD}, \   $  \mathcal{F} \models \  \langle\,  \mathcal{D}(  {\neg {\ell}} \,)  
 : \mathrm{\bf min}\{\alpha,\beta\} \,\rangle $.

\end{niceproof}

\begin{proposition}  \label{prop:soundnessgeneral} Te rule  UR$_\Sigma  $  is sound:
$$\langle{\textcolor{black} {\ell}} \, : \, \alpha \rangle  \, 
{\textcolor{black}\wedge} 
 \, \langle \,\Pi \succ ({\textcolor{black} {\vee}} \ \, \mathcal{C}({\textcolor{black} {\neg {\ell}}}) \ \, \mathcal{D}({\textcolor{black} {\neg {\ell}}}) )\, : \beta \rangle \  \models \  \langle \, \mathcal{D}({\textcolor{black} {\neg {\ell}}})  : \mathrm{\bf min}\{\alpha,\beta\} \,\rangle .$$
\end{proposition}

\begin{niceproof} By denoting 
$ \mathcal{F}=\langle{\textcolor{black} {\ell}} \, : \, \alpha 
\rangle  \, {\textcolor{black}\wedge} 
 \, \langle \,\Pi \succ ({\textcolor{black} {\vee}} \ \, \mathcal{C}({\textcolor{black} {\neg {\ell}}}) \ \, \mathcal{D}({\textcolor{black} {\neg {\ell}}}) )\, : \beta \rangle$, we have:
 
 \begin{enumerate}[$-$]
 \item By {\em MinD}, \   
  $ \mathcal{F} \models  
 \langle \ell \wedge \Pi \succ  ({\textcolor{black} {\vee}} \ \, \mathcal{C}({\textcolor{black} {\neg {\ell}}}) \ \, \mathcal{D}({\textcolor{black} {\neg {\ell}}}) )\, 
 :\mathrm{\bf min}\{\alpha,\beta\} \rangle$. 
 
\item  Then, \ $   \mathcal{F} \models  
 \langle \ell \wedge \Pi \succ  ({\textcolor{black} {\vee}} \ \, 
 \ell \wedge \mathcal{C}({\textcolor{black} {\neg {\ell}}}) \ \, 
 \ell \wedge \mathcal{D}({\textcolor{black} {\neg {\ell}}}) )\, 
 :\mathrm{\bf min}\{\alpha,\beta\} \rangle $.

\item Since  $ {\ell}\wedge\mathcal{C}(  {\neg {\ell}} \,) \models \bot $ then, 
\ $  \ \mathcal{F} \models \langle  \ell \wedge \Pi \succ   {\ell}\wedge\mathcal{D}(  {\neg {\ell}} \,)  
 : \mathrm{\bf min}\{\alpha,\beta\} \,\rangle $.
 
\item Then, \ $ \mathcal{F} \models \langle  \ell \wedge \Pi \succ    \mathcal{D}(  {\neg {\ell}} \,)  
 : \mathrm{\bf min}\{\alpha,\beta\} \,\rangle $.
 
\item By {\em MinD}, \  $ \mathcal{F} \models  \langle  \Pi \succ    \mathcal{D}(  {\neg {\ell}} \,)  
 : \mathrm{\bf min}\{\alpha,\beta\} \,\rangle $.
 \end{enumerate}
 \vspace{-.7cm}
\end{niceproof}

The next lemma states the correctness of the propositional rules of $ \mathcal{UR}_\Sigma $.

\begin{lemma} \label{lem:propo-complet} Let
  $ \mathcal{UR}_P $ be the subset 
 $\{\mbox{UR}_P, {\bf \bot} \vee, {\bf \bot} \wedge, \odot \phi, \odot \odot \} 
 \subset \mathcal{UR}_\Sigma$. 
   A propositional Horn-NC formula   $ \varphi $ is inconsistent iff  $ \mathcal{UR}_P $ 
    with input $ \langle\varphi : \alpha\rangle$ derives some
    $ \langle \bot : \alpha\rangle$.
\end{lemma}

\begin{niceproof} We analyze below both directions of the lemma.

\vspace{.1cm}
$\bullet \ \Rightarrow $ Let us assume that $\varphi$ is inconsistent. 
Then $\varphi$ must have
a sub-formula verifying the {\em $\mbox{UR}_P  $}  numerator; otherwise, all 
complementary pairs of literals 
$\ell$ and  $\neg \ell$   are 
included in disjunctions. In this case, since all disjunctions of $ \varphi $, 
by definition of Horn-NC formula, have
at least one negative literal, $ \varphi $ would be satisfied by assigning to all propositions
 the value $0$,
which contradicts the initial hypothesis. 
Therefore, {\em $\mbox{UR}_P  $} is applied to $ \varphi $ with 
two complementary  literals $\ell$ and  $\neg \ell$ and the resulting formula 
is simplified. The new  formula is equivalent to $ \varphi $ 
and has at least one literal less than $ \varphi $. Hence, by induction on the number of literals
of $ \varphi $, we  obtain that  $ \mathcal{UR}_P$   ends only
 when $ \langle (\vee ) : \alpha \rangle$ is derived.

 $ \bullet \Leftarrow$   Let us assume that   
 {\em $ \mathcal{UR}_P $}  has been iteratively applied until
 a formula $ \varphi' $ different from $ \langle (\vee ) : \alpha \rangle$ is obtained. Clearly, 
 if the {\em $\mbox{UR}_P  $} numerator is not applicable then 
there is not a conjunction of a literal $\ell$ with  a disjunction 
 including   $\neg \ell$.
 Then we have, firstly, since {\em $ \mathcal{UR}_P $} is sound, that $ \varphi $ and   $ \varphi' $ are equivalent.
Secondly, if $\varphi'$ has   complementary literals, then they are integrated in
 disjunctions. Thus   $\varphi'$ is satisfied by assigning the value 0 to all its unassigned propositions, since,  by definition of Horn-NC formula, 
 all  disjunctions have at least one negative  disjunct. 
 Therefore, since $\varphi'$ is consistent so is  $\varphi$.
\end{niceproof}

The next lemma claims the correctness of $ \mathcal{UR}_\Sigma $.

\begin{lemma} \label{lem:correctnessURSigma}  Let    $ \Sigma $ be a 
possibilistic Horn-NC base.    
  $ \mathcal{UR}_\Sigma $ derives some empty clause
$ \langle (\vee) : \alpha \rangle $ iff $ \Sigma $ is inconsistent, and 
if  $ \mathcal{UR}_\Sigma $ derives 
$ \langle (\vee) : \alpha \rangle $ then $ \mbox{Inc}(\Sigma) \geq \alpha $.
\end{lemma}

\begin{niceproof} The propositional component of possibilistic NC unit-resolution verifies
  Lemma \ref{lem:propo-complet}, and hence,
\ $ \mathcal{UR}_\Sigma $ derives an empty formula 
$ \langle (\vee) : \alpha \rangle$ iff the conjunction of the propositional
formulas in the base $\Sigma  $ is inconsistent, 
namely if  $\Sigma^*  $ is inconsistent.
%
%
If $ \mathcal{UR}_\Sigma $  derives 
$ \langle (\vee) : \alpha \rangle$, then by Lemma \ref{lem:propo-complet},
$ \mathcal{UR}_\Sigma $ detects a subset $ \Sigma_1 \subseteq \Sigma $ which is indeed inconsistent.
Then by Proposition \ref{prop:soundnessgeneral},  
    the degree $ \alpha$  found by $ \mathcal{UR}_\Sigma $ corresponds to  
    $\mbox{Inc}(\Sigma_1) $. Since obviously $\mbox{Inc}(\Sigma_1) \leq \mbox{Inc}(\Sigma)$, the lemma holds. 
\end{niceproof}

The next lemma states the correctness of the algorithm {\bf Find}.

\begin{lemma} \label{lem:correctFind} If \,{\bf Find($ \Sigma, 0$)} returns $ \alpha $
then $ \mathrm{Inc}(\Sigma)=\alpha $.
\end{lemma}

\begin{niceproof} We denote   $ \Sigma' $  and \mbox{\bf Inc}  
the  variables of {\bf Find} in a given recursion.
 Let us prove that the next hypothesis holds in every call to \mbox{\bf Find}: 
$$   \Sigma_{> \mbox{\footnotesize \bf Inc}}= \Sigma' 
 \mbox{\ \, and \ }  
 \mbox{Inc}(\Sigma)= \mbox{max}\{\mbox{Inc}(\Sigma'), \mbox{\bf Inc}\}     $$ 

-- We check that the initial call   {\bf Find($ \Sigma, 0$)}   verifies the hypothesis:

\newpage
\hspace{.5cm } - We have $ \Sigma'=\Sigma $ and \mbox{\bf Inc} = 0.  
 
 \vspace{.15cm}
\hspace{.5cm } - So indeed we have  $\Sigma_{> 0}=\Sigma'  $   and  $\mbox{Inc}(\Sigma)= \mbox{max}\{\mbox{Inc}(\Sigma'),0\}= \mbox{Inc}(\Sigma)   $

\vspace{.15cm}
-- We prove that if the hypothesis holds for $ k \geq 1 $ then  it holds
    for $ k + 1 $.

\vspace{.15cm}
-- First of all, $ \mathcal{UR}_\Sigma $ is applied  to $ \Sigma' $. 

\vspace{.15cm}
-- By Lemma \ref{lem:correctnessURSigma},
if $ \mathcal{UR}_\Sigma $   derives  $ \langle \bot : \alpha \rangle$ 
then $ \mbox{Inc}(\Sigma') \geq \alpha$, else $ \Sigma' $ is consistent.

\vspace{.15cm}
-- {\em Case $ \Sigma' $ is consistent:}  $\mbox{Inc}(\Sigma')=0 $. 

 \vspace{.15cm}
\hspace{.5cm} - By induction hypothesis $ \mbox{Inc}(\Sigma)= \mbox{max}\{\mbox{Inc}(\Sigma'), \mbox{\bf Inc}\}={\bf Inc}$

\vspace{.15cm}
\hspace{.5cm} - So {\bf Find}  correctly returns $ \mbox{Inc}(\Sigma) $
and   ends.

\vspace{.15cm}
-- {\em Case $ \Sigma' $ is inconsistent:}  $\mbox{Inc}(\Sigma')=\alpha > 0 $.

\vspace{.15cm}
\hspace{.5cm} - By induction hypothesis 
  $\Sigma' =\Sigma_{> \mbox{\footnotesize \bf Inc}}$,  and so $ \alpha > \mbox{\bf Inc} $.

  \vspace{.15cm}
\hspace{.5cm} - The next  $ \Sigma' $, noted $ \Sigma'' $,
  and   {\bf Inc}, noted {\bf Inc'}, are $ \Sigma''=\Sigma'_{>\alpha} $ and {\bf Inc'}=$ \alpha $.
   
   \vspace{.07cm}
   \hspace{.8cm} 
   We check in 
    $(i)$ and $(ii)$  that      the hypothesis holds:

\vspace{.15cm}
\hspace{.5cm}  $(i)$  By induction hypothesis:  $   \Sigma' =\Sigma_{> \mbox{\footnotesize \bf Inc}} $

\vspace{.15cm}
\hspace{1.5cm}  --  Since $ \alpha > \mbox{\bf Inc} $ then trivially $\Sigma'_{> \alpha}=\Sigma_{> \alpha}$

\vspace{.15cm}
\hspace{1.5cm}  --  Hence  
$\Sigma''=\Sigma'_{> \alpha} =\Sigma_{>\alpha}=\Sigma_{> \mbox{\footnotesize \bf Inc'}}  $.

\vspace{.15cm}
\hspace{.5cm}   $(ii)$ By Lemma  \ref{lem:correctnessURSigma}, 
$ \mbox{Inc}(\Sigma') \geq \alpha $
 
 \vspace{.15cm}
\hspace{1.5cm} --  Since $ \Sigma'=\Sigma_{> \mbox{\footnotesize \bf Inc}}$
and $ \alpha >  $ {\bf Inc} 
then $ \mbox{Inc}(\Sigma) \geq \alpha $

\vspace{.15cm}
\hspace{1.5cm} -- Hence  
$ \mbox{Inc}(\Sigma)=\mbox{max}\{\mbox{Inc}(\Sigma_{>\alpha}), \alpha \}=
\mbox{max}\{\mbox{Inc}(\Sigma''), \mbox{\bf Inc'} \}$

\vspace{.15cm} 
Altogether, the hypothesis holds until {\bf Find} finds $ \Sigma' $ consistent
and then correctly returns $ \mbox{Inc}(\Sigma)$. Hence Lemma \ref{lem:correctFind} holds.
\end{niceproof}

The next three propositions prove that finding $\mbox{Inc}(\Sigma)  $ of 
any Horn-NC  $ \Sigma $ is polynomial.

\begin{proposition} \label{propo:onesequence} If  $ \Sigma $ is a possibilistic Horn-NC base,
then $ \mathcal{UR}_\Sigma $ with  input $ \Sigma $ performs  at most $  n^2$ inferences,
 $  n$ being the number of symbols (size) of $ \Sigma $.
\end{proposition}

\begin{niceproof} On the one hand, each   rule of $ \mathcal{UR}_\Sigma $ adds a  
formula $\langle \varphi : \alpha \rangle$, where $ \varphi  $ is a sub-formula
of a  propositional formula $ \Pi $ of a possibilistic formula 
$\langle \Pi : \beta \rangle$ in the current base. Hence,
the current base always contains only sub-formulas from $\Sigma  $.
On the other hand,  the weight $ \alpha $ of added formulas $\langle \varphi : \alpha \rangle$ is the minimum  of two weights
  in the current base. Hence, the weights of 
formulas in the current base  always come  from $\Sigma  $. Thus, 
the maximum number of deduced formulas  is $ m \times k $,
where $ m $   is the number of   sub-formulas in $ \Sigma $
and  $ k $ is the number
of different weights in $ \Sigma $. Hence, the maximum number  of inferences 
performed by $ \mathcal{UR}_\Sigma $  is $ m \times k $. If $ n  $ is the size 
of $ \Sigma $, then $ m,k \leq n $ and so  the proposition holds.  
\end{niceproof}

\begin{proposition} \label{propo:inferencesconsistency} 
If $ \Sigma $ is  a possibilistic Horn-NC base, then  
{\bf Find}{\em ($\Sigma$,\,0)} performs at most $ n^2 $
recursive calls,    $  n$ being the number of symbols of $ \Sigma $.  
\end{proposition}

\begin{niceproof} {\bf Find} stops when it detects that the current base  is consistent.
If it is inconsistent then $ \langle   (\vee) : \alpha \rangle $ is deduced, and   
for future calls, {\bf Find} cancels 
the set  
$  \{\langle \phi : \beta \rangle \ \vert \ 
\langle \phi : \beta \rangle \in \Sigma, \beta \leq \alpha\}$.  This set
trivially    contains at least
one formula as $ \langle   (\vee) : \alpha \rangle $ has been derived. 
On the other hand and as discussed in the previous proof,   the maximum number
of inserted formulas in the   base is 
at most $ m \times k $,
where $ m $   is the number of   sub-formulas in $ \Sigma $
and  $ k $ is the number
of different weights in $ \Sigma $.
Thus, the number of performed recursive calls
is at most   $ m \times k $. Since $ m, k \leq n$, then
   Proposition \ref{propo:inferencesconsistency} holds.
\end{niceproof}

The next proposition states that the overall complexity to determine  
$\mbox{Inc}( \Sigma) $, $ \Sigma \in \mathcal{\overline{H}}_\Sigma $, is polynomial.

\begin{proposition} \label{propo:complexity} If $ \Sigma $ is a possibilistic Horn-NC base,
then computing $\mbox{Inc}( \Sigma) $ 
  takes polynomial time, i.e. the class $ \mathcal{\overline{H}}_\Sigma $ is polynomial.
\end{proposition}

\begin{niceproof} On the one hand, by Propositions
\ref{propo:onesequence} and \ref{propo:inferencesconsistency}, 
 the number of performed:
  the hypothesis holds.
inferences in each recursive call and the number of recursive calls are both polynomial. Hence the total number of inferences is polynomially bounded. 
On the other hand,
it is not hard to find a data structure so that each inference 
 in $ \mathcal{UR}_\Sigma $ can be polynomially performed w.r.t the size of 
 $ \Sigma $. Therefore, the  
complexity to   determine $ \mbox{Inc}(\Sigma) $ is polynomial. 
\end{niceproof}


\noindent {\bf Remark.} Determining a tight polynomial degree of the worst-case 
complexity of  computing a Horn-NC base, 
is planed for future work (see Section \ref{sec:relatedwork }). It should be mentioned
 that the polynomial degree of computing their counterpart Horn clausal bases
 has not been specified either.


\section{Related and Future Work} \label{sec:relatedwork }

In this section we briefly discuss related work in a number of  
possibilistic logical contexts and succinctly propose objectives towards which 
our future work can be oriented. 

\vspace{.3cm}
$ \Box$ \underline{Discovering    polynomial NC classes.} 

In   propositional logic, the valuable contribution to clausal efficiency of the conjunction
of Horn formulas and Horn-SAT algorithms  is reflected by the fact
that the highly efficient DPLL solvers embed a Horn-SAT-like algorithm, so-called Unit
Propagation. 
Hence, searching for polynomial (clausal) super-classes of the Horn class in propositional logic
has been a key issue for several decades towards improving clausal reasoning and has led 
 to a great number of such classes being currently  known: hidden-Horn, generalized Horn, Q-Horn, extended-Horn,    etc.  (see \cite{FrancoMartin09,Imaz2021horn} for short reviews).
So it is arguable that, just as the tractable clausal fragment 
has helped to grow  overall clausal efficiency, likewise  widening the tractable non-clausal 
fragment  would grow  overall 
 non-clausal efficiency. Indeed, inspired by such polynomial non-clausal classes, 
 efficient non-clausal algorithms can be devised. Nevertheless, we stress that the current
tractable non-clausal fragment is almost empty, which signifies
a manifest disadvantage with respect to its clausal counterpart.
We will extrapolate the previous argumentation to possibilistic logic and determine
further NC subclasses whose    inconsistency-degree computing
   have polynomial complexity.
In our next work, 
we will search   for polynomial classes obtained by lifting the possibilistic
renameable  (or hidden) Horn formulas to the NC level. 

\vspace{.3cm}
$ \Box$ \underline{Designing low-degree   polynomial algorithms.}

As we have seen in Section \ref{sect:ProofsProperties}, the number of inferences
required by $ \mathcal{UR}_\Sigma $ and the number of recursive calls to
  {\bf Find} are both bounded by $ O(n^2) $, where $ n $ is the symbol number
of the input base. Also,     all inferences in
$ \mathcal{UR}_\Sigma $  can be reasonably performed in $ O(n) $. Altogether,
the complexity of determining $ \mbox{Inc}(\Sigma) $ of Horn-NC bases is in
  $ O(n^5) $. This complexity, though polynomial, is of course no satisfactory
for real-world applications where the size of formulas can be relatively
 huge. However   no much care has been taken 
 in the proofs of Section \ref{sect:ProofsProperties} because
   the   goal was proving tractability, and so,
   first of all,   a fine-grained analysis
of complexity is pending.  This analysis 
 will permit likely to precise a tighter polynomial-degree complexity.
Nevertheless, the work presented here should be resumed 
in order to notably decrease the polynomial-degree. Also, advances
in several directions are obliged  such as finding suitable 
data-structure and proposing   optimized algorithms. Finally, we mention that 
 this optimization effort
 has not been done so far in the clausal framework either.  
 In fact, no explicit upper bounds of complexity algorithms that compute
 possibilistic clausal bases are available.

\vspace{.3cm}
  $ \Box$ \underline{Combining necessity and possibility measures.}
  
 Our approach   considers only necessity-valued formulas
 but it can be extended to bases where formulas can have
 associated a possibility or a necessity measure.  Such logical
 context has already been studied by previous authors
 \cite{Hollunder95, Lang00,CouchariereLB08}. If a   formula
 is possibility-valued $ \langle \varphi : \Pi(\varphi) \geq \alpha\rangle$ it
 can be converted into the follow:
  the hypothesis holds.ing necessity-valued formula
 $ \langle \neg \varphi : N(\neg \varphi) \leq  1- \alpha\rangle $. 
 Hence, bases having both possibility
 and necessity formulas  can be converted into equivalent  necessity
 formulas, where  the threshold $ \alpha $ is additionally accompanied
  by the indication of  whether the formula must have a necessity level
   either greater or smaller than $ \alpha $. On the other side,
   resolution for possibility-necessity formulas 
   is well-known \cite{DuboisP94, DuboisP04a, DuboisP14},
   but   its extension
   to the NC level towards defining a possibility/necessity NC unit-resolution
   calculus seems not to be trivial and is an open question.

 \vspace{.3cm}
$ \Box$ \underline{Computing general NC bases:} 

Since
$ \mathcal{\overline{H}}_\Sigma $ is an NC sub-class,  computing arbitrary
NC bases is a natural continuation of the presented approach. As said in the  Introduction, 
real-world problems are generally expressed in NC form, and so
  clausal reasoning requires the usage of a previous NC-to-clausal transformation.
However,  such transformation is highly inadvisable  as it increases the
formula size and number of variables, and losses the logical equivalence
and the  formula's original structure. Besides the clausal form is not unique
and  how to guide the nondeterministic process
 towards a  ''good'' clausal formula is not known. So, real-world efficiency 
 is reached if      formulas  are computed in its original non-clausal form.
 Here we have favored   non-clausal reasoning  with  
 contributions which particularly facilitate deduction based on  
  resolution and DPLL. Brief discussions for future work 
 related to how compute possibilistic  general NC formulas 
 by means of resolution and DPLL can be found  below in the  items
   \mbox{``$ \Box$ Defining NC Resolution''} and  ``$ \Box$ Defining NC DPLL".

 \vspace{.3cm}
$ \Box$ \underline{Developing    NC logic programming.} 

Possibilistic logic programming has received notable attention, and in fact, after
the pioneer work in \cite{DuboisLP91}, 
  a number of   approaches e.g.  \cite{BenferhatDP93,AlsinetG00,AlsinetGS02,AlsinetCGS08} are available in the state-of-the-art.
However, all of them focus exclusively on the clausal form and so  are Horn-like
  programs where
the body of rules is a conjunction of propositions and the head is a single proposition.
Our future work in this research direction will be oriented 
  to show    that 
$\mathcal{\overline{H}}_\Sigma $ and $\mathcal{UR}_\Sigma$ 
 allow to handle non-clausal logic programs, and more concretely:  
 
\noindent $(1)$  To conceive  a  language for {\em Possibilistic 
 Non-Clausal Logic Programing}  denoted $ \mathcal{LP}_\Sigma $.     
  So, instead of Horn-like rules, in $ \mathcal{LP}_\Sigma $, 
  one may handle   Horn-NC-like rules   wherein  bodies 
  and heads of rules are   NC formulas with slight
  syntactical restrictions (issued  from $\mathcal{\overline{H}}_\Sigma $).  In fact, one
  can check that a rule with syntax 
  $\langle \Pi \rightarrow \mathrm{Hnc} : \alpha \rangle$, where $ \Pi $ is an NC formula
  with only positive literals and Hnc is a Horn-NC formula, 
  is indeed a possibilistic Horn-NC formula.
  Therefore a set of such kind of rules, i.e. a program, 
  is a  Horn-NC base.

\noindent   $(2)$ To   answer queries with an efficiency qualitatively comparable to 
   the clausal one, i.e. polynomial,
 which is possible thanks to the next two   features: (i)
 an $ \mathcal{LP}_\Sigma $ program  belongs to $\mathcal{\overline{H}}_\Sigma $
  and $\mathcal{\overline{H}}_\Sigma $ is a polynomial class;  and (ii)
the bases in   $\mathcal{\overline{H}}_\Sigma $ have only one minimal model, since, as aforementioned, they are  equivalent to a Horn clausal formula.

 \vspace{.3cm}
$ \Box$ \underline{Developing    NC answer set programming.}

After the works in possibilistic logic programming, a succession 
of works on possibilistic answer set programming has been carried out, 
started by  \cite{NicolasGSL06} and   continued with e.g.  
\cite{NievesOC07,NievesOC13,ConfalonieriNOV12,ConfalonieriP14,BautersSCV12,BautersSCV14}.
Although, like in possibilistic logic programming, they also
focus on the clausal form, there exists an exception as indicated in the Introduction.
Indeed,   NC possibilistic logic has been formerly  dealt with   
by the auhors in \cite{NievesL12,NievesL15}. However, in this work   no effectiveness issues
 are addressed. Instead, the authors extend, to possibilistic logic,
 proper concepts of non-clausal answer set programming 
 within classical logic as  originally defined
in   \cite{LifschitzTT99}; so  their aim  
 is   distinct from ours. Our
future work in this research direction will be oriented 
  to show  the scope of the expressiveness of the class 
  $ \mathcal{\overline {H}}_\Sigma $ for possibilistic NC answer set programming 
  and to analyze  the efficiency allowed by $ \mathcal{\overline {H}}_\Sigma $.
Although answering queries in possibilistic answer set programming is an intractable problem
\cite{NicolasGSL06},
we do not rule out the possibility of finding  tractable sub-classes.

 \vspace{.3cm}
 $ \Box$ \underline{Partially  ordered   possibilistic logic.} 
 
 In this work we have assumed that possibility and necessity  measures, 
which rank interpretations and formulas,  were in 
the real unit-interval $[0, 1]$, i.e. 
 the available information is supposed to be totally ordered. 
 Thus for two different necessity values $ \alpha, \beta $, we have
 always that either $ \alpha > \beta $ or $ \alpha < \beta $. However, 
 considering   a more real-world context with only partial orders
 has already been studied \cite{BenferhatLP03,BenferhatLP04,CayrolDT14}. For instance,
 partial orders avoid comparing unrelated pieces of information,
 which happens when we merge multiple sources information and the merged 
 pieces do not have    a shared reference for uncertainty. So, considering a partial
 pre-order on interpretations by means of a partial pre-order on formulas
 is a more real-world  scenario to which the presented method is planned to be  
 generalized.

  \vspace{.3cm}
 $ \Box$ \underline{Casting  richer logics  into  possibilistic logic}. 
 
 In our 
 logical framework, propositional logic is casted into possibilistic logic.
 Casting richer logics into possibilistic logic has already been carried out 
 by several authors, for instance, Alsinet et al. \cite{AlsinetG00} cast G\"{o}del logic many-valued logic  in a possibilistic framework
 and several researchers \cite{Hollunder95,CouchariereLB08,QiJPD11}  do similarly  
 with description logics. See reference \cite{DuboisP04a} 
  to know other embeddings.   
 Since the   Horn-NC class has already been 
 defined in \cite{Imaz21b} for regular many-valued logic,   
 our next aim in this research line will be to embed regular-many-valued into possibilistic logic.
 Specifically, we will attempt  to define the class of regular-many-valued possibilistic  Horn-NC bases
 as well as to generalize other notions presented in this article 
 such as defining regular-many-valued possibilistic NC Unit-resolution. An open
 question is whether polinomiallity is preserved when two uncertainty logics are
 combined within the non-clausal level.

  \vspace{.3cm} 
 $ \Box$ \underline{Defining  NC resolution}. 
 
  The  formalization of the 
existing NC resolution \cite{Murray82} (see also \cite{BachmairG01}) that dates back to the 1980s    is somewhat  confusing.
Its functional-like definition has important weaknesses such as
not precisely identifying the available resolvents or requiring complex 
formal proofs \cite{HahnleMurrayRosenthal04} of its logical properties. 
Another symptom of its barriers
is that, contrary to our approach, it has not   led so far, to define either NC unit-resolution or NC hyper-resolution. Our definition of NC unit-resolution
is clausal-like because,   as stated in Section \ref{sec:Non-Clausal-Unit-resolution},  in presence of clausal formulas,
$ \mathcal{UR}_\Sigma $  coincides with clausal resolution \cite{DuboisP87, DuboisP90}.  We believe that 
 our kind of definition   is fairly well oriented
 to define (full) NC Resolution and to generalize it  
 to  some uncertainty logics, and similarly, to be analyzed and to formally prove
 its logical properties.

  \vspace{.3cm} 
 $ \Box$ \underline{Defining   NC  DPLL}. 
 
  DPLL for possibilistic
 clausal formulas have already been  studied \cite{DuboisP94, Lang00} but DPLL for 
 possibilistic non-clausal formulas has   received no attention. We argue below that 
 the present  article supposes an important step forward to 
  specify the scheme NC DPLL for possibilistic logic.
 The principle of DPLL relies on: (1) the procedure Unit-Propagation; 
 (2) a suitable heuristic
 to choice the literal $ \ell $ on which performs branching; and 
 (3) the determination of the formulas $ \Sigma \wedge \ell $ and $ \Sigma \wedge \neg \ell $,
 namely the formulas on which  DPLL should split the search. One can do
 the next observations:
 (1) NC Unit-Propagation is based on NC unit-resolution,
  and (2)  formulas $ \Sigma \wedge \ell $ and $ \Sigma \wedge \neg \ell $ are
  obtained with NC unit-resolution. Hence, 
  the unique aspect not studied here is that  of the  
  (2) heuristic to choice the branching literal.
  Thus, for future work, we will propose the DPLL schema for   
  possibilistic non-clausal reasoning after studying heuristics
  and taking a closer look to other aspects presented in this article.

  \vspace{.3cm} 
 $ \Box$ \underline{Generalized Possibilistic Logic}

A restriction in standard possibilistic logic    is that only 
conjunctions of weighted formulas are allowed. Dubois et al. 
\cite{DuboisPS12,DuboisPS17,DuboisP18} have conceived
the Generalized Possibilistic Logic (GPL) 
in which the disjunction and the
negation of standard possibilistic  formulas are handled. In this logic,
for instance the   formula 
$ (\vee \ \ \langle \varphi : \alpha \rangle \ \ \langle \psi : \beta \rangle \,) $ belongs to the GPL's language, and
since possibility-valued formulas are also smoothly embedded in GPL,
formulas like the next one belong to GPL:
$$ \{\wedge \ \ \langle (\vee \ P \ Q) : \mbox{N} \ {\bf 1} \rangle \ 
\langle   \neg Q : \mbox{N} \ {\bf .75} \rangle \ 
\neg \langle (\vee \ \neg P \ R) : \mbox{N} \ {\bf 1} \rangle
\ \ \ (\vee \   \langle \neg  R : \Pi \ {\bf .75} \rangle  \ 
\langle  \varphi : \Pi \ {\bf .25} \rangle 
\ ) \ \}$$
One can see that   in GPL, connectives can be internal  or external 
(GPL is a two-tired logic) with their 
corresponding and different semantics, and also, GPL formulas
can   be expressed in NC form. In fact, observe that the previous 
formula is non-clausal. In \cite{DuboisPS17} the authors
prove that the  satisfiability  problem   associated to GPL
 formulas is 
$ \mathcal{NP} $-complete. It is clear
that the standard possibilistic Horn formulas are encapsulated
in GPL and so are the Horn-NC formulas defined here.
 We think that in GPL, sub-classes of external Horn formulas 
 can also be defined,  as well as sub-classes that are 
 both internally and externally Horn. Thus GPL turns out to be an interesting
 logic in the sense that it embeds  
 a variety of classes of Horn-like formulas 
 which potentially could be lifted
 to the NC level.  So for future work, 
 we will study the different classes of Horn-NC GPL formulas that
 are definable and then attempt to prove their  complexity.
 It will be challenging to determine when polymiality is preserved
 in the NC level.
 Of course, once we have mastered the solution of 
GPL Horn-NC-like  formulas, the solving of unrestricted non-clausal GPL  formulas 
 can be envisaged  
 by proposing inference and solving mechanisms.

 
   \vspace{.3cm}
 $ \Box$ \underline{Finding Models or Inconsistency Subsets.} 
 
 Our method focuses on exclusively    determining $ \mbox{Inc}(\Sigma) $. However,
 one important issue for increasing theoretical and practical interest is the 
  obtaining of models or contradictory subsets of the knowledge.  
 In some frameworks as for instance, when the knowledge base is not  definitive and
   is in an experimentation phase, the only data of   $ \mbox{Inc}(\Sigma) $
 may be of not much help. For example, if one expects the knowledge base
 to be consistent and the consistency checker finds it 
 is inconsistent, knowing the knowledge subset 
  causing contradiction,  called "witness" in the literature,
 can be necessary.  Thus for future work, we will envisage deductive
 calculi oriented to providing witnesses as a return data. In this context,
 a more complicated problem is the determination  of whether a 
 knowledge base  has exactly one model or one inconsistent subset, since determining
 if a problem has a unique solution is computationally more expensive than testing 
 if has at least one
 solution  (see \cite{Papadimitriu94}, Chapter 17).

\section{Conclusions}

As  the encoding of practical problems is usually  expressed  in non-clausal form,
  restricting deductive systems to handle   clausal formulas obliges them to use 
  non-clausal-to-clausal transformations, which   are 
  very expensive in terms of: increase of  formula size and number of variables, and 
loss of  logical equivalence   and    original    formula's    structure. 
Further, the clausal form   is
not unique  and however no   insights are available  to guide  towards a ''good'' clausal form
 the non-clausal-to-clausal transformation.
These drawbacks deprive    clausal  reasoning systems to efficiently perform   
in   real-world applications. 

\vspace{.1cm}
To overcome such limitations  and   avoid
the costs induced by the normal form transformation, we process formulas in non-clausal form, 
concretely in negation normal form (NNF). This form allows an arbitrary 
nesting of conjunctions and disjunctions and only limits 
  the scope of the negation connective. NNF can be obtained deterministically
and   solely causes a negligible, easily assumable increase of the formula size.

\vspace{.1cm}
Thus, along the lines of previous works in propositional and regular many-valued logics
\cite{Imaz2021horn, Imaz21b}, we have extrapolated the previous argumentation to 
 possibilistic logic, the most extended
approach to deal with knowledge impregnated of uncertainty and 
 presenting partial inconsistencies. Thus our first   contribution has been
lifting the possibilistic Horn class to the non-clausal level 
obtaining a new possibilistic class, which 
    has been called Horn Non-Clausal, denoted 
$ \mathcal{\overline{H}}_\Sigma $
and  shown that it is a sort of non-clausal analogous
of the standard  Horn   class.  
%
%
Indeed, we have proven that $ \mathcal{\overline{H}}_\Sigma $  subsumes syntactically the 
  Horn class and that   both classes are 
semantically equivalent. We have also proven that all possibilistic
NC bases whose clausal form is Horn belong to $ \mathcal{\overline{H}}_\Sigma $.



 
 


\vspace{.1cm}
In order to compute the inconsistency degree of $ \mathcal{\overline{H}}_\Sigma $ members, 
we have established the calculus Possibilistic Non-Clausal Unit-Resolution,  
denoted $\mathcal{UR}_\Sigma $. We   formally proved that $\mathcal{UR}_\Sigma $ 
correctly computes the inconsistency degree 
of any   $ \mathcal{\overline{H}}_\Sigma $ base. $ \mathcal{\overline{H}}_\Sigma $
was nonexistent in the literature and  extends  
 the   propositional logic calculus  
 given in \cite{Imaz2021horn} to possibilistic logic. 

 \vspace{.1cm}
 After having specified  $ \mathcal{\overline{H}}_\Sigma $ and  $\mathcal{UR}_\Sigma$, 
 we have studied the computational problem of computing
  the inconsistency degree of $ \mathcal{\overline{H}}_\Sigma $ via $\mathcal{UR}_\Sigma$ and determined  that it  is
  polynomial, and hence,   $ \mathcal{\overline{H}}_\Sigma $
  is the first found class to be possibilistic, non-clausal and polynomial.
  
  \vspace{.1cm}
Our formulation of  $\mathcal{UR}_\Sigma$  is
 unambiguously clausal-like   since, when  applied to  clausal formulas,
$\mathcal{UR}_\Sigma $ indeed coincides with clausal unit-resolution. 
This aspect is relevant in the sense that it lays 
the foundations towards redefining  NC resolution in a clausal-like manner 
which could avoid  the  barriers caused by 
 the existing functional-like   definition (see related work). We believe that this clausal-like definition of NC resolution
   will allow  to generalize it to   some other uncertainty logics.
  
\vspace{.1cm}
Finally, in this work we also attempted to show that effective NC reasoning for 
possibilistic and for some other uncertainty logics is an open  
research field and, in view of our  outcomes, we consider  the presented research 
line is rather    promising. A 
 symptom of  such consideration  is the possibility  of our method  to be 
extended to different  possibilistic logic  contexts  giving rise to 
a number of future research directions that were briefly discussed:

\vspace{.1cm}
 Computing possibilistic arbitrary NC bases;
 discovering additional tractable  NC   subclasses;
 conceiving low-degree-polynomial   algorithms;
 extending  generalist possibilistic logic  to   NC;
 combining necessity and possibility measures;
 considering  partially  ordered   possibility measures;
 developing    possibilistic NC logic programming;
 developing   possibilistic  NC answer set programming; 
  casting   richer logics  in  possibilistic logic;
   defining possibilistic NC resolution; 
   defining possibilistic NC  DPLL; and
  finding models.
 
 \vspace{.15cm}
 Some of the above listed future objectives can also be searched in the context of  
 other non-classical logics such as \L ukasiewicz logics, G\"{o}del logic, product logic, etc.

\vspace{.3cm}
\noindent {\bf Funding Source.} Spanish project ISINC (PID2019-111544GB-C21).

\vspace{.2cm}

\section{Proofs of Section \ref{sec:definClassHorn-NCChapeau}}
\label{sec:append}

\vspace{.2cm}
\noindent {\bf Lemma \ref{def:disjunHNC}.}  
A NC disjunction $\varphi=(\vee \ \varphi_1 \ldots \varphi_i \ldots  \varphi_k)$ 
with  $k \geq    1$ disjuncts 
pertains to $ \mathcal{\overline{H}} $  iff  
$ \varphi $ has  $k-1$  negative   disjuncts and one Horn-NC disjunct,    
 formally 
$$\varphi=(\vee \ \varphi_1 \ldots \varphi_i \ldots \varphi_k) \in \mathcal{\overline{H}} \mbox{\em \ \ iff} 
  \quad \exists i \ \mbox{\em s.t.} \ \varphi_i \in \mathcal{\overline{H}} 
\ \ \mbox{\em and} \  \ \forall j \neq i,  \varphi_j \in \mathcal{N}^-.$$

\begin{niceproof} {\bf If:}
 As    the formulas 
$\forall j, j \neq i,  \varphi_j$
have no positive literals, the non-negative disjunctions
of $\varphi= (\vee \ \varphi_1 \ldots  \varphi_i \ldots     \varphi_k) $  
 are those of $ \varphi_i $ plus  $ \varphi_i $ and $ \varphi $ themselves.  
Given that  by hypothesis $ \varphi_i \in \mathcal{\overline{H}}  $ and that 
$ \forall j, j \neq i, \varphi_j $ has no positive literals then   
all of them pertain to $ \mathcal{\overline{H}}$.
\noindent {\bf Iff:} It can be easily proved by contradiction that if  
  any of the  two conditions of the lemma are unsatisfied,
 i.e. (i) $ \exists i, \varphi_i \notin \mathcal{\overline{H}} $  or (ii)
 $\exists i,j, i \neq j,  \varphi_i,\varphi_j \notin \mathcal{N}^-$,  then 
$  \varphi \notin \mathcal{\overline{H}}$. 
%
\end{niceproof} 

\noindent {\bf Theorem \ref{th:HNCequality}.} We have that \ $\mathcal{\widehat{H}}=\mathcal{\overline{H}}$.

\begin{niceproof}  We prove first $\mathcal{\widehat{H}} \subseteq \mathcal{\overline{H}}$ and then 
$\mathcal{\widehat{H}} \supseteq \mathcal{\overline{H}}$.

\vspace{.1cm}
$\bullet$ $\mathcal{\widehat{H}} \subseteq \mathcal{\overline{H}}$   is easily proven by structural induction as outlined below:
   
 \vspace{.1cm} 
({1})   $\mathcal{L} \subset \mathcal{\overline{H}}$  trivially holds. 
 
\vspace{.1cm} 
 ({2}) The non-recursive $ \mathcal{\widehat{H}} $ conjunctions  
 are  literal conjunctions, which trivially verify
 Definition  \ref{theorem:visual} and so are in $\mathcal{\overline{H}}$.  
  Assume that $\mathcal{\widehat{H}} \subseteq \mathcal{\overline{H}}$ holds 
  until  a given inductive step and that  $\varphi_{i} \in \mathcal{\widehat{H}}, \ \varphi_{i} \in \mathcal{\overline{H}},
 \ 1 \leq    i \leq   k$.  In the next recursion, any  
  $\varphi= \{\wedge \  \varphi_1    \ldots  \varphi_{i} \ldots  \varphi_k\} $
 may be added to $ \mathcal{\widehat{H}} $. On the other
 hand,   by induction
 hypothesis, we have $\varphi_{i} \in \mathcal{\overline{H}}, 
  1 \leq    i \leq   k$, and so  by  Lemma \ref{def:HNCconjunc}, 
 $ \varphi \in \mathcal{\overline{H}}$.  Therefore $\mathcal{\widehat{H}} \subseteq \mathcal{\overline{H}}$  holds. 

\vspace{.1cm}
 ({3}) The non-recursive    disjunctions
  in $ \mathcal{\widehat{H}} $ 
  are obviously  Horn clauses, which trivially 
fulfill Definition \ref{theorem:visual} and so are in $\mathcal{\overline{H}}$. Then
assuming that for a given recursive level  
$\mathcal{\widehat{H}} \subseteq \mathcal{\overline{H}}$ holds,
in the next recursion, only  disjunctions $ \varphi $  in ({3}) are added to $\mathcal{\widehat{H}}$. 
But  the  condition of ({3})  and that of 
 Lemma \ref{def:disjunHNC} are equal; so by  Lemma \ref{def:disjunHNC},  
 $ \varphi $  is in  $\mathcal{\overline{H}}$ also.  Therefore $\mathcal{\widehat{H}} \subseteq \mathcal{\overline{H}}$ holds.
 
 \vspace{.1cm} 
 $\bullet$ $\mathcal{\widehat{H}} \supseteq \mathcal{\overline{H}}$. 
Given  that the structures to define $\mathcal{{NC}}$ and $\mathcal{\widehat{H}}$ 
in Definition \ref{def:NCformulas}  and Definition \ref{def:syntacticalNC}, respectively,    
   are equal, the potential inclusion  of each  NC  
   formula $ \varphi$ in   $\mathcal{\widehat{H}}$ 
   is systematically considered. 
 Further, the statement:
  if $ \varphi \in  \mathcal{\overline{H}}$ then $ \varphi \in  \mathcal{\widehat{H}}$, 
    is    proven by structural induction on the depth of formulas, 
     by applying a  reasoning   similar to that of 
  the previous $\mathcal{\widehat{H}} \subseteq \mathcal{\overline{H}}$ case and by also using  Lemmas \ref{def:HNCconjunc}
  and  \ref{def:disjunHNC}.
\end{niceproof}


\vspace{.15cm}
\noindent {\bf Proof of Theorems  \ref{the:HNCtoHorn} and \ref{theo:relation-NC-H}.} Before proving both theorems,  the preliminary  
Theorem  \ref{th:kdisjuncts} is required.

\begin{definition} \label{def:ECNF}  For every $\varphi \in \mathcal{NC}$,  
we define   $cl(\varphi)$ as
    the unique clausal formula    that
  results from   applying  
 $\vee/\wedge$  distributivity to $\varphi$ until a clausal formula, viz. $cl(\varphi)$, is obtained. 

\end{definition}

\begin{theorem} \label{th:kdisjuncts}  
 Let  $\varphi=(\vee  \ \varphi_{1}   \ldots \varphi_i \ldots \varphi_k \,) \in \mathcal{\overline{H}}$. 
   $cl(\varphi) \in \mathcal{H}$    iff  $\varphi$ has  $k-1$  negative  disjuncts and one
disjunct s.t. $cl(\varphi_i) \in \mathcal{H}$, 
 formally: 
 
 \vspace{.2cm}    
$cl(\,(\vee  \ \varphi_{1}   \ldots \varphi_i \ldots \varphi_k \,)\,) 
\in \mathcal{H}$ \ 
 $\mathrm{iff}$ \ 
 $(1) \  \exists i,   \mbox{\,s.t.}
  \ cl(\varphi_i) \in \mathcal{{H}} \   \mbox{\,and}   
\ \,(2) \ \forall j, j \neq i, \varphi_j \in \mathcal{N}^-.$    
\end{theorem}

\begin{niceproof} 
\noindent {\em If-then.} By refutation: let     
 $cl(\,(\vee  \ \varphi_{1}   \ldots  \varphi_i \ldots  \varphi_k)\,) \in  \mathcal{H}$ and prove that if   (1) or (2)  are  violated, then  $cl(\varphi)  \notin \mathcal{H}$. 
\begin{enumerate}

\item   [$\bullet$]  $(1) \  \nexists i,   \mbox{\,s.t.}
  \ cl(\varphi_i) \in \mathcal{{H}}$  

\vspace{.1cm}
 $-$ If we take  the case $k=1$, then $\varphi = \varphi_1$.
 
  $-$  But  $cl(\varphi_1) \notin \mathcal{H}$  implies 
  $cl(\varphi) \notin  \mathcal{H}$.

 \item  [$\bullet$] (2) $\exists j, j \neq i, \varphi_j \notin \mathcal{N}^-.$ 

\vspace{.1cm}
 $-$  Suppose that, besides $\varphi_i$,  $\varphi_j \notin \mathcal{N}^-$ 
 with $\,j \neq i$. 

  $-$   We take   $k=2, \,\varphi_1=P$
and  $\varphi_2=Q$.

 $-$  So, $\varphi=(\vee \ \varphi_1 \ \varphi_2) = (\vee \  P  \ Q)$, and
hence     $cl(\varphi)  \notin \mathcal{H}$.
\end{enumerate}

\noindent {\em Only-If.} Without loss of generality, we take   
 $(\vee \ \varphi_1 \ldots \varphi_{i} \ldots \varphi_{k-1})
=\varphi^- \in  
\mathcal{N}^- 
   \mbox{ and } \varphi_k \in \mathcal{\overline{H}}$,  
and  prove: 

\vspace{.2cm}
\hspace{1.cm} 
$cl(\varphi)=cl(\,(\vee \ \  \varphi_1 \ldots \varphi_i \ldots \varphi_{k-1}  \,\varphi_k) \,)
=cl(\,(\vee \ \varphi^- \ \varphi_k)\,) \in  \mathcal{H}.$



\vspace{.25cm} $-$ To obtain $ cl(\varphi) $, one must obtain first $cl(\varphi^-)$
and $cl(\varphi_k)$, and so

 \vspace{.2cm}
\hspace{1.cm}
 $(i) \ \ cl(\varphi) = cl(\,(\vee \ \  \varphi^-  \  \varphi_k) \,)= cl(\,(\vee \ \ cl(\varphi^-)  \ \ cl(\varphi_k)\,)\,).$

\vspace{.3cm}
 $-$  By definition of $\varphi^- \in \mathcal{N}^-$, 
 
 \vspace{.2cm}
\hspace{0.8cm} $(ii) \ \  cl(\varphi^-)=\{\wedge \ D^-_1   \ldots D^-_{m-1} \,D^-_m\}$; \          
the $D^-_i$'s are negative clauses.

\vspace{.3cm}
  $-$  Since  $\varphi_k \in \mathcal{\overline{H}}$,

\vspace{.2cm}
\hspace{.8cm} $(iii) \ \ cl(\varphi_k)=\mathrm{H}=\{\wedge \ h_1 \  \ldots h_{n-1} \,h_n \}$; \
  the $h_i$'s are Horn clauses.

\vspace{.3cm}
 $-$ By   ($i$),  ($ii$) and ($iii$),  
  
\vspace{.2cm} 
\hspace{1.cm} $cl(\varphi) = cl(\,(\vee \ \  \{\wedge \ D^-_1 \   \, \ldots \, D^-_{m-1} \,D^-_m\} \ \ \{\wedge \ h_1 \   \ldots h_{n-1} \,h_n \, \}\,)\,).$

\vspace{.3cm}
 $-$ Applying $\vee\//\wedge$ distributivity to $cl(\varphi)$ and noting $C_{i}=(\vee \ D^-_1 \ h_{i} \,)$, 

\vspace{.2cm}
\hspace{1.cm} $cl(\varphi) = 
 cl(\,\{\wedge \ \ \{\wedge \ C_1   \ldots C_i \ldots C_n\} \ \  
    (\vee \  \, \{\wedge \ D^-_2 \ldots D^-_{m-1}  \,D^-_m \, \} \  \ \mathrm{H} \,)  \,\} \ ).$

\vspace{.3cm}
 $-$ Since the  $C_i=(\vee \ D^-_1 \ h_{i} \,)$'s are  Horn clauses, 
 
 \vspace{.2cm} 
\hspace{1.cm} $ \{\wedge \ C_1   \ldots C_i \ldots C_n\}=\mathrm{H}_1    \in \mathcal{H}$. 

\vspace{.2cm} 
\hspace{1.cm} $cl(\varphi) = cl (\,\{\wedge \ \ \mathrm{H}_1 \ \ 
(\vee \ \ \{\wedge \ D^-_2 \ldots D^-_{m-1} \,D^-_m \, \}   \ \ \mathrm{H}\, ) \,\} \ ).$

\vspace{.3cm}
 $-$ For $j < m$ we have,
 
\vspace{.2cm} 
\hspace{1.cm} $cl(\varphi) = cl( \ \{\wedge \ \mathrm{H}_1 \ \ldots  \mathrm{H}_{j-1}  \mathrm{H}_j   \  \ (\vee \ \ \{\wedge \ D^-_{j+1} \ldots D^-_{m-1} \,D^-_m\}  \ \  \mathrm{H} \,) \, \} \ ).$

\vspace{.3cm}
 $-$ For $j = m$, \ 
   $cl(\varphi) = \{\wedge \ \mathrm{H}_1 \  \ldots \mathrm{H}_{m-1} \,\mathrm{H}_m \     
\mathrm{H} \,\}=\mathrm{H}' \in \mathcal{H}.$

\vspace{.3cm}
 $-$  Hence   $cl(\varphi)   \in  \mathcal{H}$.
\end{niceproof}

\vspace{.3cm}

{\em   \noindent {\bf Theorem \ref{the:HNCtoHorn}.}
 $\forall \varphi \in  \mathcal{\overline{H}}$  we have $cl(\varphi) \in \mathcal{H} $.  
}

\begin{niceproof} 
 We consider Definition \ref{def:syntacticalNC}    of $ \mathcal{\overline{H}} $. The proof 
    is done by  structural induction  on   the depth $r(\varphi)$  of any
    $\varphi \in \mathcal{\overline{H}}$ and    defined  below, where $ \ell $ is a literal:
\[r(\varphi)= \left\{
\begin{array}{l l l}

\vspace{.15cm}

0    &     \varphi=[ \odot \  \ell_1 \ \ldots \ell_{k-1} \  \ell_k ]  
  \ \mbox{or} \ \varphi =  \ell.   \\ 

1+max\,\{r(\varphi_1), \ldots, r(\varphi_{k-1}),\,r(\varphi_k)\}   & 
\varphi=[ \odot \  \varphi_1 \ \ldots \varphi_{k-1} \  \varphi_k ]. \\

\end{array} \right. \]

$\bullet$ {\it Base Case:} $r(\varphi)=0.$ 

\vspace{.25cm}
\quad   \    --  Clearly,
   $r(\varphi)=0$ entails   $\varphi= [\odot \  \ell_1 \ \ldots \ell_{k-1} \  \ell_k ] \in \mathcal{H} $
   and $\varphi= \ell \in  \mathcal{H}$.
  
\vspace{.15cm}
\quad   \   
  -- So  $cl(\varphi)=\varphi \in \mathcal{H}$.

\vspace{.25cm} 
$\bullet$ {\em Induction hypothesis:} 
$     \forall \varphi,  \ r(\varphi) \leq n, \ \ \varphi \in \mathcal{\overline{H}} \mbox{ \ entails \ } 
cl(\varphi) \in \mathcal{H}.$

\vspace{.25cm}
$\bullet$ {\em Induction proof: \ $r(\varphi)=n+1$.}  

\vspace{.2cm}
\quad    By Definition \ref{def:syntacticalNC}, lines (2) ad (3)  below arise: 

\vspace{.35cm} 
\hspace{.35cm}  (2)     $\varphi=
\{\wedge  \ \varphi_1   \ldots \varphi_{i}  \ldots  \varphi_k\}
$, where $k \geq 1$.

  \vspace{.3cm} 
\hspace{1.2cm} $-$  By definition of $r(\varphi)$, 

\vspace{.25cm}
\hspace{1.8cm}  \quad
$r(\varphi)=n+1$ \,entails \   $1 \leq i \leq k, \ r(\varphi_i) \leq n.$

\vspace{.25cm}
\hspace{1.2cm} $-$  By induction  hypothesis,

\vspace{.23cm}
\hspace{1.8cm}  \quad $\varphi_i \in \mathcal{\overline{H}} 
\mbox{\, and  \,} r(\varphi_i) \leq n
  \mbox{  \ entail \ }  cl(\varphi_i) \in \mathcal{H}$.

\vspace{.25cm}
\hspace{1.2cm}  $-$ It is obvious that,

\vspace{.23cm}
\hspace{1.8cm}  \quad  $  cl(\varphi) = \{\wedge  \ \  cl(\varphi_1) \  
 \ldots  \ cl(\varphi_{i})   \ldots  cl(\varphi_k) \,\}.$
    
 \vspace{.25cm}
\hspace{1.2cm} $-$  Therefore,

\vspace{.23cm}
\hspace{1.8cm}  \quad
$cl(\varphi) =  \{\wedge  \ \mathrm{H}_1 
\ldots \mathrm{H}_{i}  \ldots \mathrm{H}_k\}= \mathrm{H} \in \mathcal{H}$.

\vspace{.35cm}
\hspace{.35cm}    (3)
 $\varphi=(\vee  \ \varphi_1  \ldots \varphi_{i} \ldots \varphi_{k-1} \ \varphi_k) 
 \in \mathcal{\overline{H}}$,  where:

 \vspace{.25cm}
\hspace{1.8cm} \quad $k \geq 1$,  \ $0 \leq i \leq k-1, \ \varphi_i   \in \mathcal{N}^-$ and   \,$\varphi_k \in \mathcal{\overline{H}}.$

 \vspace{.35cm} 
\hspace{1.2cm} $-$ By  definition of $r(\varphi)$,
 
\vspace{.2cm}
\hspace{1.8cm}  \quad   
$r(\varphi) = n+1$ \ entails   \,$r(\varphi_k) \leq n.$

 \vspace{.25cm} 
\hspace{1.2cm} $-$  By  induction  hypothesis,

\vspace{.25cm}
\hspace{1.8cm}  \quad $d(\varphi_k)  \,\leq n$ \,and \,$\varphi_k \,\in \mathcal{\overline{H}}$
 \,entail  $cl(\varphi_k) \,\in \mathcal{H}.$

\vspace{.25cm}
\hspace{1.2cm} $-$  By Theorem  \ref{th:kdisjuncts},   {\em only-if,} \  

 \vspace{.25cm}
\hspace{1.8cm} \quad     $0 \leq i \leq k-1, \ \varphi_i   \in \mathcal{N}^-$ and   \,$cl(\varphi_k) \in \mathcal{H}$ entail:

\vspace{.25cm}
\hspace{2.8cm}
 $cl(\,(\vee  \ \varphi_1  \ldots \varphi_{i} \ldots \varphi_{k-1} \ \varphi_k)\, )  \in \mathcal{H}$.
\end{niceproof}


 \vspace{.15cm}
{\em  
\noindent {\bf Theorem \ref{theo:relation-NC-H}}.  $\forall \varphi \in \mathcal{NC}$:    
if  $cl(\varphi) \in \mathcal{H}$  then    $\varphi \in  \mathcal{\overline{H}}$.
}

\begin{niceproof}   It is done  by  structural induction on   the depth $d(\varphi)$  of $\varphi$ defined   as     
\[d(\varphi)= \left\{
\begin{array}{l l l}

\vspace{.1cm}

0     &   \varphi \in \mathcal{C}.\\
    
1+max\,\{d(\varphi_1), \ldots, d(\varphi_{i}),\ldots \,,d(\varphi_k)\}   &  
\varphi=[ \odot \  \varphi_1  \ldots \varphi_{i} \ldots  \varphi_k ].\\

\end{array} \right. \]

\vspace{.2cm}
$\bullet$ {\it Base case:} $d(\varphi)=0$ and    $cl(\varphi) \in \mathcal{H}$.

\vspace{.2cm}
\ \  \  $-$  $d(\varphi)=0$ entails  $\varphi \in \mathcal{C}$.

\vspace{.2cm}
\ \  \    $-$ If $\varphi \notin \mathcal{H}$, then $cl(\varphi)  \notin  \mathcal{H}$,  
 contradicting  the  assumption.

\vspace{.2cm}
\ \  \  $-$ Hence  $\varphi \in \mathcal{H}$ and so by Definition \ref{def:syntacticalNC},
 $\varphi \in \mathcal{\overline{H}}$.


\vspace{.2cm}
$\bullet$ {\em Inductive hypothesis:}  
$ \forall \varphi \in \mathcal{NC}, \ d(\varphi) \leq n, \ 
cl(\varphi) \in \mathcal{H} \mbox{ \ entail \ } \varphi \in \mathcal{\overline{H}}.$

\vspace{.25cm}
$\bullet$ {\em Induction proof: $d(\varphi)=n+1$.} 

\vspace{.25cm}
\quad   By Definition \ref{def:NCformulas} of $\mathcal{NC}$, cases  $(i)$ and $(ii)$  below arise.

\vspace{.3cm}
\hspace{.35cm}    $(i)$  \   $cl(\varphi)= 
  cl(\, \{\wedge \  \varphi_{1}   \ldots \varphi_{i} \ldots   \,\varphi_k\} \, )
\in \mathcal{H}$ and $k \geq 1$.

\vspace{.25cm} 
\hspace{1.2cm} $-$ Since $\varphi$ is a conjunction, $1 \leq i \leq k, \,cl(\varphi_i)  \in \mathcal{H}.$
 

 \vspace{.25cm} 
\hspace{1.2cm} $-$  By  definition of $d(\varphi)$,
 
\vspace{.2cm}
\hspace{1.8cm}  \quad   
$d(\varphi) = n+1$ \ entails   \,$1 \leq i \leq k, \,d(\varphi_i) \leq n.$

\vspace{.25cm} 
\hspace{1.2cm}  $-$ By induction hypothesis,

\vspace{.2cm}
\hspace{1.8cm} \quad 
$1 \leq i \leq k, \ \,d(\varphi_i) \leq n,  \ cl(\varphi_i) \in \mathcal{H}$
$  \mbox{\ entail\ } \varphi_i \in \mathcal{\overline{H}}.$

\vspace{.25cm}
\hspace{1.2cm} $-$  By Definition  \ref{def:syntacticalNC}, line (2),  

\vspace{.2cm} 
  \hspace{1.8cm} \quad $1 \leq i \leq k, \ \varphi_i   \in \mathcal{\overline{H}}$ 
  entails $\varphi \in \mathcal{\overline{H}}$.

\vspace{.25cm}
 \hspace{.35cm}     $(ii)$   \ $cl(\varphi) = 
 cl(\, (\vee  \ \varphi_{1}  \ldots \varphi_i  \ldots   \varphi_{k-1} \ \varphi_k) \,) \in \mathcal{H}$ 
  and $k \geq 1$.

 \vspace{.25cm} 
\hspace{1.2cm} $-$  By   Theorem \ref{th:kdisjuncts}, {\em if-then,}  

\vspace{.2cm}
\hspace{1.8cm}   \quad    
 $0 \leq i \leq k-1, \ \varphi_i  \in \mathcal{N}^-$ and   \,$cl(\varphi_k) \in \mathcal{H}.$
 
    \vspace{.25cm} 
\hspace{1.2cm} $-$  By  definition of $d(\varphi)$,
 
\vspace{.2cm}
\hspace{1.8cm}  \quad   
$d(\varphi) = n+1$ \ entails   \,$d(\varphi_k) \leq n.$

 \vspace{.25cm} 
\hspace{1.2cm}  $-$ By       induction  hypothesis,

\vspace{.2cm}
\hspace{1.8cm}  \quad   $d(\varphi_k)  \,\leq n$ and $cl(\varphi_k) \in \mathcal{H}$
\,entail   $\varphi_k \,\in \mathcal{\overline{H}}.$

\vspace{.25cm}
\hspace{1.2cm}   $-$  By Definition \ref{def:syntacticalNC},  line (3),

\vspace{.2cm} 
  \hspace{1.8cm} \quad  $0 \leq i \leq k-1, \ \varphi_i   \in \mathcal{N}^-$ and  
   \,$\varphi_k \in \mathcal{\overline{H}}$ 
   entail:
  
 \vspace{.25cm} 
  \hspace{2.8cm} \quad $(\vee  \ \varphi_{1}  \ldots \varphi_i  \ldots   \varphi_{k-1} \ \varphi_k) = 
  \varphi \in \mathcal{\overline{H}}$. 
  
\end{niceproof}

\end{document}